\documentclass{article}
\usepackage{amsmath,amssymb,amsthm,cases}
\usepackage[utf8]{inputenc}
\usepackage[margin=1in]{geometry}
\usepackage{algorithm}

\usepackage{multirow,multicol}
\usepackage{microtype}
\usepackage{graphicx}
\usepackage{subfigure}
\usepackage{booktabs}
\usepackage{natbib}
\setcitestyle{open={(},close={)}}

\usepackage[colorlinks=true,citecolor=blue]{hyperref}

\newtheorem{theo}{Theorem}
\newtheorem{defi}{Definition}

\newtheorem{prop}{Proposition}

\newtheorem{fact}{Fact}

\newcommand*{\defeq}{\stackrel{\text{def}}{=}}

\begin{document}

\title{\textbf{From a few Accurate 2D Correspondences
to\\ 3D Point Clouds}}

\author{\vspace{0.5in}\\\textbf{Trung-Kien Le} ~and~ \textbf{Ping Li}\vspace{0.2in}\\
Cognitive Computing Lab\\
 Baidu Research\\
 10900 NE 8th St. Bellevue, WA 98004, USA\\\\
 \{hieukien1207,\ pingli98\}@gmail.com\\\\
}

\date{}

\maketitle

\begin{abstract}
\noindent Key points, correspondences, projection matrices, point clouds and dense clouds are the skeletons in image-based 3D reconstruction, of which point clouds have the important role in generating a realistic and natural model for a 3D reconstructed object. To achieve a good 3D reconstruction, the point clouds must be almost everywhere in the surface of the object. In this article, with a main purpose to build the point clouds covering the entire surface of the object, we propose a new feature named a \emph{geodesic feature} or \emph{geo-feature}. Based on the new geo-feature, if there are several (given) initial world points on the object's surface along with all accurately estimated projection matrices, some new world points on the geodesics connecting any two of these given world points will be reconstructed. Then the regions on the surface bordering by these initial world points will be covered by the point clouds. Thus, if the initial world points are around the surface, the point clouds will cover the entire surface. A prerequisite for applying the geo-feature to build the point clouds is having several initial world points on the object's surface and the accurate estimated projection matrices. Fortunately, from the multi-view geometry theory, the initial world points and the projection matrices can be accurately derived by their correspondences. Following this task, this article proposes a new method to estimate the world points and projection matrices from their correspondences. This method derives the closed-form and iterative solutions for the world points and projection matrices and proves that when the number of world points is less than seven and the number of images is at least five, the proposed solutions are global optimal. Finally, in this article, we propose an algorithm named \emph{World points from their Correspondences} (\texttt{WPfC}) to estimate the world points and projection matrices from their correspondences, and another algorithm named \emph{Creating Point Clouds} (\texttt{CrPC}) to create the point clouds from the world points and projection matrices given by the first algorithm. To be simple, from a few accurate correspondences as an initialization, we obtain the estimations of all projection matrices and the point clouds covering the entire surface of the object to be 3D reconstructed. We apply the algorithms \texttt{WPfC} and \texttt{CrPC} on the 3D Buddha statue reconstruction project. The obtained point clouds have more than ten millions points and they are almost everywhere on the statue's surface. More than $95\%$ relative errors computed by correspondences are less than 30 pixels. From our experience on image-based 3D reconstruction, this achievement appears dramatic and validates our study.
\end{abstract}

\newpage

\section{Introduction}\label{Sec:Introduction}

Image-based 3D reconstruction have been widely applied to our life and science. The central purpose of 3D reconstruction is to create a spatial virtual model that brings more information from images than what we can see with the naked eyes. With this dynamic, image-based 3D reconstruction is an age-old method, originally used in architecture long before computer-aided imaging techniques were introduced. As early as the Renaissance, scholars studied the appearance of the architecture of the past, analyzing it by means of images, and using it to construct their own contemporary buildings~\citep[page 6]{Carpo2001}. Nowadays, with strong support from computer graphics and computer vision that a virtual model is built accurately and very quickly, and more importantly it can be visualized, we can see applications of image-based 3D reconstruction on robotics, autonomous driving~\citep{Wang2007,Cadena2016,Bresson2017}, medical imaging analysis~\citep{McInerney1996,Heimann2009,Francisco2014}, physical geography~\citep{Smith2016}, archaeology~\citep{Hermon2008,Bruno2010}, and many others. Recently, the advent of \emph{Metaverse}\footnote{Metaverse is a network of 3D virtual worlds focused on social connection. (https://en.wikipedia.org/wiki/Metaverse)} and \emph{Second Life}\footnote{Second Life is an online multimedia platform that allows people to create an avatar for themselves and have a second life in an online 3D virtual world. (https://en.wikipedia.org/wiki/Second\_Life)} reflects the power of digital 3D reconstruction~\citep{Kemp2006,Dionisio2013}. One of the challenges for the Metaverse and Second Life is to create digital environments that appear realistic and natural-looking spaces. Some high-accurate and high-solution 3D virtual models of buildings, physical locations, and objects are rendered and passed to computers to process and generate virtual worlds. Thus, beside \emph{artificial intelligence} (AI), 3D reconstruction is a key technology to power Metaverse and Second Life.

\begin{figure*}[t]
  \centering
  \includegraphics[width=0.9\linewidth]{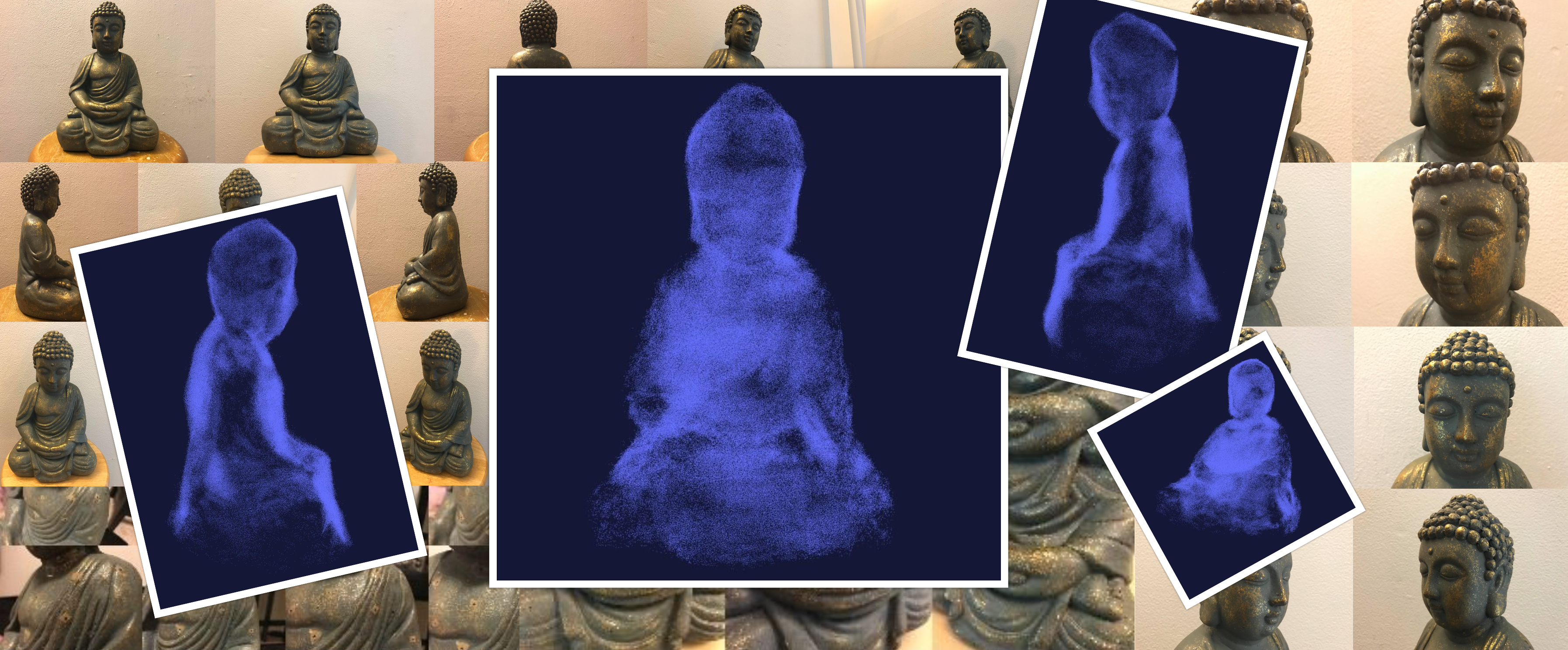}
  \caption{The image-based 3D Buddha statue reconstruction: Images and Point clouds.}
  \label{Fig:Main}\vspace{0.1in}
\end{figure*}

There is a forest of knowledge and techniques in image-based 3D reconstruction, despite its deceptively simple definition. Starting with multiple 2D images capturing an object, image-based 3D reconstruction attempts to build a dense to estimate the surface of this object. This process can be summarized by the three following steps that (i) detect and collect as many \emph{correspondences} as possible, (ii) estimate projection matrices of all images and build \emph{point clouds} from these correspondences, and (iii) apply some smooth surface (rendering) techniques on the point clouds to generate a \emph{dense} and complete the process of 3D reconstruction. Figure~\ref{Fig:Main} is an example of image-based 3D Buddha statue reconstruction. There are a lot of the Buddha statue images (far-view and near-view) used to create the point clouds (the \emph{blue dots} in four black background pictures).

To detect correspondences, we first describe or define and then estimate key points on each image. The \emph{key point estimation} is also known by another name of the \emph{feature detection}. Depending on the object to be 3D reconstructed, key points or features can be defined as special points~\citep{Hsieh1996}, edges~\citep{Ziou1998,dollar2013}, corners~\citep{Zheng1999,Rosten2006}, circles~\citep{Alhichri2003}, or special curves. Of course with different definitions of key points as special points, edges, corners or circles, we have different techniques to extract these key points. In general, SIFT~\citep{Lowe1999} and SURF~\citep{Bay2006} are two popular and effective methods for extracting key points. When key points on each image are collected, correspondences are detected based on the matching key points or the feature matching problem~\citep{Munkres1957,Torresani2012,Yan2016,Swoboda2019}. \emph{Matching key points} is a difficult problem and has received a lot of attention from researchers. The well-known formulation of the matching key points problem is the \emph{quadratic assignment problem} (QAP) that is NP-hard~\citep{Lawler1963}. In addition, the matching key points across multiple images is still a challenging task even though there are several potential methods for this on two images. More recently, \citet{le2022multi} have proposed some algebraic properties for multi-view correspondences. These properties say that there are some constraints among key points if they are correspondences. These constraints will help us on estimating, refining and recovering correspondences from the given key points.

From correspondences, a \emph{fundamental matrix} of each pair  images is determined based on the most popular formula in 3D reconstruction that ${\bf x}'{\bf F}{\bf x} = 0$ where ${\bf x}$ and ${\bf x}'$ are two matching key points on two images and ${\bf F}$ is the fundamental matrix of these two images~\citep{Luong1996}. When the number of two-view correspondences $\{{\bf x}\leftrightarrow{\bf x}'\}$ is at least eight, the \emph{eight-point algorithm}~\citep{Hartley2007} derives the fundamental matrix ${\bf F}$ by linear equations, and when this number is five, the \emph{five-point algorithms}~\citep{Nister2004,Barath2018} find solutions for ${\bf F}$ by high-degree polynomial equations. Since the fundamental matrix presents a relationship between two projection matrices corresponding to two images, all projection matrices of all images are determined by fixing one projection matrix of the reference image. As the projection matrices are given, the point clouds are built based on \emph{triangulation} methods~\citep{Hartigan1979,Stewenius2005,yang2019,Sharp2020} together with a process on creating lots of new correspondences. Finally, the point clouds are rendered to become a dense by the \emph{bundle adjustment} techniques, \emph{harmonic maps} and \emph{conformal maps}~\citep{Triggs1999,Agarwal2010,Zach2014,Wu2020,Choi2022}.

Continuing our previous work in~\cite{le2022multi} on solving the first step of the image-based 3D reconstruction that detects a few accurate correspondences from multiple images, this paper will focus on the second step of estimating projection matrices from a few correspondences and creating point clouds. With the first task that estimate projection matrices, we motivate to obtain very accurate projection matrices. Since image distortion is difficult to avoid, there is always an error in the correspondence estimate. Indeed, to be successful in 3D reconstruction with good point clouds, we need to have very precise projection matrices for all images. As we discussed in the previous paragraph, most of existing methods for the projection matrix estimation\footnote{The projection matrix estimation is also known as another named the \emph{camera calibration}.}~\citep{Weng1992,Wei1994,Zhang2000,Remondino2006,Chuang2021} estimate projection matrices through fundamental matrix estimation. There are ambiguities on the image-based 3D reconstruction as we shall present in Section~\ref{subsec:Ambiguity}, and these methods solve the ambiguity problem by fixing one projection matrix of the reference image. When this projection matrix is given, other projection matrices are easily derived based on the fundamental matrices. In our study, we shall estimate projection matrices directly from correspondences, not through the fundamental matrix estimation. The main difference between our method and other existent methods is in handling the ambiguities of the image-based 3D reconstruction. From our experience on the \emph{localization} problem~\citep{Kien2016,Kien2020,Kien2020_2}, we solve the ambiguity problem by fixing the first five world points corresponding to the correspondences.

More precisely, if correspondences $\{{\bf x}_m^{\text{\tiny(1)}}\leftrightarrow{\bf x}_m^{\text{\tiny(2)}}\leftrightarrow\cdots\leftrightarrow{\bf x}_m^{\text{\tiny(N)}}\}^M_{m=1}$ of $M$ ($M\geq 5$) world points $\{{\bf X}_1,{\bf X}_2,\ldots,{\bf X}_M\}$ and $N$ projection matrices $\{{\bf P}_1,{\bf P}_2,\ldots,{\bf P}_N\}$ are used, from the ambiguities we can fix ${\bf X}_1 = [0\,,\,0\,,\,0]^T$, ${\bf X}_2 = [0\,,\,0\,,\,1]^T$, ${\bf X}_3 = [0\,,\,1\,,\,0]^T$, ${\bf X}_4 = [1\,,\,0\,,\,0]^T$ and ${\bf X}_5 = [1\,,\,1\,,\,1]^T$ in the process of finding solutions for $N$ projection matrix $\{{\bf P}_1,{\bf P}_2,\ldots,{\bf P}_N\}$. The handle of fixing the first five world points allows us to obtain the closed-form solutions for both $N$ projection matrices $\{{\bf P}_1,{\bf P}_2,\ldots,{\bf P}_N\}$ and $M$ world points $\{{\bf X}_1, {\bf X}_2,\ldots,{\bf X}_M\}$ when $N\geq 5$ and $M\geq 6$ (Theorem~\ref{Theo:Sol_WPfC_NoNoise}). In addition, when $M = 6$ we prove that our closed-form solutions are the best from the point of view of minimizing the mean squared errors on the correspondences (Proposition~\ref{Prop:CF_Sol_6}). Beside the closed-form solutions for projection matrices, we propose the iterative solutions (Propositions~\ref{Prop:Opt_by_Iso} and~\ref{Prop:Opt_by_Iter}) that are guaranteed that our final solutions are not only better than the closed-form solutions but also local optimal.

\newpage

Regarding the second task of creating point clouds, we motivate to build the point clouds such that they can cover everywhere on the object's surface. This is a tough challenge. Although the property of being everywhere on the object surface of point clouds is important, to our knowledge no 3D reconstruction studies have addressed this property. Obviously, in order to have point clouds covering everywhere on the object's surface, we need a feature that can recognize all the points on the surface by images. A new concept of \emph{geo-feature} proposed by this paper (Definition~\ref{Def:Geo-feature}) can do what we want. A new world point has a \emph{geo-feature} if there are two old world points on the object's surface such that the new world point lays on the \emph{geodesic} in this surface connecting these two old world points. The `old' and `new' world points mean their coordinates are and are not known, and a geodesic of two points on a surface is the shortest curves on this surface connecting these two points~\citep{Busemann2012}. The definitions of geo-feature and geodesic confirm that if a point has a geo-feature it will be on the object's surface. Therefore, if a new geo-feature point is recognized and determined, it becomes an old world point. Then along with other old world points it helps us to find other new world points. To recognize and determine a geo-feature point, we first estimate the geodesic of two old world points by straight lines connecting image points of these two old world points on each image. From these 2D straight lines, we find correspondences corresponding to geo-feature points, then derive these geo-feature points based on the projection matrices that estimated by the first task. This process will help us to create the point clouds with millions points only from a few given points on the surface. Not only with a huge number of points in them, these point clouds also cover the entire surface of the object.

As a consequence of our results in studying the first step, we propose an algorithm to estimate both the world points and projection matrices from their correspondences. This algorithm names \emph{World points from their Correspondences} (\texttt{WPfC}) and is given by Algorithm~\ref{Alg:Iteration}. From the results for the second step, we also propose another algorithm to create the point clouds based on the world points and projection matrices from the first algorithm. The second algorithm is called \emph{Creating Point Clouds} (\texttt{CrPC})
and presented by Algorithm~\ref{Alg:Point_Clouds}. Unlike the first algorithm in which its correctness and efficiency are guaranteed by some theoretical results like the global solution given by Proposition~\ref{Prop:CF_Sol_6} and the iterative solutions given by Proposition~\ref{Prop:Opt_by_Iter}, the validity of the second algorithm is to be thoroughly investigated through practical experiments.

We apply our study in this paper to the \emph{image-based 3D Buddha statue reconstruction} project. A purpose of this project is modeling the Buddha statue by point clouds based on a lot of images capturing this statue. Note that point clouds, not like dense clouds, are 3D colorless points and they do not have any continuous regions. As such, seeing the Buddha statue through point clouds is difficult. In other words, if we can see the 3D shape of the statue through the point clouds, then these point clouds would be very good. On another hand, getting ground-truths for the Buddha statue is a hard and expensive job if we don't want to say that it is impossible task. Thus, in this paper, we mainly evaluate the quality of the point clouds based on their similarity of these point clouds with the statue through our eyes.

In this project, the correspondences are found based on 112 special points marked on the statue's surface as we see clearly in Figures~\ref{Fig:Sec2},~\ref{Fig:Images_Keypoints},~\ref{Fig:Geo_features} and~\ref{Fig:Experiments}. We use SIFT~\citep{Lowe1999} and SURF~\citep{Bay2006} to detect these special points on images, and then use the work in~\citet{le2022multi} to estimate the correspondences. There are twelve experiments in the project. Each experiment has around 80$\sim$150 images that capture around 20$\sim$30 special points. A role of each experiment is reconstructing an area corresponding to this experiment. For instance, the \emph{front body experiment} reconstructs the 3D shape of the front body, and the \emph{face experiment} reconstructs the 3D shape of the face of the statue. Note that there are a lot of correspondences in each experiment. Indeed, from twelve experiments, we have thousands outputs of Algorithm \texttt{WPfC} and millions outputs of Algorithm \texttt{CrPC}. These huge numbers of implementing \texttt{WPfC} and \texttt{CrPC} are sufficient for us to get correct evaluations of these two algorithms and also the study in this paper. Since we do not have ground-truths for world points, projection matrices and point clouds, each implementation of \texttt{WPfC} or \texttt{CrPC} is evaluated based on a \emph{relative error} defined as a difference between the correspondences in the input and other correspondences computed by the output of \texttt{WPfC} or \texttt{CrPC}. Although there are no theoretical results to warrant a correctness for a relative error, we still believe that small relative error reflects a good estimation from these two algorithms. Hence, in addition to seeing with the naked eye the similarity between the statue and the point clouds, we use the histograms of relative errors to evaluate our study. From these histograms, we know that 95 percent relative errors on key points (the image points of the 112 special points marked in the statue) are less than 49 pixels, and 95 percent relative errors on image points (image points of ten millions points in the point clouds) are less than 25 pixels. Working with the images of size $3024\times 4032~\text{pixel}^2$, these small relative errors on our 3D reconstruction appear to be really dramatic.\\

\begin{table}[t!]
\centering
{\small\begin{tabular}{ll}
\hline\\[-2mm]
${\bf X}_m = \big[x_m\,,\,y_m\,,\,z_m\big]^T$ & World point $m$, a $3\times 1$ real vector\\[1.5mm]
${\bf P}_n = \begin{bmatrix}{\bf p}_{1,n}\\{\bf p}_{2,n}\\{\bf p}_{3,n}\end{bmatrix} = \begin{bmatrix}p_{1n} & p_{2n} & p_{3n} & p_{4n}\\p_{5n} & p_{6n} & p_{7n} & p_{8n}\\p_{9n} & p_{10n} & p_{11n} & 1\end{bmatrix}$ & Projection matrix $n$, a $3\times 4$ real matrix\\[5.5mm]
${\bf x}^{\text{\tiny(n)}}_m = \big[u_{mn}\,,\,v_{mn}\big]^T$ & Key point of the world point $m$ on the image $n$, a $2\times 1$ real vector\\[1.5mm]
$\big\{{\bf x}^{\text{\tiny(1)}}_m\leftrightarrow{\bf x}^{\text{\tiny(2)}}_m\leftrightarrow\cdots\leftrightarrow{\bf x}^{\text{\tiny(N)}}_m\big\}$ & A correspondence of world point ${\bf X}_m$\\[1.5mm]
$\big\{{\bf x}^{\text{\tiny(1)}}_m\leftrightarrow{\bf x}^{\text{\tiny(2)}}_m\leftrightarrow\cdots\leftrightarrow{\bf x}^{\text{\tiny(N)}}_m\big\}^M_{m=1}$ & Correspondences of $M$ world points ${\bf X}_1,\ldots,{\bf X}_M$\\[1.5mm]
\multirow{2}{*}{$\{{\bf X}_m,{\bf P}_n\}$} & Group of $M$ world points $\{{\bf X}_1,{\bf X}_2,\ldots,{\bf X}_M\}$\\
~ & and $N$ projection matrices $\{{\bf P}_1,{\bf P}_2,\ldots,{\bf P}_n\}$\\[1.5mm]
\emph{\normalsize World points from Correspondences:} &~\\[1mm]
\multirow{2}{*}{\qquad$\mathcal{K}_n\big(\{{\bf X}_m\}\big)$} & $12\times 2M$ matrix given by~\eqref{eq:Camera_from_Point} used to estimate projection matrix ${\bf P}_n$\\
~ & from $M$ world points $\{{\bf X}_m\}^M_{m=1}$ and $M$ key points $\{{\bf x}^{\text{\tiny(n)}}_m\}^M_{m=1}$\\[1.5mm]
\multirow{2}{*}{\qquad$\mathcal{H}_m\big(\{{\bf P}_n\}\big)$} & $2N\times 4$ matrix given by~\eqref{eq:Point_from_Camera} used to estimate world point ${\bf X}_m$\\
~ & from $N$ projection matrices $\{{\bf P}_n\}^N_{n=1}$ and $N$ key points $\{{\bf x}^{\text{\tiny(n)}}_m\}^N_{n=1}$\\[1.5mm]
\qquad${\bf T}$ & Projective transformation on $\mathbb{R}^3$, a non-singular $4\times 4$ matrix\\[1.5mm]
\qquad$\xi^{\text{\tiny(n)}}_{m_1m_2,m_3}$ & Defined by~\eqref{eq:Xi} that is computed by three key points ${\bf x}^{\text{\tiny(n)}}_{m_1}, {\bf x}^{\text{\tiny(n)}}_{m_2}, {\bf x}^{\text{\tiny(n)}}_{m_3}$\\[1.5mm]
\qquad${\bf \Lambda}_m, f_m, f_{xm}, f_{ym}, f_{zm}$ & Used to estimate world point ${\bf X}_m,~m\geq 6$\\[2mm]
\qquad$\alpha_n, \beta_n, \gamma_n$ & Used to estimate projection matrix ${\bf P}_n$\\[1.5mm]
\qquad$\big\{\hat{\bf x}^{\text{\tiny(1)}}_m\leftrightarrow\hat{\bf x}^{\text{\tiny(2)}}_m\leftrightarrow\cdots\leftrightarrow\hat{\bf x}^{\text{\tiny(N)}}_m\big\}^M_{m=1}$ & Estimated correspondences\\[1.5mm]
\qquad$\{{\bf X}^{\circ}_m,{\bf P}^{\circ}_n\}$ & Closed-form solution, the output of Algorithm~\ref{Alg:Closed_form}, \texttt{CF-WPfC}\\[1.5mm]
\qquad$\{{\bf X}^{\circ}_m,{\bf P}^*_n\}$ & Better solution than $\{{\bf X}^{\circ}_m,{\bf P}^{\circ}_n\}$ because of Proposition~\ref{Prop:Opt_by_Iso}\\[1.5mm]
\qquad$\{{\bf X}^{**}_m,{\bf P}^{**}_n\}$ & Better solution than $\{{\bf X}^{\circ}_m,{\bf P}^*_n\}$ because of Proposition~\ref{Prop:Opt_by_Iter}\\[1.5mm]
\qquad$\{{\bf X}^{\star}_m,{\bf P}^{\star}_n\}$ & The best solution in our work, the output of Algorithm~\ref{Alg:Iteration}, \texttt{WPfC}\\[1.5mm]
\emph{\normalsize Creating Point clouds:} &~\\[1mm]
\qquad\emph{key point} & An image point of the world point in $\{{\bf X}^{\star}_1, {\bf X}^{\star}_2,\ldots,{\bf X}^{\star}_M\}$\\[1.5mm]
\qquad\emph{image point} & An image point of a point in point clouds\\[1.5mm]
\qquad$\mathcal{PC}$ & Point clouds\\[1.5mm]
\qquad$\textbf{Img}_n$ & Image $n$, a $H\times W$ real matrix\\[1.5mm]
\qquad$\textbf{Img}_n(h,\omega)\,,\,\textbf{Img}_n\big({\bf x}^{\text{\tiny(n)}}_k\big)$ & Value of $\textbf{Img}_n$ at $(h,\omega)$ or image point ${\bf x}^{\text{\tiny(n)}}_k$\\[1.5mm]
\multirow{2}{*}{\qquad$\mathcal{G}({\bf X}_k,{\bf X}_{k'})$} & Geodesic between two points ${\bf X}_k, {\bf X}_{k'}$\\
~ & These two points can be world points or points in point clouds\\[1.5mm]
\multirow{2}{*}{\qquad$\textbf{Img}_n\big({\bf X}_k,{\bf X}_{k'}|L\big)$} & Defined by~\eqref{eq:Image_Line}, $L\times 1$ vector collecting all values on $\textbf{Img}_n$ in the\\
~ & straight line connecting image points ${\bf x}^{\text{\tiny(n)}}_k$ and ${\bf x}^{\text{\tiny(n)}}_{k'}$ of ${\bf X}_k$ and ${\bf X}_{k'}$\\[1.5mm]
\qquad$\mathcal{E}_d({\bf X}), \mathcal{E}_i({\bf X})$ & Two metrics to evaluate a new created point ${\bf X}$\\[1.5mm]
\qquad$\theta, \ell$ & Two parameters in \emph{geo-feature} detection\\[1.5mm]
\qquad$\mathcal{C}\big({\bf X}_k,{\bf X}_{k'},\theta,\ell\big)$ & Collecting all candidates for geo-features w.r.t. $\{{\bf X}_k,{\bf X}_{k'},\theta,\ell\}$\\[1.5mm]
\emph{\normalsize Evaluation:} &~\\[1mm]
\qquad$\mathcal{PC}_{\delta}$ & Refinements of $\mathcal{PC}$ corresponding to distance $\delta$\\[1.5mm]
\qquad$\mathcal{PC}_{\epsilon}$ & Refinements of $\mathcal{PC}$ corresponding to pixel $\epsilon$\\[1.5mm]
\qquad$\text{Error}_{\text{Rel}}\big(\{{\bf X}^{\star}_m,{\bf P}^{\star}_n\}\big)$ & Defined by~\eqref{eq:Relative_Eval_WP_PM}, used to evaluate Algorithm~\ref{Alg:Iteration}, \texttt{WPfC}\\[1.5mm]
\qquad$\text{Error}_{\text{Rel}}\big({\bf X}\big)$ & Defined by~\eqref{eq:Relative_Eval_PClouds}, used to evaluate Algorithm~\ref{Alg:Point_Clouds}, \texttt{CrPC}\\[1.5mm]
\hline
\end{tabular}}
\caption{\emph{Notations and Explanations}}
\label{Tab:Notations}
\end{table}

\textbf{Roadmap.}\ Most important notations of the paper and their explanations are given by the Table~\ref{Tab:Notations} in the next page. The first task of estimating projection matrices is studied in Section~\ref{Sec:World_points} and the second task of creating point clouds is presented by Section~\ref{Sec:Point_Clouds}. In Section~\ref{Sec:World_points}, we first discuss the ambiguities of image-based 3D reconstruction and handle them by Proposition~\ref{Prop:Ambiguity}. The closed-form solutions for the world points and projection matrices from their correspondences are derived by Theorem~\ref{Theo:Sol_WPfC_NoNoise} and Algorithm~\ref{Alg:Closed_form} in Subsection~\ref{subsec:Sol_Corres}. Subsection~\ref{subsec:Sol_EstCorres} starts by introducing the optimization~\eqref{eq:Opt_WPfC} to solve the \emph{World points from their Correspondences} problem when the correspondences have noise. The global optimal solutions for this optimization are derived for the case $M \leq 6$ by Propositions~\ref{Prop:Global_Sol} and~\ref{Prop:CF_Sol_6}. As $M \geq 7$, Proposition~\ref{Prop:Opt_by_Iso} brings the better solution than the closed-forms given by Theorem~\ref{Theo:Sol_WPfC_NoNoise} and Proposition~\ref{Prop:Opt_by_Iter} gives the better solution than the one from Proposition~\ref{Prop:Opt_by_Iso}. In addition, the results from Propositions~\ref{Prop:Opt_by_Iso} and~\ref{Prop:Opt_by_Iter} allow us to propose an iterative process to find the local optimal solutions from the initial closed-from solutions given by Theorem~\ref{Theo:Sol_WPfC_NoNoise}. This iterative process is Algorithm~\ref{Alg:Iteration}, named by \texttt{WPfC}. In Section~\ref{Sec:Point_Clouds}, we introduce the concept of \emph{geodesic feature} or \emph{geo-feature} in the Subsection~\ref{Subsec:Geo_feature} then propose an algorithm to detect geo-features and creating the point clouds in Subsection~\ref{Subsec:Alg}. The most important result in Section~\ref{Sec:Point_Clouds} is Algorithm~\ref{Alg:Point_Clouds}, named by \texttt{CrPC}. This algorithm creates new world points for the point clouds based on two old world points and the projection matrices of all images. Since the results in Sections~\ref{Sec:World_points} and~\ref{Sec:Point_Clouds} are only materials and Algorithms \texttt{WPfC} and \texttt{CrPC} are tools only, to complete the process of building point clouds from a few accurate correspondences we need a protocol. Through the \emph{front body experiment} in the 3D Buddha statue reconstruction project, Section~\ref{Sec:PC_from_Correspondences} introduces the proposed image-based 3D reconstruction protocol. Finally, Section~\ref{Sec:Evaluation} gives an evaluation of our work based on the 3D Buddha statue reconstruction project and Section~\ref{Sec:Conclusion} concludes this article.

\section{World points from Correspondences}\label{Sec:World_points}
Let
\begin{equation*}
\big\{{\bf x}^{\text{\tiny(1)}}_m\leftrightarrow{\bf x}^{\text{\tiny(2)}}_m\leftrightarrow{\bf x}^{\text{\tiny(3)}}_m\leftrightarrow\cdots\leftrightarrow{\bf x}^{\text{\tiny(N)}}_m\big\}^M_{m=1}
\end{equation*}
be \emph{correspondences} of $M$ unknown \emph{world points} ${\bf X}_1$, ${\bf X}_2$, $\ldots$, ${\bf X}_M$ capturing by $N$ unknown \emph{cameras} ${\bf P}_1, {\bf P}_2, \ldots, {\bf P}_N$. The world point ${\bf X}_m$ is in three-dimensional space $\mathbb{R}^3$ and indicated by three coordinates $x_m, y_m$ and $z_m$ and the camera or the projection matrix ${\bf P}_n$ is the $3\times 4$ matrix denoted by
\begin{equation*}
{\bf P}_n \defeq \begin{bmatrix}p_{1n} & p_{2n} & p_{3n} & p_{4n}\\
p_{5n} & p_{6n} & p_{7n} & p_{8n}\\
p_{9n} & p_{10n} & p_{11n} & p_{12n}\end{bmatrix}_{3\times 4} \defeq~ \begin{bmatrix}{\bf p}_{1,n}\\{\bf p}_{2,n}\\{\bf p}_{3,n}\end{bmatrix}
\end{equation*}
From the \emph{multiple view geometry} theory~\citep{Hartley2004}, the relationship between the correspondences, world points and projection matrices are given by nonlinear equations
\begin{equation}\label{eq:parallel}
\qquad\qquad {\bf P}_n\begin{bmatrix}{\bf X}_m\\1\end{bmatrix} \sim \begin{bmatrix}{\bf x}^{\text{\tiny(n)}}_{m}\\1\end{bmatrix} \qquad\quad \text{for all}\quad m,n\,,
\end{equation}
where ${\bf a} \sim {\bf b}$ means that two vectors ${\bf a}$ and ${\bf b}$ are parallel. In this section we study the problem that estimating both the unknown world points ${\bf X}_1, {\bf X}_2, \ldots, {\bf X}_M$ and the unknown projection matrices ${\bf P}_1, {\bf P}_2, \ldots, {\bf P}_N$ from the known key-points $\{{\bf x}^{\text{\tiny(n)}}_m\}$ based on the nonlinear equations~\eqref{eq:parallel}. This problem is named \emph{the world points from their correspondences} (WPfC). If the world points $\{{\bf X}_m\}$ and the projection matrices $\{{\bf P}_n\}$ satisfy~\eqref{eq:parallel} corresponding to the correspondences $\{{\bf x}_m^{\text{\tiny(1)}}\leftrightarrow{\bf x}_m^{\text{\tiny(2)}}\leftrightarrow\cdots\leftrightarrow{\bf x}_m^{\text{\tiny(N)}}\}^M_{m=1}$, we call the group $\{{\bf X}_m,{\bf P}_n\}$ be a \emph{solution} with respect to (w.r.t.) these correspondences. Vice versa, if there exist a group $\{{\bf X}_m,{\bf P}_n\}$ satisfying~\eqref{eq:parallel}, $\{{\bf x}_m^{\text{\tiny(1)}}\leftrightarrow{\bf x}_m^{\text{\tiny(2)}}\leftrightarrow\cdots\leftrightarrow{\bf x}_m^{\text{\tiny(N)}}\}^M_{m=1}$ are called the \emph{correspondences}.

\newpage

Before go to further on finding solutions for the world points and the projection matrices based on the correspondences, we mention that when the world points are given, the solution for the projection matrix ${\bf P}_n$ is derived by the following linear equations
\begin{equation}\label{eq:Camera_from_Point}
\begin{bmatrix}
\begin{matrix}~{\bf X}_1~ & {\bf 0}\\1 & 0\\[0.5mm]{\bf 0} & ~{\bf X}_1~\\0 & 1\end{matrix} & \begin{matrix}~{\bf X}_2~ & {\bf 0}\\1 & 0\\[1mm]{\bf 0} & ~{\bf X}_2~\\0 & 1\end{matrix} & \begin{matrix}\cdots\\\cdots\\[0.5mm]\cdots\\\cdots\end{matrix} & \begin{matrix}~{\bf X}_M~ & {\bf 0}\\1 & 0\\[0.5mm]{\bf 0} & ~{\bf X}_M~\\0 & 1\end{matrix}\\[7mm]
-{\bf x}_1^{\text{\tiny(n)} T}{\bf X}_1 & -{\bf x}_2^{\text{\tiny(n)} T}{\bf X}_2 & \cdots & -{\bf x}_M^{\text{\tiny(n)} T}{\bf X}_M\\[2mm]
-{\bf x}_1^{\text{\tiny(n)} T} & -{\bf x}_2^{\text{\tiny(n)} T} & \cdots & -{\bf x}_M^{\text{\tiny(n)} T}
\end{bmatrix}^T_{12\times 2M}\begin{bmatrix}~{\bf p}_{1,n}^T~\\[3mm]{\bf  p}_{2,n}^T\\[3mm]{\bf p}_{3,n}^T\end{bmatrix} = {\bf 0}_{2M\times 1\, ,}
\end{equation}
for all $n$. Since the correspondences ${\bf x}_m^{\text{\tiny (n)}}$ are always known, we denote the $12\times 2M$ matrix in the left side of~\eqref{eq:Camera_from_Point} by $\mathcal{K}_n(\{{\bf X}_m\})$ as the function of the world points. If $M\geq 6$, the column vector containing all elements of ${\bf P}_n$ will be the eigenvector corresponding to the smallest eigenvalue of the positive semidefinite matrix $\big[\mathcal{K}_n(\{{\bf X}_m\})\mathcal{K}_n(\{{\bf X}_m\})^T\big]$. Note that we can fix $p_{12n} = 1$.

Vice versa, when the projection matrices $\{{\bf P}_n\}$ are given, the solution for the world point ${\bf X}_m$ is given by
\begin{equation}\label{eq:Point_from_Camera}
\begin{bmatrix}
\begin{bmatrix}{\bf p}_{1,1}\\{\bf p}_{2,1}\end{bmatrix} - {\bf x}_m^{\text{\tiny(1)}}{\bf p}_{3,1}~~~\\[5mm]
\begin{bmatrix}{\bf p}_{1,2}\\{\bf p}_{2,2}\end{bmatrix} - {\bf x}_m^{\text{\tiny(2)}}{\bf p}_{3,2}~~~\\\vdots~~~~\\
\begin{bmatrix}{\bf p}_{1,N}\\{\bf p}_{2,N}\end{bmatrix} - {\bf x}_m^{\text{\tiny(N)}}{\bf p}_{3,N}
\end{bmatrix}_{2N\times 4}\begin{bmatrix}{\bf X}_m\\[1mm]1\end{bmatrix} ~=~ {\bf 0}_{2N\times 1\, .}
\end{equation}
We denote the $2N\times 4$ matrix in the left side of~\eqref{eq:Point_from_Camera} by $\mathcal{H}_m(\{{\bf P}_n\})$ as the function of the projection matrices. When $N\geq 2$, the world point ${\bf X}_m$ will be the three first coordinates of the eigenvector corresponding to the smallest eigenvalue of the positive semidefinite matrix $\big[\mathcal{H}_m(\{{\bf P}_n\})^T\mathcal{H}_m(\{{\bf P}_n\})\big]$.

The solutions for the projection matrices from the world points given by~\eqref{eq:Camera_from_Point} and the solutions for the world points from the projection matrices in~\eqref{eq:Point_from_Camera} confirm that the \emph{degrees of freedom} (DoFs) of WPfC should be smaller than $\min\{3M,11N\}$. In the next subsection, by studying an ambiguity of WPfC, we show that the DoFs of this problem should be not larger than $3M-15$. In addition, when there are only six world points ($M = 6$), this DoFs is at most three and the WPfC problem is solvable.

\subsection{Ambiguity}\label{subsec:Ambiguity}

From~\eqref{eq:parallel} we easily see that if $\{{\bf X}_m,{\bf P}_n\}$ is a solution of the correspondences $\{{\bf x}_m^{\text{\tiny(1)}}\leftrightarrow\cdots\leftrightarrow{\bf x}_m^{\text{\tiny(N)}}\}$ and $p_{12n}\neq 0$, the new group given by
\begin{equation*}
\qquad\qquad\widehat{\bf X}_m = {\bf X}_m\qquad \text{and} \qquad \widehat{\bf P}_n = \frac{1}{p_{12n}}{\bf P}_n\, ,\qquad\quad \text{for all}\quad m,n\,,
\end{equation*}
then
\begin{equation*}
    \qquad\qquad\widehat{\bf P}_n\begin{bmatrix}\widehat{\bf X}_m\\1\end{bmatrix} = \frac{1}{p_{12n}}{\bf P}_n\begin{bmatrix}{\bf X}_m\\1\end{bmatrix} \sim \begin{bmatrix}{\bf x}_m^{\text{\tiny(n)}}\\1\end{bmatrix}\,,\qquad\quad\text{for all}\quad m,n\,.
\end{equation*}
It means that $\{\widehat{\bf X}_m,\widehat{\bf P}_n\}$ is another solution w.r.t. $\{{\bf x}_m^{\text{\tiny(1)}}\leftrightarrow\cdots\leftrightarrow{\bf x}_m^{\text{\tiny(N)}}\}$. Thus, to fix this ambiguity, we can assume that the (3,4)th element of the projection matrices ${\bf P}_n$ is unit for all $n$. In addition, by using any rotation or reflection ${\bf R}$ in three-dimensional space to transfer the world points and the projection matrices as follows
\begin{equation*}
    \widetilde{\bf X}_m = {\bf R}{\bf X}_m\qquad \text{and} \qquad
    \widetilde{\bf P}_n = \begin{bmatrix}\begin{bmatrix}p_{1n} & p_{2n} & p_{3n}\\p_{5n} & p_{6n} & p_{7n}\\p_{9n} & p_{10n} & p_{11n}\end{bmatrix}{\bf R}^{-1} & \begin{matrix}p_{4n}\\p_{8n}\\p_{12n}\end{matrix}\end{bmatrix}
\end{equation*}
then the group $\{\widetilde{\bf X}_m,\widetilde{\bf P}_n\}$ is also a solution because of
\begin{equation*}
    \widetilde{\bf P}_n\begin{bmatrix}\widetilde{\bf X}_m\\1\end{bmatrix} = {\bf P}_n\begin{bmatrix}{\bf X}_m\\1\end{bmatrix} \sim \begin{bmatrix}{\bf x}_m^{\text{\tiny(n)}}\\1\end{bmatrix}\, .
\end{equation*}
In another word, the world points from their correspondences is invariant under any rotations and reflections.

In general,~\cite{Hartley2004} shows that WPfC is invariant under any \emph{projective transformation}, a linear transformation that transfers points into points, lines into lines, planes into planes, and any two incident elements into two incident elements. Precisely, let ${\bf T}$ be the \emph{projective transformation} on the three-dimensional space, a non-singular $4\times4$ matrix with 15 degrees of freedom~\citep{Hartley2004}, we create a new group $\{\Breve{\bf X}_m,\Breve{\bf P}_n\}$ as follows
\begin{equation*}
\qquad\qquad\begin{bmatrix}\Breve{\bf X}_m\\1\end{bmatrix} = {\bf T}\begin{bmatrix}{\bf X}_m\\1\end{bmatrix}\qquad \text{and} \qquad \Breve{\bf P}_n = {\bf P}_n{\bf T}^{-1} \qquad\quad \text{for all}\quad m, n\, .
\end{equation*}
The new group will be a solution w.r.t. the correspondences because of the identity of ${\bf T}^{-1}{\bf T}$.

A projective transformation on the three-dimensional space is uniquely defined by five projected points, unless four of them are in the same plane. Based on this property and the invariant property of the world points from their correspondences under any projective transformation, the following result is derived to describe clearly the ambiguities of WPfC via the world points and the projection matrices.

\begin{prop}\label{Prop:Ambiguity}
Given the correspondences $\big\{{\bf x}_m^{\text{\tiny(1)}}\leftrightarrow{\bf x}_m^{\text{\tiny(2)}}\leftrightarrow\cdots\leftrightarrow{\bf x}_m^{\text{\tiny(N)}}\big\}^M_{m=1}$ with $M \geq 5$, there always exists a solution $\{{\bf X}_m,{\bf P}_n\}$ w.r.t. these correspondences and
\begin{equation}\label{eq:Ambi_total}
\begin{bmatrix}{\bf X}_1 & {\bf X}_2 & {\bf X}_3 & {\bf X}_4 & {\bf X}_5\end{bmatrix} = \begin{bmatrix}
0 & 0 & 0 & 1 & 1\\
0 & 0 & 1 & 0 & 1\\
0 & 1 & 0 & 0 & 1\end{bmatrix}\quad \text{and} \quad p_{12n} = 1\, .
\end{equation}
\end{prop}
\begin{proof}
Since $\{{\bf x}_m^{\text{\tiny(1)}}\leftrightarrow{\bf x}_m^{\text{\tiny(2)}}\leftrightarrow\cdots\leftrightarrow{\bf x}_m^{\text{\tiny(N)}}\}^M_{m=1}$ are the correspondences, there exists a solution $\{{\bf X}_m^{\circ},{\bf P}_n^{\circ}\}$ such that
\begin{equation*}
    \qquad\qquad{\bf P}_n^{\circ}\begin{bmatrix}{\bf X}_m^{\circ}\\1\end{bmatrix} \sim \begin{bmatrix}{\bf x}_m^{\text{\tiny(n)}}\\1\end{bmatrix} \qquad\quad \text{for all} \quad m, n\, .
\end{equation*}
From $\{{\bf X}_m^{\circ }\}$, we find a projective transformation $\widehat{\bf T}\defeq \begin{bmatrix}
    \hat{t}_{11} & \hat{t}_{12} & \hat{t}_{13} & \hat{t}_{14}\\
    \hat{t}_{21} & \hat{t}_{22} & \hat{t}_{23} & \hat{t}_{24}\\
    \hat{t}_{31} & \hat{t}_{32} & \hat{t}_{33} & \hat{t}_{34}\\
    \hat{t}_{41} & \hat{t}_{42} & \hat{t}_{43} & 1
    \end{bmatrix}$ such that
\begin{equation}\label{eq:Prop1_01}
\widehat{\bf T}\begin{bmatrix}{\bf X}^{\circ}_1 & {\bf X}^\circ_2 & {\bf X}^\circ_3 & {\bf X}^\circ_4 & {\bf X}^\circ_5\\[2mm]1 & 1 & 1 & 1 & 1\end{bmatrix} = \begin{bmatrix}
0 & 0 & 0 & \alpha_3 & \alpha_4\\
0 & 0 & \alpha_2 & 0 & \alpha_4\\
0 & \alpha_1 & 0 & 0 & \alpha_4\\
1 & \alpha_1 & \alpha_2 & \alpha_3 & \alpha_5\end{bmatrix}
\end{equation}
where $\alpha_1, \alpha_2, \alpha_3, \alpha_4$ and $\alpha_5$ are unknown and nonzero real numbers. Using linear transformations to remove unknown parameters $\alpha_1, \alpha_2, \alpha_3, \alpha_4, \alpha_5$, the matrix form~\eqref{eq:Prop1_01} is presented by linear equations of fifteen variables $\hat{t}_{ij}$ in $\widehat{\bf T}$ as follows
\begin{equation}\label{eq:Sol_T}
\begin{bmatrix}
{\bf X}_1^{\circ T} & 1 & {\bf 0}^T & 0 & {\bf 0}^T & 0 & {\bf 0}^T\\[1mm]
{\bf 0}^T & 0 & {\bf X}_1^{\circ T} & 1 & {\bf 0}^T & 0 & {\bf 0}^T\\[1mm]
{\bf 0}^T & 0 & {\bf 0}^T & 0 & {\bf X}_1^{\circ T} & 1 & {\bf 0}^T\\[1mm]
{\bf 0}^T & 0 & {\bf 0}^T & 0 & {\bf 0}^T & 0 & {\bf X}_1^{\circ T}\\[1mm]
{\bf X}_2^{\circ T} & 1 & {\bf 0}^T & 0 & {\bf 0}^T & 0 & {\bf 0}^T\\[1mm]
{\bf 0}^T & 0 & {\bf X}_2^{\circ T} & 1 & {\bf 0}^T & 0 & {\bf 0}^T\\[1mm]
{\bf 0}^T & 0 & {\bf 0}^T & 0 & {\bf X}_2^{\circ T} & 1 & -{\bf X}_2^{\circ T}\\[1mm]
{\bf X}_3^{\circ T} & 1 & {\bf 0}^T & 0 & {\bf 0}^T & 0 & {\bf 0}^T\\[1mm]
{\bf 0}^T & 0 & {\bf X}_3^{\circ T} & 1 & {\bf 0}^T & 0 & -{\bf X}_3^{\circ T}\\[1mm]
{\bf 0}^T & 0 & {\bf 0}^T & 0 & {\bf X}_3^{\circ T} & 1 & {\bf 0}^T\\[1mm]
{\bf X}_4^{\circ T} & 1 & {\bf 0}^T & 0 & {\bf 0}^T & 0 & -{\bf X}_4^{\circ T}\\[1mm]
{\bf 0}^T & 0 & {\bf X}_4^{\circ T} & 1 & {\bf 0}^T & 0 & {\bf 0}^T\\[1mm]
{\bf 0}^T & 0 & {\bf 0}^T & 0 & {\bf X}_4^{\circ T} & 1 & {\bf 0}^T\\[1mm]
{\bf X}_5^{\circ T} & 1 & -{\bf X}_5^{\circ T} & -1 & {\bf 0}^T & 0 & {\bf 0}^T\\[1mm]
{\bf X}_5^{\circ T} & 1 & {\bf 0}^T & 0 & -{\bf X}_5^{\circ T} & -1 & {\bf 0}^T
\end{bmatrix}_{15\times 15}
\begin{bmatrix}~\hat{t}_{11}~\\[1mm]\hat{t}_{12}\\[1mm]\hat{t}_{13}\\[1mm]\hat{t}_{14}\\[1mm]\hat{t}_{21}\\[1mm]\hat{t}_{22}\\[1mm]\hat{t}_{23}\\[1mm]\hat{t}_{24}\\[1mm]\hat{t}_{31}\\[1mm]\hat{t}_{32}\\[1mm]\hat{t}_{33}\\[1mm]\hat{t}_{34}\\[1mm]\hat{t}_{41}\\[1mm]\hat{t}_{42}\\[1mm]\hat{t}_{43}\end{bmatrix} = \begin{bmatrix}~0~\\[1mm]0\\[1mm]0\\[1mm]0\\[1mm]0\\[1mm]0\\[1mm]1\\[1mm]0\\[1mm]1\\[1mm]0\\[1mm]1\\[1mm]0\\[1mm]0\\[1mm]0\\[1mm]0\end{bmatrix}
\end{equation}
where ${\bf 0}$ denotes the zero vector in three-dimensional space. When the square matrix in the left side of~\eqref{eq:Sol_T} is invertible, which is almost the case in our study, the solution for the projective transformation $\widehat{\bf T}$ is easily derived to satisfy~\eqref{eq:Prop1_01}.

Let
\begin{equation*}
    \qquad\qquad\begin{bmatrix}\widehat{\bf X}_m\\1\end{bmatrix} = \widehat{\bf T}\begin{bmatrix}{\bf X}_m^{\circ}\\ 1\end{bmatrix} \quad \text{and} \quad \widehat{\bf P}_n = {\bf P}_n^{\circ}\widehat{\bf T}^{-1}\qquad\quad \text{for all} \quad m, n\, .
\end{equation*}
Because $\widehat{\bf T}$ is determined by five points from $\{{\bf X}_1^{\circ}, {\bf X}_2^{\circ}, \ldots,{\bf X}_5^{\circ}\}$ to $\{\widehat{\bf X}_1,\widehat{\bf X}_2,\ldots,\widehat{\bf X}_5\}$, it is a projective transformation. Thus, $\{\widehat{\bf X}_m,\widehat{\bf P}_n\}$ is the solution w.r.t. $\{{\bf x}_m^{\text{\tiny(1)}}\leftrightarrow{\bf x}_m^{\text{\tiny(2)}}\leftrightarrow\cdots\leftrightarrow{\bf x}_m^{\text{\tiny(N)}}\}^M_{m=1}$. In addition,~\eqref{eq:Prop1_01} yields
\begin{equation}\label{eq:Prop1_02}
\begin{bmatrix}\widehat{\bf X}_1 & \widehat{\bf X}_2 & \widehat{\bf X}_3 & \widehat{\bf X}_4 & \widehat{\bf X}_5\end{bmatrix} = \begin{bmatrix}
0 & 0 & 0 & 1 & \alpha\\
0 & 0 & 1 & 0 & \alpha\\
0 & 1 & 0 & 0 & \alpha
\end{bmatrix}\, ,
\end{equation}
where $\alpha = \alpha_4/\alpha_5$ is non-zero. We continue transforming the solution $\{\widehat{\bf X}_m,\widehat{\bf P}_n\}$ to obtain the solution $\{{\bf X}_m,{\bf P}_n\}$ by the following projective transformation
\begin{equation*}
    {\bf T} = \begin{bmatrix}
    (3\alpha-1)/(2\alpha) & 0 & 0 & 0\\
    0 & (3\alpha-1)/(2\alpha) & 0 & 0\\
    0 & 0 & (3\alpha-1)/(2\alpha) & 0\\
    (\alpha-1)/(2\alpha) & (\alpha-1)/(2\alpha) & (\alpha-1)/(2\alpha) & 1
    \end{bmatrix}\, .
\end{equation*}
The five points $\{{\bf X}_1, {\bf X}_2,\ldots,{\bf X}_5\}$ are determined by $\{\widehat{\bf X}_1,\widehat{\bf X}_2,\ldots,\widehat{\bf X}_5\}$ as follows
\begin{equation*}
\begin{split}
    {\bf T}\begin{bmatrix}\widehat{\bf X}_1 & \widehat{\bf X}_2 & \widehat{\bf X}_3 & \widehat{\bf X}_4 & \widehat{\bf X}_5\\1 & 1 & 1 & 1 & 1\end{bmatrix}
    &= \begin{bmatrix}
    (3\alpha-1)/(2\alpha) & 0 & 0 & 0\\[1mm]
    0 & (3\alpha-1)/(2\alpha) & 0 & 0\\[1mm]
    0 & 0 & (3\alpha-1)/(2\alpha) & 0\\[1mm]
    (\alpha-1)/(2\alpha) & (\alpha-1)/(2\alpha) & (\alpha-1)/(2\alpha) & 1
    \end{bmatrix}\begin{bmatrix}
    0 & 0 & 0 & 1 & \alpha\\[1mm]
    0 & 0 & 1 & 0 & \alpha\\[1mm]
    0 & 1 & 0 & 0 & \alpha\\[1mm]
    1 & 1 & 1 & 1 & 1
    \end{bmatrix}\\[3mm]
    &= \begin{bmatrix}
    0 & 0 & 0 & (3\alpha-1)/(2\alpha) & (3\alpha+1)/2\\[1mm]
    0 & 0 & (3\alpha-1)/(2\alpha) & 0 & (3\alpha+1)/2\\[1mm]
    0 & (3\alpha-1)/(2\alpha) & 0 & 0 & (3\alpha+1)/2\\[1mm]
    1 & (3\alpha-1)/(2\alpha) & (3\alpha-1)/(2\alpha) & (3\alpha-1)/(2\alpha) & (3\alpha+1)/2
    \end{bmatrix}\\[3mm]
    &\sim\begin{bmatrix}
    0 & 0 & 0 & 1 & 1\\
    0 & 0 & 1 & 0 & 1\\
    0 & 1 & 0 & 0 & 1\\
    1 & 1 & 1 & 1 & 1
    \end{bmatrix}\, .
\end{split}
\end{equation*}
It means
\begin{equation*}
\begin{bmatrix}{\bf X}_1 & {\bf X}_2 & {\bf X}_3 & {\bf X}_4 & {\bf X}_5\end{bmatrix} = \begin{bmatrix}
0 & 0 & 0 & 1 & 1\\
0 & 0 & 1 & 0 & 1\\
0 & 1 & 0 & 0 & 1\end{bmatrix}\,.
\end{equation*}
The remain condition $p_{12n} = 1$ we always obtain from the ambiguity of the projection matrices. Hence, Proposition~\ref{Prop:Ambiguity} is proven.
\end{proof}

From Proposition~\ref{Prop:Ambiguity}, we can find solutions $\{{\bf X}_m,{\bf P}_n\}$ w.r.t. the correspondences $\{{\bf x}_m^{\text{\tiny(1)}}\leftrightarrow{\bf x}_m^{\text{\tiny(2)}}\leftrightarrow\cdots\leftrightarrow{\bf x}_m^{\text{\tiny(N)}}\}^M_{m=1}$ within a setting given by~\eqref{eq:Ambi_total}. Indeed, the DoFs of the problem is now upper bound by $\min\{3M-15,11N\}$. Specially, when there are only six world points ($M=6$), the DoFs is at most three that is possible to solve. Using the Theorem II given by~\citet{le2022multi}, the next subsection shows that the solution for the world points ${\bf X}_m~(m\geq 6)$ is easily derived based on the setting~\eqref{eq:Ambi_total}. From that \emph{closed-form} solutions for all world points and projection matrices are derived.

\subsection{Solutions from Correspondences}\label{subsec:Sol_Corres}

We begin this subsection by reminding the Theorem II in our previous article~\citep{le2022multi}. First, for any three key-points ${\bf x}_{m_1}^{\text{\tiny(n)}}$, ${\bf x}_{m_2}^{\text{\tiny(n)}}$ and ${\bf x}_{m_3}^{\text{\tiny(n)}}$ in the same image $n$, we define the value $\xi_{m_1m_2,m_3}^{\text{\tiny(n)}}$ as a determinant of a $2\times 2$ matrix generating by two column vectors ${\bf x}_{m_3}^{\text{\tiny(n)}}-{\bf x}_{m_1}^{\text{\tiny(n)}}$ and ${\bf x}_{m_3}^{\text{\tiny(n)}}-{\bf x}_{m_2}^{\text{\tiny(n)}}$, i.e.,
\begin{equation}\label{eq:Xi}
    \xi_{m_1m_2,m_3}^{\text{\tiny(n)}} \defeq \text{det}\left(\begin{bmatrix}{\bf x}_{m_3}^{\text{\tiny (n)}} - {\bf x}_{m_1}^{\text{\tiny(n)}} & {\bf x}_{m_3}^{\text{\tiny (n)}} - {\bf x}_{m_2}^{\text{\tiny(n)}}\end{bmatrix}\right)\, .
\end{equation}
Now, given sub-correspondences $\{{\bf x}_i^{\text{\tiny(1)}}\leftrightarrow{\bf x}_i^{\text{\tiny(2)}}\leftrightarrow\cdots\leftrightarrow{\bf x}_i^{\text{\tiny(N)}}\}_{i=1,2,3,4,5,m}$ from only six world points ${\bf X}_1$, ${\bf X}_2$, ${\bf X}_3$, ${\bf X}_4$, ${\bf X}_5$ and ${\bf X}_m$ ($6\leq m \leq M$) on $N$ images ($N\geq 5$), we consider a $N\times 5$ matrix that
\begin{equation}\label{eq:Lambda}
{\bf \Lambda}_m ~\defeq~\begin{bmatrix}
\xi^{\text{\tiny(1)}}_{34,5}\xi^{\text{\tiny(1)}}_{12,m} & \xi^{\text{\tiny(1)}}_{42,5}\xi^{\text{\tiny(1)}}_{13,m} & \xi^{\text{\tiny(1)}}_{23,5}\xi^{\text{\tiny(1)}}_{14,m} &
\xi^{\text{\tiny(1)}}_{12,5}\xi^{\text{\tiny(1)}}_{34,m} & \xi^{\text{\tiny(1)}}_{13,5}\xi^{\text{\tiny(1)}}_{42,m}\\[3mm]
\xi^{\text{\tiny(2)}}_{34,5}\xi^{\text{\tiny(2)}}_{12,m} & \xi^{\text{\tiny(2)}}_{42,5}\xi^{\text{\tiny(2)}}_{13,m} & \xi^{\text{\tiny(2)}}_{23,5}\xi^{\text{\tiny(2)}}_{14,m} &
\xi^{\text{\tiny(2)}}_{12,5}\xi^{\text{\tiny(2)}}_{34,m} & \xi^{\text{\tiny(2)}}_{13,5}\xi^{\text{\tiny(2)}}_{42,m}\\[3mm]
\xi^{\text{\tiny(3)}}_{34,5}\xi^{\text{\tiny(3)}}_{12,m} & \xi^{\text{\tiny(3)}}_{42,5}\xi^{\text{\tiny(3)}}_{13,m} & \xi^{\text{\tiny(3)}}_{23,5}\xi^{\text{\tiny(3)}}_{14,m} &
\xi^{\text{\tiny(3)}}_{12,5}\xi^{\text{\tiny(3)}}_{34,m} & \xi^{\text{\tiny(3)}}_{13,5}\xi^{\text{\tiny(3)}}_{42,m}\\[1mm]
\vdots & \vdots & \vdots & \vdots & \vdots\\[1.5mm]
\xi^{\text{\tiny(N)}}_{34,5}\xi^{\text{\tiny(N)}}_{12,m} & \xi^{\text{\tiny(N)}}_{42,5}\xi^{\text{\tiny(N)}}_{13,m} & \xi^{\text{\tiny(N)}}_{23,5}\xi^{\text{\tiny(N)}}_{14,m} &
\xi^{\text{\tiny(N)}}_{12,5}\xi^{\text{\tiny(N)}}_{34,m} & \xi^{\text{\tiny(N)}}_{13,5}\xi^{\text{\tiny(N)}}_{42,m}
\end{bmatrix}_{N\times 5\, .}
\end{equation}

\newpage

Let ${\bf e}_m = [e_{1m},e_{2m},e_{3m},e_{4m},e_{5m}]^T$ be the eigenvector corresponding to the smallest eigenvalue of the $5\times 5$ positive semidefinite matrix ${\bf \Lambda}_m^T{\bf \Lambda}_m$, and let $x_m,y_m,z_m$ be three coordinates of the world points ${\bf X}_m$. Thanks to the setting in~\eqref{eq:Ambi_total}, Theorem II in~\citet{le2022multi} proves that
\begin{equation}\label{eq:WP_from_Eigen}
\begin{bmatrix}
z_m(1-x_m-y_m-z_m) + 2y_mz_m\\[1mm]
y_m(1-x_m-y_m-z_m) + 2y_mz_m\\[1mm]
x_m(1-x_m-y_m-z_m) + 2y_mz_m\\[1mm]
-2x_my_m + 2y_mz_m\\[1mm]
-2x_mz_m + 2y_mz_m
\end{bmatrix} \sim \begin{bmatrix}
e_{1m}\\[1mm]e_{2m}\\[1mm]e_{3m}\\[1mm]e_{4m}\\[1mm]e_{5m}
\end{bmatrix}
\end{equation}
Hence, let
\begin{equation}\label{eq:f_Eigen}
\begin{split}
f_{m} &= (2e_{2m}e_{5m} + 2e_{1m}e_{4m} - 3e_{1m}e_{2m} - e_{4m}e_{5m})(e_{1m} + e_{4m} - e_{2m} - e_{5m})\\
&\hspace{60mm}+ (e_{2m}-e_{1m})(e_{5m}-e_{1m})(e_{2m}-e_{4m})\\[0.5mm]
f_{xm} &= (e_{1m} - e_{5m})(e_{5m} - e_{4m})(e_{2m} - e_{4m})\\[0.5mm]
f_{ym} &= (e_{2m}e_{5m} - e_{1m}e_{4m})(e_{2m} - e_{4m})\\[0.5mm]
f_{zm} &= (e_{2m}e_{5m} - e_{1m}e_{4m})(e_{1m} - e_{5m})
\end{split}
\end{equation}
\eqref{eq:WP_from_Eigen} easily implies
\begin{equation}\label{eq:Sol_WPoints}
x_m = \frac{f_{xm}}{f_{m}}\quad,\quad y_m = \frac{f_{ym}}{f_{m}}\quad,\quad z_m = \frac{f_{zm}}{f_{m}}\, .
\end{equation}
\eqref{eq:Sol_WPoints} gives us the exact solution for the world point ${\bf X}_m$ ($6\leq m \leq M)$ only based on ${\bf \Lambda}_m$ constructed by the correspondences from six world points ${\bf X}_1, \ldots,{\bf X}_5$ and ${\bf X}_m$.

When the solutions for all world points are acquired, the solution for the projection matrix ${\bf P}_n$ can be derived by~\eqref{eq:Camera_from_Point} via the matrix $\mathcal{K}_n(\{{\bf X}_m\})$. Using some linear transformations on the matrix $\mathcal{K}_n(\{{\bf X}_m\})$,~\eqref{eq:Camera_from_Point} becomes
\begin{equation}\label{eq:Camera_from_Point_1}
\begin{bmatrix}p_{1n} & p_{2n} & p_{3n} & p_{4n}\\[1mm]
p_{5n} & p_{6n} & p_{7n} & p_{8n}\end{bmatrix} = \begin{bmatrix}{\bf x}_4^{\text{\tiny(n)}} & {\bf x}_3^{\text{\tiny(n)}} & {\bf x}_2^{\text{\tiny(n)}} & {\bf x}_1^{\text{\tiny(n)}}\end{bmatrix}\begin{bmatrix}
p_{9n}+1 & 0 & 0 & 0\\
0 & p_{10n}+1 & 0 & 0\\
0 & 0 & p_{11n}+1 & 0\\
-1 & -1 & -1 & 1\end{bmatrix}
\end{equation}
and
\begin{equation}\label{eq:Camera_from_Point_2}
\begin{bmatrix}
{\bf x}_4^{\text{\tiny(n)}}-{\bf x}_5^{\text{\tiny(n)}} & {\bf x}_3^{\text{\tiny(n)}}-{\bf x}_5^{\text{\tiny(n)}} & {\bf x}_2^{\text{\tiny(n)}}-{\bf x}_5^{\text{\tiny(n)}} & {\bf x}_1^{\text{\tiny(n)}}-{\bf x}_5^{\text{\tiny(n)}}\\[1mm]
\frac{f_{x6}}{f_6}({\bf x}_4^{\text{\tiny(n)}}-{\bf x}_6^{\text{\tiny(n)}}) & \frac{f_{y6}}{f_6}({\bf x}_3^{\text{\tiny(n)}}-{\bf x}_6^{\text{\tiny(n)}}) & \frac{f_{z6}}{f_6}({\bf x}_2^{\text{\tiny(n)}}-{\bf x}_6^{\text{\tiny(n)}}) & {\bf x}_1^{\text{\tiny(n)}}-{\bf x}_6^{\text{\tiny(n)}}\\[1mm]
\vdots & \vdots & \vdots & \vdots\\[1mm]
\frac{f_{xM}}{f_M}({\bf x}_4^{\text{\tiny(n)}}-{\bf x}_M^{\text{\tiny(n)}}) & \frac{f_{yM}}{f_M}({\bf x}_3^{\text{\tiny(n)}}-{\bf x}_M^{\text{\tiny(n)}}) & \frac{f_{zM}}{f_M}({\bf x}_2^{\text{\tiny(n)}}-{\bf x}_M^{\text{\tiny(n)}}) & {\bf x}_1^{\text{\tiny(n)}}-{\bf x}_M^{\text{\tiny(n)}}\end{bmatrix}\begin{bmatrix}p_{9n}+1\\[1mm]p_{10n}+1\\[1mm]p_{11n}+1\\[1mm]-2\end{bmatrix} = {\bf 0}\, .
\end{equation}
The equations from the first row of the matrix in~\eqref{eq:Camera_from_Point_2} imply
\begin{equation}\label{eq:Camera_from_Point_3}
p_{9n}+1 = \frac{\xi_{23,5}^{\text{\tiny(n)}}(p_{11n}+1) - 2\xi_{13,5}^{\text{\tiny(n)}}}{\xi_{34,5}^{\text{\tiny(n)}}} \qquad \text{and} \qquad p_{10n}+1 = \frac{2\xi_{14,5}^{\text{\tiny(n)}} - \xi_{24,5}^{\text{\tiny(n)}}(p_{11n}+1)}{\xi_{34,5}^{\text{\tiny(n)}}}\, .
\end{equation}
Replacing the solutions for $p_{9n}$ and $p_{10n}$ in~\eqref{eq:Camera_from_Point_3} to the remain equations in~\eqref{eq:Camera_from_Point_2}, the final variable $p_{11n}$ in the matrix ${\bf P}_n$ is a solution of the following linear equations
\begin{equation}\label{eq:Camera_from_Point_4}
\begin{bmatrix}
f_{y6}\xi_{34,6}^{\text{\tiny(n)}}\xi_{24,5}^{\text{\tiny(n)}} - f_{z6}\xi_{24,6}^{\text{\tiny(n)}}\xi_{34,5}^{\text{\tiny(n)}} & (f_{x6} + f_{y6} + f_{z6} - f_6)\xi_{14,6}^{\text{\tiny(n)}}\xi_{34,5}^{\text{\tiny(n)}} - 2f_{y6}\xi_{34,6}^{\text{\tiny(n)}}\xi_{14,5}^{\text{\tiny(n)}}\\[2mm]
f_{y7}\xi_{34,7}^{\text{\tiny(n)}}\xi_{24,5}^{\text{\tiny(n)}} - f_{z7}\xi_{24,7}^{\text{\tiny(n)}}\xi_{34,5}^{\text{\tiny(n)}} & (f_{x7} + f_{y7} + f_{z7} - f_7)\xi_{14,7}^{\text{\tiny(n)}}\xi_{34,5}^{\text{\tiny(n)}} - 2f_{y7}\xi_{34,7}^{\text{\tiny(n)}}\xi_{14,5}^{\text{\tiny(n)}}\\
\vdots & \vdots\\[1mm]
f_{yM}\xi_{34,M}^{\text{\tiny(n)}}\xi_{24,5}^{\text{\tiny(n)}} - f_{zM}\xi_{24,M}^{\text{\tiny(n)}}\xi_{34,5}^{\text{\tiny(n)}} & (f_{xM} + f_{yM} + f_{zM} - f_M)\xi_{14,M}^{\text{\tiny(n)}}\xi_{34,5}^{\text{\tiny(n)}} - 2f_{yM}\xi_{34,M}^{\text{\tiny(n)}}\xi_{14,5}^{\text{\tiny(n)}}
\end{bmatrix}\begin{bmatrix}p_{11n} + 1\\[1mm] 1\end{bmatrix} = {\bf 0}\, .
\end{equation}
Concretely, the solution for the projection matrix ${\bf P}_n$ is given by

\begin{equation}\label{eq:Sol_Matrix}
    {\bf P}_n = \begin{bmatrix}{\bf x}_4^{\text{\tiny(n)}} & {\bf x}_3^{\text{\tiny(n)}} & {\bf x}_2^{\text{\tiny(n)}} & {\bf x}_1^{\text{\tiny(n)}}\\[2mm]1 & 1 & 1 & 1 \end{bmatrix}\begin{bmatrix}\gamma_n & 0 & 0 & 0\\0 & \beta_n & 0 & 0\\0 & 0 & \alpha_n & 0\\-1 & -1 & -1 & 1\end{bmatrix}
\end{equation}
where $\alpha_n = p_{11n}+1$ is the solution of~\eqref{eq:Camera_from_Point_4} and $\beta_n = p_{10n}+1$ and $\gamma_n = p_{9n}+1$ are given by~\eqref{eq:Camera_from_Point_3}.

Finally, the theorem below will formally present the solutions for the world points and the projection matrices of the WPfC problem.

\begin{theo}\label{Theo:Sol_WPfC_NoNoise}
Given correspondences $\big\{{\bf x}_m^{\text{\tiny(1)}}\leftrightarrow{\bf x}_m^{\text{\tiny(2)}}\leftrightarrow\cdots\leftrightarrow{\bf x}_m^{\text{\tiny(N)}}\big\}^M_{m=1}$ with $M\geq 6$ and $N\geq 5$, the following group $\{{\bf X}_m,{\bf P}_n\}$ with
\begin{subequations}\label{eq:Sol_WPfC}
\begin{flalign}
    &\begin{bmatrix}{\bf X}_1 & {\bf X}_2 & {\bf X}_3 & {\bf X}_4 & {\bf X}_5\end{bmatrix} ~=~ \begin{bmatrix}
    0 & 0 & 0 & 1 & 1\\
    0 & 0 & 1 & 0 & 1\\
    0 & 1 & 0 & 0 & 1
    \end{bmatrix}\label{eq:Sol_WPfC_a}\\[2mm]
    &~{\bf X}_m = \frac{1}{f_m}\begin{bmatrix}f_{xm}\\f_{ym}\\f_{zm}\end{bmatrix}\hspace{56mm} \text{for all}\quad 6\leq m \leq M\, ,\label{eq:Sol_WPfC_b}\\[2mm]
    &~{\bf P}_n = \begin{bmatrix}{\bf x}_4^{\text{\tiny(n)}} & {\bf x}_3^{\text{\tiny(n)}} & {\bf x}_2^{\text{\tiny(n)}} & {\bf x}_1^{\text{\tiny(n)}}\\[2mm]1 & 1 & 1 & 1 \end{bmatrix}\begin{bmatrix}\gamma_n & 0 & 0 & 0\\0 & \beta_n & 0 & 0\\0 & 0 & \alpha_n & 0\\-1 & -1 & -1 & 1\end{bmatrix}\qquad \text{for all}\quad 1\leq n \leq N\, ,\label{eq:Sol_WPfC_c}
\end{flalign}
\end{subequations}
is a solution w.r.t. these correspondences, where $f_m, f_{xm}, f_{ym}, f_{zm}$ are given by~\eqref{eq:f_Eigen}, $\alpha_n$ is given by~\eqref{eq:Camera_from_Point_4} and $\beta_n, \gamma_n$ are given by~\eqref{eq:Camera_from_Point_3}.
\end{theo}

\subsection{Solutions from Estimated Correspondences}\label{subsec:Sol_EstCorres}

In the last subsection, we solve the WPfC w.r.t. correspondences in which (i) the solution for the world point ${\bf X}_m~(m\geq 6)$ can be derived independently with others, and (ii) there always exists a solution for $\alpha_n = p_{11n}+1$ from~\eqref{eq:Camera_from_Point_4}. This subsection studies the WPfC w.r.t. \emph{estimated correspondences}. The estimated correspondences are not the correspondences that there exist a group of world points and projection matrices satisfying~\eqref{eq:parallel}, but is expected to exist another group  that~\eqref{eq:parallel} closely occurs. Informally, $\{\hat{\bf x}_m^{\text{\tiny(1)}}\leftrightarrow\hat{\bf x}_m^{\text{\tiny(2)}}\leftrightarrow\cdots\leftrightarrow\hat{\bf x}_m^{\text{\tiny(N)}}\}^M_{m=1}$ are called the \emph{estimated correspondences} if there exists the $M$ world points $\{{\bf X}_m\}$ and $N$ projection matrices $\{{\bf P}_n\}$ such that
\begin{equation}\label{eq:Def_Est_Correspondences}
\begin{bmatrix}{\bf p}_{1,n}\\{\bf p}_{2,n}\end{bmatrix}\begin{bmatrix}{\bf X}_m\\1\end{bmatrix}~\simeq~\hat{\bf x}_m^{\text{\tiny(n)}}{\bf p}_{3,n}\begin{bmatrix}{\bf X}_m\\1\end{bmatrix} \qquad \text{for all} \quad m,n\, ,
\end{equation}
where ${\bf p}_{i,n}$ denotes the $i$th row of ${\bf P}_n$.

From the vague definition of the estimated correspondences, we state the WPfC w.r.t. estimated correspondences as the following optimization problem. Given the estimated correspondences $\{\hat{\bf x}_m^{\text{\tiny(1)}}\leftrightarrow\hat{\bf x}_m^{\text{\tiny(2)}}\leftrightarrow\cdots\leftrightarrow\hat{\bf x}_m^{\text{\tiny(N)}}\}^M_{m=1}$, find the $M$ world points $\{{\bf X}_m\}$ and $N$ projection matrices $\{{\bf P}_n\}$ such that
\begin{equation}\label{eq:Opt_WPfC}
\sum^M_{m=1}\sum^N_{n=1}\left\{\Bigg(\frac{{\bf p}_{1,n}\big[{\bf X}_m^T\,,\,1\big]^T}{{\bf p}_{3,n}\big[{\bf X}_m^T\,,\,1\big]^T} - \hat{u}_{mn}\Bigg)^2 + \Bigg(\frac{{\bf p}_{2,n}\big[{\bf X}_m^T\,,\,1\big]^T}{{\bf p}_{3,n}\big[{\bf X}_m^T\,,\,1\big]^T} - \hat{v}_{mn}\Bigg)^2\right\}\quad\longrightarrow \quad \text{minimum}
\end{equation}
where $\hat{u}_{mn}, \hat{v}_{mn}$ are two coordinates of $\hat{\bf x}_m^{\text{\tiny(n)}}$. A value in the left side of~\eqref{eq:Opt_WPfC} is called the \emph{objective value} of $\{{\bf X}_m,{\bf P}_n\}$ and it is always non-negative. We denote this value by $d(\{{\bf X}_m,{\bf P}_n\})$. Thanks to the theory we study on the WPfC w.r.t. correspondences in the previous subsection, by finding the group $\{{\bf X}^{\star}_m,{\bf P}^{\star}_n\}$ such that their objective values are zero, a proposition below solves totally the optimization~\eqref{eq:Opt_WPfC}  for $M\leq 5$.

\newpage

\begin{prop}\label{Prop:Global_Sol}
When $M \leq 5$, there always exists a group $\{{\bf X}^{\star}_m,{\bf P}^{\star}_n\}$ such that its objective value is zero. It means $\{{\bf X}^{\star}_m,{\bf P}^{\star}_n\}$ is a global optimal solution for~\eqref{eq:Opt_WPfC}.
\end{prop}
\begin{proof}
We start the proof of Proposition~\ref{Prop:Global_Sol} with a mention that the ambiguity given by Proposition ~\ref{Prop:Ambiguity} is still correct for the WPfC w.r.t. estimated correspondences. Indeed, the global optimal solution $\{{\bf X}^{\star}_m,{\bf P}^{\star}_n\}$ for~\eqref{eq:Opt_WPfC} will be found based on the setting~\eqref{eq:Ambi_total} for the world points and
\begin{equation}\label{eq:Global_Sol_PM}
{\bf P}^{\star}_n = \begin{bmatrix}\hat{\bf x}_4^{\text{\tiny(n)}} & \hat{\bf x}_3^{\text{\tiny(n)}} & \hat{\bf x}_2^{\text{\tiny(n)}} & \hat{\bf x}_1^{\text{\tiny(n)}}\\[1mm]1 & 1 & 1 & 1\end{bmatrix}\begin{bmatrix}
\gamma_n & 0 & 0 & 0\\
0 & \beta_n & 0 & 0\\
0 & 0 & \alpha_n & 0\\
-1 & -1 & -1 & 1
\end{bmatrix}
\end{equation}
for the projection matrices. It easily sees that when $M = 1$ with the estimated correspondences $\{\hat{\bf x}_1^{\text{\tiny(n)}}\}^N_{n=1}$, we always have
\begin{equation*}
    {\bf P}^{\star}_n\begin{bmatrix}{\bf X}^{\star}_1\\1\end{bmatrix} = \begin{bmatrix}\hat{\bf x}_4^{\text{\tiny(n)}} & \hat{\bf x}_3^{\text{\tiny(n)}} & \hat{\bf x}_2^{\text{\tiny(n)}} & \hat{\bf x}_1^{\text{\tiny(n)}}\\[1mm]1 & 1 & 1 & 1\end{bmatrix}\begin{bmatrix}
\gamma_n & 0 & 0 & 0\\
0 & \beta_n & 0 & 0\\
0 & 0 & \alpha_n & 0\\
-1 & -1 & -1 & 1
\end{bmatrix}\begin{bmatrix}0\\0\\0\\1\end{bmatrix} = \begin{bmatrix}\hat{\bf x}_1^{\text{\tiny(n)}}\\[1mm]1\end{bmatrix}
\end{equation*}
for any values of the remain parameters $\hat{\bf x}_2^{\text{\tiny(n)}}, \hat{\bf x}_3^{\text{\tiny(n)}}, \hat{\bf x}_4^{\text{\tiny(n)}}$ and $\alpha_n, \beta_n, \gamma_n$. The equation above says that the objective value of $\{{\bf X}^{\star}_1,{\bf P}^{\star}_n\}$ is zero and $\{{\bf X}^{\star}_1,{\bf P}^{\star}_n\}$ is a global optimal solution for~\eqref{eq:Opt_WPfC} w.r.t. the estimated correspondences $\{\hat{\bf x}_1^{\text{\tiny(n)}}\}^N_{n=1}$. We do similar explanations for the case $M = 2, 3,$ and 4 by fixing the values of $\{\hat{\bf x}_1^{\text{\tiny(n)}}, \hat{\bf x}_2^{\text{\tiny(n)}}\}$ (for $M=2$), $\{\hat{\bf x}_1^{\text{\tiny(n)}}, \hat{\bf x}_2^{\text{\tiny(n)}}, \hat{\bf x}_3^{\text{\tiny(n)}}\}$ (for $M=3$), and $\{\hat{\bf x}_1^{\text{\tiny(n)}}, \hat{\bf x}_2^{\text{\tiny(n)}}, \hat{\bf x}_3^{\text{\tiny(n)}}, \hat{\bf x}_4^{\text{\tiny(n)}}\}$ (for $M=4$) in~\eqref{eq:Global_Sol_PM} from the estimated correspondences. For example, in the case $M=4$ with the estimated correspondences $\{\hat{\bf x}_1^{\text{\tiny(n)}}\,,\,\hat{\bf x}_2^{\text{\tiny(n)}}\,,\,\hat{\bf x}_3^{\text{\tiny(n)}}\,,\,\hat{\bf x}_4^{\text{\tiny(n)}}\}^N_{n=1}$, we can choose any values for $\alpha_n, \beta_n, \gamma_n$ and always guarantee that the zero value for the objective value of $\{{\bf X}^{\star}_1,{\bf X}^{\star}_2,{\bf X}^{\star}_3,{\bf X}^{\star}_4,{\bf P}^{\star}_n\}$ because
\begin{equation*}
    \begin{split}
        {\bf P}^{\star}_n\begin{bmatrix}{\bf X}^{\star}_1 & {\bf X}^{\star}_2 & {\bf X}^{\star}_3 & {\bf X}^{\star}_4\\[1mm]1 & 1 & 1 & 1\end{bmatrix} &= \begin{bmatrix}\hat{\bf x}_4^{\text{\tiny(n)}} & \hat{\bf x}_3^{\text{\tiny(n)}} & \hat{\bf x}_2^{\text{\tiny(n)}} & \hat{\bf x}_1^{\text{\tiny(n)}}\\[1mm]1 & 1 & 1 & 1\end{bmatrix}\begin{bmatrix}
\gamma_n & 0 & 0 & 0\\
0 & \beta_n & 0 & 0\\
0 & 0 & \alpha_n & 0\\
-1 & -1 & -1 & 1
\end{bmatrix}\begin{bmatrix}0 & 0 & 0 & 1\\0 & 0 & 1 & 0\\0 & 1 & 0 & 0\\1 & 1 & 1 & 1\end{bmatrix}\\[2mm]
&= \begin{bmatrix}\gamma_n\hat{\bf x}_4^{\text{\tiny(n)}} - \hat{\bf x}_1^{\text{\tiny(n)}} & \beta_n\hat{\bf x}_3^{\text{\tiny(n)}} - \hat{\bf x}_1^{\text{\tiny(n)}} & \alpha_n\hat{\bf x}_2^{\text{\tiny(n)}} - \hat{\bf x}_1^{\text{\tiny(n)}} & \hat{\bf x}_1^{\text{\tiny(n)}}\\[1mm]\gamma_n-1 & \beta_n-1 & \alpha_n-1 & 1\end{bmatrix}\begin{bmatrix}0 & 0 & 0 & 1\\0 & 0 & 1 & 0\\0 & 1 & 0 & 0\\1 & 1 & 1 & 1\end{bmatrix}\\[2mm]
&= \begin{bmatrix}\hat{\bf x}_1^{\text{\tiny(n)}} & \alpha_n\hat{\bf x}_2^{\text{\tiny(n)}} & \beta_n\hat{\bf x}_3^{\text{\tiny(n)}} & \gamma_n\hat{\bf x}_4^{\text{\tiny(n)}}\\[1mm]1 & \alpha_n & \beta_n & \gamma_n\end{bmatrix}\\[2mm]
&\sim \begin{bmatrix}\hat{\bf x}_1^{\text{\tiny(n)}} & \hat{\bf x}_2^{\text{\tiny(n)}} & \hat{\bf x}_3^{\text{\tiny(n)}} & \hat{\bf x}_4^{\text{\tiny(n)}}\\[1mm]1 & 1 & 1 & 1\end{bmatrix}.
    \end{split}
\end{equation*}
For the case $M=5$ with the estimated correspondences $\{\hat{\bf x}_1^{\text{\tiny(n)}}\,,\,\hat{\bf x}_2^{\text{\tiny(n)}}\,,\,\hat{\bf x}_3^{\text{\tiny(n)}}\,,\,\hat{\bf x}_4^{\text{\tiny(n)}}\,,\,\hat{\bf x}_5^{\text{\tiny(n)}}\}^N_{n=1}$, we need to choose $\beta_n$ and $\gamma_n$ such that
\begin{equation}\label{eq:beta_gamma_from_alpha}
    \beta_n = \frac{1}{\hat{\xi}_{34,5}^{\text{\tiny(n)}}}\big(2\hat{\xi}_{14,5}^{\text{\tiny(n)}} - \hat{\xi}_{24,5}^{\text{\tiny(n)}}\alpha_n\big) \qquad \text{and} \qquad \gamma_n = \frac{1}{\hat{\xi}_{34,5}^{\text{\tiny(n)}}}\big(\hat{\xi}_{23,5}^{\text{\tiny(n)}}\alpha_n - 2\hat{\xi}_{13,5}^{\text{\tiny(n)}}\big)
\end{equation}
where $\hat{\xi}_{m_1m_2,m_3}$ is given by~\eqref{eq:Xi} when ${\bf x}_m^{\text{\tiny(n)}}$ is replaced by its estimation $\hat{\bf x}_m^{\text{\tiny(n)}}$. From that we confirm
\begin{equation*}
\begin{split}
    &(\gamma_n-1)(\hat{\bf x}_4^{\text{\tiny(n)}} - \hat{\bf x}_5^{\text{\tiny(n)}}) + (\beta_n-1)(\hat{\bf x}_3^{\text{\tiny(n)}} - \hat{\bf x}_5^{\text{\tiny(n)}}) + (\alpha_n-1)(\hat{\bf x}_2^{\text{\tiny(n)}} - \hat{\bf x}_5^{\text{\tiny(n)}})\\
    &\quad\qquad + (\hat{\bf x}_4^{\text{\tiny(n)}} - \hat{\bf x}_5^{\text{\tiny(n)}}) + (\hat{\bf x}_3^{\text{\tiny(n)}} - \hat{\bf x}_5^{\text{\tiny(n)}}) + (\hat{\bf x}_2^{\text{\tiny(n)}} - \hat{\bf x}_5^{\text{\tiny(n)}}) - 2(\hat{\bf x}_1^{\text{\tiny(n)}} - \hat{\bf x}_5^{\text{\tiny(n)}}) ~=~ 0\,,
\end{split}
\end{equation*}
as the explanation in~\eqref{eq:Camera_from_Point_2} and~\eqref{eq:Camera_from_Point_3}. Then
\begin{equation*}
\begin{split}
    {\bf P}^{\star}_n\begin{bmatrix}{\bf X}^{\star}_5\\[1mm]1\end{bmatrix} &= \begin{bmatrix}\hat{\bf x}_4^{\text{\tiny(n)}} & \hat{\bf x}_3^{\text{\tiny(n)}} & \hat{\bf x}_2^{\text{\tiny(n)}} & \hat{\bf x}_1^{\text{\tiny(n)}}\\[1mm]1 & 1 & 1 & 1\end{bmatrix}\begin{bmatrix}
\gamma_n & 0 & 0 & 0\\
0 & \beta_n & 0 & 0\\
0 & 0 & \alpha_n & 0\\
-1 & -1 & -1 & 1
\end{bmatrix}\begin{bmatrix}1\\1\\1\\1\end{bmatrix}\\[2mm]
&= \begin{bmatrix}\gamma_n\hat{\bf x}_4^{\text{\tiny(n)}} + \beta_n\hat{\bf x}_3^{\text{\tiny(n)}} + \alpha_n\hat{\bf x}_2^{\text{\tiny(n)}} - 2\hat{\bf x}_1^{\text{\tiny(n)}}\\[1mm]\alpha_n + \beta_n + \gamma_n - 2\end{bmatrix} ~\sim \begin{bmatrix}\hat{\bf x}_5^{\text{\tiny(n)}}\\[1mm] 1\end{bmatrix}
\end{split}
\end{equation*}
and the objective value of $\{{\bf X}^{\star}_1,\ldots,{\bf X}^{\star}_5,{\bf P}^{\star}_n\}$ is zero for arbitrary values of $\alpha_n,~n = 1,2,\ldots,N$. \end{proof}

By exploiting the ambiguities of WPfC, Proposition~\ref{Prop:Global_Sol} proposes some simple and potential solutions for WPfC with at most five world points. Interestingly, all these solutions are global optimal. In optimization theory, finding global optimal solutions is the most important problem we hope to solve. Unfortunately, finding global solutions for the non-convex and non-concave objective functions like~\eqref{eq:Opt_WPfC} is  difficult~\citep{Horst1996}. Our strategy to solve the optimization~\eqref{eq:Opt_WPfC} is summarized as follows. First, basing on the global optimal solutions in Proposition~\ref{Prop:Global_Sol}, we find closed-form solutions $\{{\bf X}_m^{\circ},{\bf P}_n^{\circ}\}$ for~\eqref{eq:Opt_WPfC} with expecting that these closed-form solutions are closed to the global optimal ones. Second, we propose an iterative algorithm to generate some other better solutions $\{{\bf X}^*_m,{\bf P}^*_n\}$ from its initialization $\{{\bf X}_m^{\circ},{\bf P}_n^{\circ}\}$, the closed-form solution found by the first step.

\subsubsection{Closed-form solutions}\label{SubSec:Closed-form}

This subsection starts with the closed-form solutions for WPfC when the number of world points are six ($M=6$). These closed-forms are global optimal solutions of another optimization that is from~\eqref{eq:Opt_WPfC} and adding five constraints related to the first five world points. This optimization is stated as follows.
\begin{equation}\label{eq:Opt_CF_6}
\begin{split}
    &\arg\hspace{-2mm}\min_{\hspace{-5mm}\big\{{\bf X}_m,{\bf P}_n\big\}}\sum^N_{n=1}\left\{\Bigg(\frac{{\bf p}_{1,n}\big[{\bf X}_6^T\,,\,1\big]^T}{{\bf p}_{3,n}\big[{\bf X}_6^T\,,\,1\big]^T} - \hat{u}_{6n}\Bigg)^2 + \Bigg(\frac{{\bf p}_{2,n}\big[{\bf X}_6^T\,,\,1\big]^T}{{\bf p}_{3,n}\big[{\bf X}_6^T\,,\,1\big]^T} - \hat{v}_{6n}\Bigg)^2\right\}\\[2mm]
    &\text{subject to }\quad {\bf P}_n\begin{bmatrix}{\bf X}_m\\1\end{bmatrix} \sim \begin{bmatrix}\hat{\bf x}_m^{\text{\tiny(n)}}\\1\end{bmatrix}\quad \text{for all}\quad m = 1,2,3,4,5,\quad n = 1,2,\ldots, N.
\end{split}
\end{equation}
Note that under the constraints ${\bf P}_n\begin{bmatrix}{\bf X}_m\\1\end{bmatrix} \sim \begin{bmatrix}\hat{\bf x}_m^{\text{\tiny(n)}}\\1\end{bmatrix}$ for all $1 \leq m  \leq 5$,
\begin{equation*}
\sum^5_{m=1}\sum^N_{n=1}\left\{\Bigg(\frac{{\bf p}_{1,n}\big[{\bf X}_m^T\,,\,1\big]^T}{{\bf p}_{3,n}\big[{\bf X}_m^T\,,\,1\big]^T} - \hat{u}_{mn}\Bigg)^2 + \Bigg(\frac{{\bf p}_{2,n}\big[{\bf X}_m^T\,,\,1\big]^T}{{\bf p}_{3,n}\big[{\bf X}_m^T\,,\,1\big]^T} - \hat{v}_{mn}\Bigg)^2\right\} = 0\, .
\end{equation*}
Thus the objective value of~\eqref{eq:Opt_CF_6} is the exact objective value of~\eqref{eq:Opt_WPfC}. The proof of Proposition~\ref{Prop:Global_Sol} says that the group $\{{\bf X}^{\star}_1,{\bf X}^{\star}_2,{\bf X}^{\star}_3,{\bf X}^{\star}_4,{\bf X}^{\star}_5,{\bf X}^{\star}_6,{\bf P}^{\star}_n\}$ given by
\begin{subequations}\label{eq:CF_6}
\begin{align}
&\begin{bmatrix}{\bf X}^{\star}_1 & {\bf X}^{\star}_2 & {\bf X}^{\star}_3 & {\bf X}^{\star}_4 & {\bf X}^{\star}_5 & {\bf X}^{\star}_6\end{bmatrix} =
\begin{bmatrix}0 & 0 & 0 & 1 & 1 & f^{\star}_x/f^{\star}\\0 & 0 & 1 & 0 & 1 & f^{\star}_y/f^{\star}\\0 & 1 & 0 & 0 & 1 & f^{\star}_z/f^{\star}\end{bmatrix}\label{eq:CF_WP_6}\\[2mm]
&{\bf P}^{\star}_n = \begin{bmatrix}\hat{\bf x}_4^{\text{\tiny(n)}} & \hat{\bf x}_3^{\text{\tiny(n)}} & \hat{\bf x}_2^{\text{\tiny(n)}} & \hat{\bf x}_1^{\text{\tiny(n)}}\\[1mm]1 & 1 & 1 & 1\end{bmatrix}\begin{bmatrix}
\gamma^{\star}_n & 0 & 0 & 0\\
0 & \beta^{\star}_n & 0 & 0\\
0 & 0 & \alpha^{\star}_n & 0\\
-1 & -1 & -1 & 1
\end{bmatrix}\quad (n = 1,2,\ldots,N),\label{eq:CF_PM_6}\\[1mm]
&\text{with}\quad \beta^{\star}_n = \frac{1}{\hat{\xi}_{34,5}^{\text{\tiny(n)}}}\big(2\hat{\xi}_{14,5}^{\text{\tiny(n)}} - \hat{\xi}_{24,5}^{\text{\tiny(n)}}\alpha^{\star}_n\big),~~ \gamma^{\star}_n = \frac{1}{\hat{\xi}_{34,5}^{\text{\tiny(n)}}}\big(\hat{\xi}_{23,5}^{\text{\tiny(n)}}\alpha^{\star}_n - 2\hat{\xi}_{13,5}^{\text{\tiny(n)}}\big)\nonumber
\end{align}
\end{subequations}
will satisfy the constraints of~\eqref{eq:Opt_CF_6} for any values of $f^{\star}, f^{\star}_x, f^{\star}_y, f^{\star}_z,$ and $\{\alpha^{\star}_n\}^N_{n=1}$. Based on this fact, the following proposition gives global optimal solutions for~\eqref{eq:Opt_CF_6}.

\begin{prop}\label{Prop:CF_Sol_6}
As $f^{\star}, f^{\star}_x, f^{\star}_y, f^{\star}_z$ are given by~\eqref{eq:Lambda}-\eqref{eq:f_Eigen} in which ${\bf \Lambda}_6$ computing based on the estimated correspondences, and
\begin{equation}\label{eq:CF_alpha}
\alpha^{\star}_n ~=~ \frac{(a_{2n}^2+b_{2n}^2)c_{1n}^2 - (a_{1n}^2+b_{1n}^2)c_{2n}^2 \pm \sqrt{\Delta_n}}{2(a_{1n}^2+b_{1n}^2)c_{1n}c_{2n} - 2(a_{1n}a_{2n}+b_{1n}b_{2n})c_{1n}^2}
\end{equation}
where
\begin{equation}\label{eq:CF_alpha_1}
\begin{split}
    a_{1n} &= \hat{\xi}_{23,5}^{\text{\tiny(n)}}f^{\star}_x(\hat{u}_{4n}-\hat{u}_{6n}) - \hat{\xi}_{24,5}^{\text{\tiny(n)}}f^{\star}_y(\hat{u}_{3n}-\hat{u}_{6n}) + f^{\star}_z(\hat{u}_{2n}-\hat{u}_{6n})\\[1mm]
    a_{2n} &= -2\hat{\xi}_{13,5}^{\text{\tiny(n)}}f^{\star}_x(\hat{u}_{4n}-\hat{u}_{6n}) + 2\hat{\xi}_{14,5}^{\text{\tiny(n)}}f^{\star}_y(\hat{u}_{3n}-\hat{u}_{6n}) + (f^{\star} - f^{\star}_x - f^{\star}_y - f^{\star}_z)(\hat{u}_{1n} - \hat{u}_{6n})\\[1mm]
    b_{1n} &= \hat{\xi}_{23,5}^{\text{\tiny(n)}}f^{\star}_x(\hat{v}_{4n}-\hat{v}_{6n}) - \hat{\xi}_{24,5}^{\text{\tiny(n)}}f^{\star}_y(\hat{v}_{3n}-\hat{v}_{6n}) + f^{\star}_z(\hat{v}_{2n}-\hat{v}_{6n})\\[1mm]
    b_{2n} &= -2\hat{\xi}_{13,5}^{\text{\tiny(n)}}f^{\star}_x(\hat{v}_{4n}-\hat{v}_{6n}) + 2\hat{\xi}_{14,5}^{\text{\tiny(n)}}f^{\star}_y(\hat{v}_{3n}-\hat{v}_{6n}) + (f^{\star} - f^{\star}_x - f^{\star}_y - f^{\star}_z)(\hat{v}_{1n} - \hat{v}_{6n})\\[1mm]
    c_{1n} &= \hat{\xi}_{23,5}^{\text{\tiny(n)}}f^{\star}_x - \hat{\xi}_{24,5}^{\text{\tiny(n)}}f^{\star}_y + f^{\star}_z\\[1mm]
    c_{2n} &= -2\hat{\xi}_{13,5}^{\text{\tiny(n)}}f^{\star}_x + 2\hat{\xi}_{14,5}^{\text{\tiny(n)}}f^{\star}_y + (f^{\star} - f^{\star}_x - f^{\star}_y - f^{\star}_z)\\[1mm]
    \Delta_n &= \big[(a_{1n}^2+b_{1n}^2)c_{2n}^2 - (a_{2n}^2+b_{2n}^2)c_{1n}^2\big]^2\\
    &\quad- 4\big[(a_{1n}^2+b_{1n}^2)c_{1n}c_{2n} - (a_{1n}a_{2n}+b_{1n}b_{2n})c_{1n}^2\big]\big[(a_{1n}a_{2n}+b_{1n}b_{2n})c_{2n}^2 - (a_{2n}^2+b_{2n}^2)c_{1n}c_{2n}\big]\,,
\end{split}
\end{equation}
the group $\{{\bf X}^{\star}_1, {\bf X}^{\star}_2, {\bf X}^{\star}_3, {\bf X}^{\star}_4, {\bf X}^{\star}_5, {\bf X}^{\star}_6, {\bf P}^{\star}_n\}$ will be a global optimal solution for the optimization~\eqref{eq:Opt_CF_6}.
\end{prop}

\begin{proof}
From the setting of ${\bf X}^{\star}_1,{\bf X}^{\star}_2,{\bf X}^{\star}_3,{\bf X}^{\star}_4,{\bf X}^{\star}_5,$ and $\{{\bf P}^{\star}_n\}^N_{n=1}$,~\eqref{eq:Opt_CF_6} becomes
\begin{equation}\label{eq:Opt_CF_6_1}
\begin{split}
&\arg\hspace{-5mm}\min_{\hspace{-5mm}\big\{f, f_x, f_y, f_z, \alpha_n\big\}\,}\sum^N_{n=1}\Bigg\{\frac{\big[\gamma_nf_x(\hat{u}_{4n}\hspace{-0.25mm}-\hspace{-0.25mm}\hat{u}_{6n}) \hspace{-0.25mm}+\hspace{-0.25mm} \beta_nf_y(\hat{u}_{3n}\hspace{-0.25mm}-\hspace{-0.25mm}\hat{u}_{6n}) \hspace{-0.25mm}+\hspace{-0.25mm} \alpha_nf_z(\hat{u}_{2n}\hspace{-0.25mm}-\hspace{-0.25mm}\hat{u}_{6n}) \hspace{-0.25mm}+\hspace{-0.25mm} (f\hspace{-0.25mm}-\hspace{-0.25mm}f_x\hspace{-0.25mm}-\hspace{-0.25mm}f_y\hspace{-0.25mm}-\hspace{-0.25mm}f_z)(\hat{u}_{1n}\hspace{-0.25mm}-\hspace{-0.25mm}\hat{u}_{6n})\big]^2}{\big[\gamma_nf_x \hspace{-0.25mm}+\hspace{-0.25mm} \beta_nf_y \hspace{-0.25mm}+\hspace{-0.25mm} \alpha_nf_z \hspace{-0.25mm}+\hspace{-0.25mm} (f \hspace{-0.25mm}-\hspace{-0.25mm} f_x \hspace{-0.25mm}-\hspace{-0.25mm} f_y \hspace{-0.25mm}-\hspace{-0.25mm} f_z)\big]^2}\\[2mm]
&\hspace{23mm} + \frac{\big[\gamma_nf_x(\hat{v}_{4n}\hspace{-0.25mm}-\hspace{-0.25mm}\hat{v}_{6n}) \hspace{-0.25mm}+\hspace{-0.25mm} \beta_nf_y(\hat{v}_{3n}\hspace{-0.25mm}-\hspace{-0.25mm}\hat{v}_{6n}) \hspace{-0.25mm}+\hspace{-0.25mm} \alpha_nf_z(\hat{v}_{2n}\hspace{-0.25mm}-\hspace{-0.25mm}\hat{v}_{6n}) \hspace{-0.25mm}+\hspace{-0.25mm} (f\hspace{-0.25mm}-\hspace{-0.25mm}f_x\hspace{-0.25mm}-\hspace{-0.25mm}f_y\hspace{-0.25mm}-\hspace{-0.25mm}f_z)(\hat{v}_{1n}\hspace{-0.25mm}-\hspace{-0.25mm}\hat{v}_{6n})\big]^2}{\big[\gamma_nf_x \hspace{-0.25mm}+\hspace{-0.25mm} \beta_nf_y \hspace{-0.25mm}+\hspace{-0.25mm} \alpha_nf_z \hspace{-0.25mm}+\hspace{-0.25mm} (f \hspace{-0.25mm}-\hspace{-0.25mm} f_x \hspace{-0.25mm}-\hspace{-0.25mm} f_y \hspace{-0.25mm}-\hspace{-0.25mm} f_z)\big]^2}\Bigg\}
\end{split}
\end{equation}
where $a_{1n}, b_{1n}, c_{1n}$ and $a_{2n}, b_{2n}, c_{2n}$ are given by~\eqref{eq:CF_alpha_1}.  When $f^{\star}, f^{\star}_x, f^{\star}_y, f^{\star}_z$ are fixed, the optimization~\eqref{eq:Opt_CF_6_1} is globally optimized by the values $\{\alpha^{\star}_n\}^N_{n=1}$ given by~\eqref{eq:CF_alpha}. Hence, to finish the proof of Proposition~\ref{Prop:CF_Sol_6} we need to prove that~\eqref{eq:Opt_CF_6_1} is globally optimized at $f^{\star}, f^{\star}_x, f^{\star}_y, f^{\star}_z$ given by~\eqref{eq:Lambda}-\eqref{eq:f_Eigen} while ${\bf P}^{\star}_n, \alpha^{\star}_n, \beta^{\star}_n, \gamma^{\star}_n$ are functions of $f^{\star}, f^{\star}_x, f^{\star}_y$ and $f^{\star}_z$. Thanks to Theorem II in~\citet{le2022multi}, we know that the sixth world point ${\bf X}^{\star}_6 = \frac{1}{f^{\star}}[f^{\star}_x\,,\,f^{\star}_y\,,\,f^{\star}_z]^T$ should be found from the following optimization
\begin{equation}\label{eq:Opt_CF_6_2}
   \{f^{\star}\,,\, f^{\star}_x\,,\, f^{\star}_y\,,\, f^{\star}_z\} =\quad \arg\hspace{-3mm}\min_{\hspace{-5mm}\big\{f,f_x,f_y,f_z\big\}}\left\|~{\bf \Lambda}_6\begin{bmatrix}
    f_z(f - f_x - f_y - f_z) + 2f_y f_z\\[1mm]
    f_y(f - f_x - f_y - f_z) + 2f_y f_z\\[1mm]
    f_x(f - f_x - f_y - f_z) + 2f_y f_z\\[1mm]
    -2f_x f_y + 2f_y f_z\\[1mm]
    -2f_x f_z + 2f_y f_z\end{bmatrix}~\right\|^2_2
\end{equation}
where ${\bf \Lambda}_6$ given by~\eqref{eq:Lambda} is the $N\times 5$ matrix computing from the estimated correspondences. Let $[e_1,e_2,e_3,e_4,e_5]^T$ be the eigenvector corresponding to the smallest eigenvalue of ${\bf \Lambda}^T_6{\bf \Lambda}_6$. The optimization~\eqref{eq:Opt_CF_6_2} will be globally optimized if
\begin{equation}\label{eq:f_Eigen_1}
    \begin{bmatrix}
    f_z(f - f_x - f_y - f_z) + 2f_y f_z\\[1mm]
    f_y(f - f_x - f_y - f_z) + 2f_y f_z\\[1mm]
    f_x(f - f_x - f_y - f_z) + 2f_y f_z\\[1mm]
    -2f_x f_y + 2f_y f_z\\[1mm]
    -2f_x f_z + 2f_y f_z\end{bmatrix} \sim \begin{bmatrix}e_1\\[1mm]e_2\\[1mm]e_3\\[1mm]e_4\\[1mm]e_5\end{bmatrix}\,.
\end{equation}
The parameters $\{f^{\star}, f^{\star}_x, f^{\star}_y, f^{\star}_z\}$ given by~\eqref{eq:f_Eigen} satisfy~\eqref{eq:f_Eigen_1}. Thus, they are the global optimal solutions for~\eqref{eq:Opt_CF_6_2} and also for~\eqref{eq:Opt_CF_6_1}. Proposition~\ref{Prop:CF_Sol_6} is proven.
\end{proof}

Note that Proposition~\ref{Prop:CF_Sol_6} derives the global optimal solutions for~\eqref{eq:Opt_CF_6} but these solutions are not global optimal solutions for~\eqref{eq:Opt_WPfC}. In addition, because of the strict constraints in~\eqref{eq:Opt_CF_6}, its global optimal solutions are far to ones of~\eqref{eq:Opt_WPfC}. In the case $M > 6$, the optimization~\eqref{eq:Opt_CF_6} naturally becomes
\begin{equation}\label{eq:Opt_CF_more6}
\begin{split}
    &\arg\hspace{-2mm}\min_{\hspace{-5mm}\big\{{\bf X}_m,{\bf P}_n\big\}}\sum^M_{m=6}\,\sum^N_{n=1}\left\{\hspace{-1mm}\Bigg(\frac{{\bf p}_{1,n}\big[{\bf X}_m^T\,,\,1\big]^T}{{\bf p}_{3,n}\big[{\bf X}_m^T\,,\,1\big]^T} - \hat{u}_{mn}\hspace{-0.5mm}\Bigg)^2 \hspace{-1.5mm}+\hspace{-0.5mm} \Bigg(\frac{{\bf p}_{2,n}\big[{\bf X}_m^T\,,\,1\big]^T}{{\bf p}_{3,n}\big[{\bf X}_m^T\,,\,1\big]^T} - \hat{v}_{mn}\hspace{-0.5mm}\Bigg)^2\right\}\\[2mm]
    &\text{subject to }\quad {\bf P}_n\begin{bmatrix}{\bf X}_m\\1\end{bmatrix} \sim \begin{bmatrix}\hat{\bf x}_m^{\text{\tiny(n)}}\\1\end{bmatrix}\quad \text{for all}\quad 1 \leq m  \leq 5,\quad 1 \leq n \leq N.
\end{split}
\end{equation}
The optimization~\eqref{eq:Opt_CF_more6} is more complex than the optimization~\eqref{eq:Opt_CF_6}. Finding global optimal solutions for~\eqref{eq:Opt_CF_more6} is too difficult. Moreover, similar as the case of $M=6$, we believe that the global optimal solutions for~\eqref{eq:Opt_CF_more6} are far to the global optimal solutions for~\eqref{eq:Opt_WPfC}. Because of this reason, this study does not focus on obtaining the global optimal solutions for~\eqref{eq:Opt_CF_more6}. We shall use the global optimal solutions in the case $M=6$ as the candidates for the optimal solutions in the case $M > 6$. More precisely, to generate optimal solutions for~\eqref{eq:Opt_CF_more6}, we first propose the optimal world points ${\bf X}^{\circ}_m = \frac{1}{f^{\star}_m}[f^{\star}_{xm}\,,\,f^{\star}_{ym}\,,\,f^{\star}_{zm}]^T$, ($6 \leq m \leq M$) as the global optimal solutions for~\eqref{eq:Opt_CF_6_2} when ${\bf \Lambda}_6$ is replaced by ${\bf \Lambda}_m$. Then for each $m\in\{6,7,\ldots,M\}$ that brings a group of $\{f^{\star}, f^{\star}_x, f^{\star}_y, f^{\star}_z\}$ and $\{a_{1n},b_{1n},c_{1n},a_{2n},b_{2n},c_{2n},\Delta_n\}^N_{n=1}$, one candidate for $\{\alpha^{\star}_n\}^N_{n=1}$ is generated by~\eqref{eq:CF_alpha}, and a corresponded candidate for $\{{\bf P}^{\circ}_n\}^N_{n=1}$ is obtained by~\eqref{eq:CF_PM_6}. From a lot of these candidates, the best one is chosen based on the objective value in~\eqref{eq:Opt_CF_more6}. Formally, our process on estimating the closed-form solutions for WPfC is presented by Algorithm~\ref{Alg:Closed_form}. In this paper, we denote
\begin{equation*}
    \{{\bf X}^{\circ}_m,{\bf P}^{\circ}_n\} = \texttt{CF-WPfC}\big(\{\hat{\bf x}^{\text{\tiny(1)}}_m \,\leftrightarrow\,\hat{\bf x}^{\text{\tiny(2)}}_m \,\leftrightarrow\,\cdots \,\leftrightarrow\,\hat{\bf x}^{\text{\tiny(N)}}_m\}^M_{m=1}\big)
\end{equation*}
to say that $\{{\bf X}^{\circ}_m,{\bf P}^{\circ}_n\}$ is the output of the CF-WPfC algorithm with the input $\{\hat{\bf x}^{\text{\tiny(1)}}_m \,\leftrightarrow\,\hat{\bf x}^{\text{\tiny(2)}}_m \,\leftrightarrow\,\cdots \,\leftrightarrow\,\hat{\bf x}^{\text{\tiny(N)}}_m\}^M_{m=1}$.

\begin{algorithm}[t!]
\caption{Closed-form solutions for WPfC (\textbf{CF-WPfC})}
\label{Alg:Closed_form}
{
{\bf Input:} $\big\{\hat{\bf x}^{\text{\tiny(1)}}_m \,\leftrightarrow\,\hat{\bf x}^{\text{\tiny(2)}}_m \,\leftrightarrow\,\cdots \,\leftrightarrow\,\hat{\bf x}^{\text{\tiny(N)}}_m\big\}^M_{m=1}$~:~ Correspondences.\\[-2mm]

{\bf Conditions:} $M\geq 6$ and $N\geq 5$.\\[-2mm]

{\bf Implementation:}\\[-2mm]

$\text{Eval} \longleftarrow 10^8$: a huge number\\[-2mm]

For all $1 \leq m_1 < m_2 < m_3 < m_4 < m_5 \leq M$:\\[-2mm]

{\bf 1.} $[{\bf X}_{m_1}\,,\,{\bf X}_{m_2}\,,\,{\bf X}_{m_3}\,,\,{\bf X}_{m_4}\,,\,{\bf X}_{m_5}] \quad \longleftarrow \quad \begin{bmatrix}0 & 0 & 0 & 1 & 1\\0 & 0 & 1 & 0 & 1\\0 & 1 & 0 & 0 & 1\end{bmatrix}$ \quad as the setting~\eqref{eq:Ambi_total}.\\

{\bf 2.} For all $1\leq m \leq M$ such that $m\notin\{m_1,m_2,m_3,m_4,m_5\}$:\\[-2mm]

\hspace{6mm}{\bf a.} $\hat{\xi}_{i_1i_2,j}^{\text{\tiny(n)}} ~\overset{\eqref{eq:Xi}}{\longleftarrow}~\big\{\hat{\bf x}_{i_1}^{\text{\tiny(n)}}\,,\,\hat{\bf x}_{i_2}^{\text{\tiny(n)}}\,,\,\hat{\bf x}_j^{\text{\tiny(n)}}\big\}$ \quad for all $1\leq n \leq N$,\\[-2mm]

\hspace{6mm}{\bf b.} ${\bf \Lambda}_m \hspace{3mm}~\overset{\eqref{eq:Lambda}}{\longleftarrow}~\big\{\hat{\bf x}_i^{\text{\tiny(1)}}\leftrightarrow\hat{\bf x}_i^{\text{\tiny(2)}}\leftrightarrow\cdots\leftrightarrow\hat{\bf x}_i^{\text{\tiny(N)}}\big\}_{i\in\{m_1,m_2,m_3,m_4,m_5,m\}}$\\[-2mm]

\hspace{6mm}{\bf c.} $[e_{1m},e_{2m},e_{3m},e_{4m},e_{5m}]^T$: the eigenvector corresponding to the smallest eigenvalue of ${\bf \Lambda}^T_m{\bf \Lambda}_m$,\\[-2mm]

\hspace{6mm}{\bf d.} $\{f_m, f_{xm}, f_{ym}, f_{zm}\} ~\overset{\eqref{eq:f_Eigen}}{\longleftarrow}~\{e_{1m},e_{2m},e_{3m},e_{4m},e_{5m}\}$\\[-2mm]

\hspace{6mm}{\bf e.} ${\bf X}_m ~\longleftarrow~\frac{1}{f_m}[f_{xm}\,,\,f_{ym}\,,\,f_{zm}]^T$.\\[-2mm]

{\bf 3.} For all $1\leq n \leq N$ ~and~ $1\leq m \leq M$ such that $m\notin\{m_1,m_2,m_3,m_4,m_5\}$:\\[-2mm]

\hspace{6mm}{\bf a.} $\{a_{1n},a_{2n},b_{1n},b_{2n},c_{1n},c_{2n},\Delta_n\} ~\overset{\eqref{eq:CF_alpha_1}}{\longleftarrow}~\{f_m, f_{xm}, f_{ym}, f_{zm}\}$, $\{\hat{\xi}_{i_1i_2,j}^{\text{\tiny(n)}}\}$, $\big\{\hat{\bf x}_{i}^{\text{\tiny(n)}}\big\}_{i\in\{m_1,m_2,m_3,m_4,m_5,m\}}$\\[-2mm]

\hspace{6mm}{\bf b.} if $\Delta_n > 0$:\\[-2mm]
\[\alpha_n ~\overset{\eqref{eq:CF_alpha}}{\longleftarrow}~ \frac{(a_{2n}^2+b_{2n}^2)c_{1n}^2 - (a_{1n}^2+b_{1n}^2)c_{2n}^2 \pm \sqrt{\Delta_n}}{2(a_{1n}^2+b_{1n}^2)c_{1n}c_{2n} - 2(a_{1n}a_{2n}+b_{1n}b_{2n})c_{1n}^2}\]

\hspace{11mm}else: ~ $\alpha_n ~\longleftarrow~ \big\{-a_{2n}/a_{1n}~,~ -b_{2n}/b_{1n} ~,~ -(a_{1n}a_{2n}+b_{1n}b_{2n})/(a_{1n}^2+b_{1n}^2)\big\}$\\[-2mm]

\hspace{6mm}{\bf c.} ${\bf P}_n ~\overset{\eqref{eq:CF_PM_6}}{\longleftarrow}~\big\{\hat{\bf x}_{m_1}^{\text{\tiny(n)}}\,,\,\hat{\bf x}_{m_2}^{\text{\tiny(n)}}\,,\,\hat{\bf x}_{m_3}^{\text{\tiny(n)}}\,,\,\hat{\bf x}_{m_4}^{\text{\tiny(n)}}\,,\,\alpha_n\,,\,\beta_n\,,\,\gamma_n\big\}$\\[-2mm]

{\bf 4.}\\[-5mm]
\[\text{Obj} \longleftarrow \sum_{m\notin\{m_1,m_2,m_3,m_4,m_5\}}\,\sum^N_{n=1}\left\{\Bigg(\frac{{\bf p}_{1,n}\big[{\bf X}_m^T\,,\,1\big]^T}{{\bf p}_{3,n}\big[{\bf X}_m^T\,,\,1\big]^T} - \hat{u}_{mn}\Bigg)^2 + \Bigg(\frac{{\bf p}_{2,n}\big[{\bf X}_m^T\,,\,1\big]^T}{{\bf p}_{3,n}\big[{\bf X}_m^T\,,\,1\big]^T} - \hat{v}_{mn}\Bigg)^2\right\}\]

{\bf 5.} If~ $\text{Obj} < \text{Eval}$:
\[\text{Eval} \longleftarrow \text{Obj}\qquad\text{and}\qquad \big\{{\bf X}^{\circ}_m, {\bf P}^{\circ}_n\big\}~\longleftarrow~\big\{{\bf X}_m,{\bf P}_n\big\}\,.\]

{\bf Output:} $\big\{{\bf X}^{\circ}_m, {\bf P}^{\circ}_n\big\}$\,.\\[-1mm]
}
\end{algorithm}

\subsubsection{Iterative solutions}

The output $\{{\bf X}^{\circ}_m,{\bf P}^{\circ}_n\}$ of the CF-WPfC algorithm is an optimal solution for~\eqref{eq:Opt_CF_more6}, from that there always exist five world points ${\bf X}^{\circ}_{m_1}, {\bf X}^{\circ}_{m_2}, {\bf X}^{\circ}_{m_3}, {\bf X}^{\circ}_{m_4}, {\bf X}^{\circ}_{m_5}$ such that
\begin{equation*}
    \qquad\qquad{\bf P}^{\circ}_n\begin{bmatrix}{\bf X}^{\circ}_{m_i}\\[1mm]1\end{bmatrix} \sim\begin{bmatrix}\hat{\bf x}_{m_i}^{\text{\tiny(n)}}\\[1mm]1\end{bmatrix}\qquad \text{for all}\quad n \quad \text{and} \quad i = 1,2,3,4,5\,.
\end{equation*}
Under these constraints, the solution $\{{\bf X}^{\circ}_m,{\bf P}^{\circ}_n\}$ can be  far to the global optimal solution for~\eqref{eq:Opt_WPfC}. In this subsection, we propose a method to release these constraints and then create better solutions for~\eqref{eq:Opt_WPfC}. The proposed method starts with the following result. Given the estimated correspondences $\big\{\hat{\bf x}_m^{\text{\tiny(1)}}\leftrightarrow\hat{\bf x}_m^{\text{\tiny(2)}}\leftrightarrow\cdots\leftrightarrow\hat{\bf x}_m^{\text{\tiny(N)}}\big\}^M_{m=1}$ with $M\geq 6$ and $N\geq 5$, let
\begin{equation*}
    \{{\bf X}^{\circ}_m,{\bf P}^{\circ}_n\} = \texttt{CF-WPfC}\big(\{\hat{\bf x}_m^{\text{\tiny(1)}}\leftrightarrow\hat{\bf x}_m^{\text{\tiny(2)}}\leftrightarrow\cdots\leftrightarrow\hat{\bf x}_m^{\text{\tiny(N)}}\}^M_{m=1}\big)
\end{equation*}
and $\big\{\breve{\bf x}_m^{\text{\tiny(1)}}\leftrightarrow\breve{\bf x}_m^{\text{\tiny(2)}}\leftrightarrow\cdots\leftrightarrow\breve{\bf x}_m^{\text{\tiny(N)}}\big\}^M_{m=1}$ be its correspondences, i.e.,
\begin{equation}\label{eq:Corresponences_Est}
\qquad\qquad\breve{\bf x}_m^{\text{\tiny(n)}} = \left[~\frac{{\bf p}^{\circ}_{1,n}\big[{\bf X}^{\circ T}_m\,,\,1\big]^T}{{\bf p}^{\circ}_{3,n}\big[{\bf X}^{\circ T}_m\,,\,1\big]^T}~,~\frac{{\bf p}^{\circ}_{2,n}\big[{\bf X}^{\circ T}_m\,,\,1\big]^T}{{\bf p}^{\circ}_{3,n}\big[{\bf X}^{\circ T}_m\,,\,1\big]^T}~\right]^T\qquad \text{for all}\quad m,n\, .
\end{equation}
On the $n$th image, there are two $M$-groups of points $\{\hat{\bf x}_m^{\text{\tiny(n)}}\}^M_{m=1}$ and $\{\breve{\bf x}_m^{\text{\tiny(n)}}\}^M_{m=1}$. A difference between these two groups is a part corresponding to the index $n$ of the object value of~\eqref{eq:Opt_WPfC} for the solution $\{{\bf X}^{\circ}_m,{\bf P}^{\circ}_n\}$. Thus, we expect that these $N$ differences are small to ensure that $\{{\bf X}^{\circ}_m,{\bf P}^{\circ}_n\}$ is good. The theory of \emph{absolute orientation}~\citep{Horn1987} allows us to have an optimal \emph{reflection} or \emph{rotation} ${\bf R}^*_n\in\mathbb{R}^{2\times 2}$ and an optimal \emph{translation} ${\bf t}^*_n\in\mathbb{R}^2$ to transfer $\{\breve{\bf x}_m^{\text{\tiny(n)}}\}^M_{m=1}$ closest to $\{\hat{\bf x}_m^{\text{\tiny(n)}}\}^M_{m=1}$. Exactly, the result in~\citet{Horn1987} solves the following optimization
\begin{equation}\label{eq:Isometry}
\begin{split}
\{{\bf R}^*_n\,,\,{\bf t}^*_n\} = &\arg\min\sum^M_{m=1}\Big\|\big({\bf R}_n\breve{\bf x}_m^{\text{\tiny(n)}} + {\bf t}_n\big) - \hat{\bf x}_m^{\text{\tiny(n)}}\Big\|_2\\[1mm]
&\text{subject to} \quad {\bf R}_n ~\text{: a rotation or a reflection},~~ {\bf t}_n ~\text{: a translation.}
\end{split}
\end{equation}
With the optimal reflection or rotation ${\bf R}^*_n$ and the optimal translation ${\bf t}^*_n$ for the $n$th image, we consider new projection matrices that
\begin{equation}\label{eq:PM_Iso}
{\bf P}^*_n = \begin{bmatrix}{\bf R}^*_n & {\bf t}^*_n\\[1mm]{\bf 0}_{1\times 2} & 1\end{bmatrix}{\bf P}^{\circ}_n \qquad \text{for all} \quad n\,,
\end{equation}
and get a better solution for~\eqref{eq:Opt_WPfC} by the following proposition.

\newpage

\begin{prop}\label{Prop:Opt_by_Iso}
According to the optimization~\eqref{eq:Opt_WPfC}, $\{{\bf X}^{\circ}_m,{\bf P}^*_n\}$ is better than $\{{\bf X}^{\circ}_m,{\bf P}^{\circ}_n\}$, i.e.,
\begin{equation}\label{eq:Comparation1}
\sum^M_{m=1}\sum^N_{n=1}\big\|\bar{\bf x}_m^{\text{\tiny(n)}} - \hat{\bf x}_m^{\text{\tiny(n)}}\big\|^2_2~\leq~ \sum^M_{m=1}\sum^N_{n=1}\big\|\breve{\bf x}_m^{\text{\tiny(n)}} - \hat{\bf x}_m^{\text{\tiny(n)}}\big\|^2_2\,,
\end{equation}
where $\big\{\bar{\bf x}_m^{\text{\tiny(1)}}\leftrightarrow\bar{\bf x}_m^{\text{\tiny(2)}}\leftrightarrow\cdots\leftrightarrow\bar{\bf x}_m^{\text{\tiny(N)}}\big\}$ be the correspondences of $\{{\bf X}^{\circ}_m,{\bf P}^*_n\}$, i.e.,
\begin{equation*}
\qquad\qquad\bar{\bf x}_m^{\text{\tiny(n)}} = \left[~\frac{{\bf p}^*_{1,n}\big[{\bf X}^{\circ T}_m\,,\,1\big]^T}{{\bf p}^*_{3,n}\big[{\bf X}^{\circ T}_m\,,\,1\big]^T}~,~\frac{{\bf p}^*_{2,n}\big[{\bf X}^{\circ T}_m\,,\,1\big]^T}{{\bf p}^*_{3,n}\big[{\bf X}^{\circ T}_m\,,\,1\big]^T}~\right]^T\qquad \text{for all}\quad m,n\, .
\end{equation*}
\end{prop}
\begin{proof}
The formula of ${\bf P}^*_n$ given by~\eqref{eq:PM_Iso} implies
\begin{equation*}
{\bf P}^*_n\begin{bmatrix}{\bf X}^{\circ}_m\\[1mm] 1\end{bmatrix} = \begin{bmatrix}{\bf R}^*_n & {\bf t}^*_n\\[1mm]{\bf 0}_{1\times 2} & 1\end{bmatrix}{\bf P}^{\circ}_n\begin{bmatrix}{\bf X}^{\circ}_m\\[1mm]1\end{bmatrix} \sim \begin{bmatrix}{\bf R}^*_n & {\bf t}^*_n\\[1mm]{\bf 0}_{1\times 2} & 1\end{bmatrix}\begin{bmatrix}\breve{\bf x}_m^{\text{\tiny(n)}}\\[1mm] 1\end{bmatrix} = \begin{bmatrix}{\bf R}^*_n\breve{\bf x}_m^{\text{\tiny(n)}} + {\bf t}^*_n\\[1mm]1\end{bmatrix}\,.
\end{equation*}
Since $\{\bar{\bf x}_m^{\text{\tiny(1)}}\leftrightarrow\bar{\bf x}_m^{\text{\tiny(2)}}\leftrightarrow\cdots\leftrightarrow\bar{\bf x}_m^{\text{\tiny(N)}}\}$ be the correspondences of $\{{\bf X}^{\circ}_m,{\bf P}^*_n\}$, the above equations confirm $\bar{\bf x}_m^{\text{\tiny(n)}} = {\bf R}^*_n\breve{\bf x}_n^{\text{\tiny(n)}} + {\bf t}^*_n$ for all $m$ and $n$. Thus, for all $n$
\begin{equation}\label{eq:Final_Prop_Iso}
    \sum^M_{m=1}\big\|\bar{\bf x}_m^{\text{\tiny(n)}} - \hat{\bf x}_m^{\text{\tiny(n)}}\big\|^2_2 \leq \sum^M_{m=1}\Big\|\big({\bf I}_2\breve{\bf x}_m^{\text{\tiny(n)}} + {\bf 0}_{2\times 1}\big) - \hat{\bf x}_m^{\text{\tiny(n)}}\Big\|^2_2 = \sum^M_{m=1}\big\|\breve{\bf x}_m^{\text{\tiny(n)}} - \hat{\bf x}_m^{\text{\tiny(n)}}\big\|^2_2\, ,
\end{equation}
where ${\bf I}_2$ is a $2\times 2$ identity matrix.~\eqref{eq:Final_Prop_Iso} sufficiently finishes the proof of Proposition~\ref{Prop:Opt_by_Iso}.
\end{proof}

From the closed-form solution $\{{\bf X}^{\circ}_m,{\bf P}^{\circ}_n\}$ given by CF-WPfC algorithm, Proposition~\ref{Prop:Opt_by_Iso} proposes the better solution $\{{\bf X}^{\circ}_m,{\bf P}^*_n\}$ by using the absolute orientation theory to change the projection matrices. In the absolute orientation theory, we are interested in the following fact.
\begin{fact}\label{Fact:Orientation}
Given a group of $M$ points $\{{\bf x}_m\}^M_{m=1}$ in $\mathbb{R}^2$ and its estimation $\{{\bf x}'_m\}^M_{m=1}$. Assuming that $\{{\bf x}''_m\}^M_{m=1}$ is a linear transformation from $\{{\bf x}_m\}^M_{m=1}$ and $\{{\bf x}'_m\}^M_{m=1}$, and
\begin{equation*}
    \sum^M_{m=1}\big\|{\bf x}'_m - {\bf x}_m\big\|^2_2~\leq~\sum^M_{m=1}\big\|{\bf x}''_m - {\bf x}_m\big\|^2_2
\end{equation*}
then
\begin{equation*}
    \sum^M_{m=1}\big\|\big({\bf R}'{\bf x}'_m + {\bf t}'\big) - {\bf x}_m\big\|^2_2~\leq~\sum^M_{m=1}\big\|\big({\bf R}''{\bf x}''_m + {\bf t}''\big) - {\bf x}_m\big\|^2_2
\end{equation*}
where $\{{\bf R}',{\bf t}'\}$ and $\{{\bf R}'',{\bf t}''\}$ are the optimal solutions for~\eqref{eq:Isometry} corresponding to $\{{\bf x}'_m\}$ and $\{{\bf x}''_m\}$, respectively.
\end{fact}
Fact~\ref{Fact:Orientation} is not always correct, however the probability for the correctness of Fact~\ref{Fact:Orientation} is close to one. For example, Fact~\ref{Fact:Orientation} is almost correct with the correspondences and their estimations in our study. Hence, in our work, we use this fact to create better solutions. Formally, we find the better solution $\{{\bf X}^{**}_m,{\bf P}^{**}_n\}$ than $\{{\bf X}^{\circ}_m,{\bf P}^*_n\}$ by the following process. From the correspondences $\{\bar{\bf x}_m^{\text{\tiny(1)}}\leftrightarrow\bar{\bf x}_m^{\text{\tiny(2)}}\leftrightarrow\cdots\leftrightarrow\bar{\bf x}_m^{\text{\tiny(N)}}\}$ of $\{{\bf X}^{\circ}_m,{\bf P}^*_n\}$, we create a new estimated correspondences $\{\ddot{\bf x}_m^{\text{\tiny(1)}}\leftrightarrow\ddot{\bf x}_m^{\text{\tiny(2)}}\leftrightarrow\cdots\leftrightarrow\ddot{\bf x}_m^{\text{\tiny(N)}}\}$ from the old estimated correspondences $\{\hat{\bf x}_m^{\text{\tiny(1)}}\leftrightarrow\hat{\bf x}_m^{\text{\tiny(2)}}\leftrightarrow\cdots\leftrightarrow\hat{\bf x}_m^{\text{\tiny(N)}}\}$ by a simple linear transformation that
\begin{equation*}
\qquad\qquad\qquad\ddot{\bf x}_m^{\text{\tiny(n)}} = \frac{1}{2}\big(\bar{\bf x}_m^{\text{\tiny(n)}} + \hat{\bf x}_m^{\text{\tiny(n)}}\big) \qquad \text{for all}\quad m,n\, .
\end{equation*}
Let
\begin{equation}\label{eq:Sol_ddot}
\{\ddot{\bf X}_m,\ddot{\bf P}_n\} ~=~ \texttt{CF-WPfC}\big(\{\ddot{\bf x}_m^{\text{\tiny(1)}}\leftrightarrow\ddot{\bf x}_m^{\text{\tiny(2)}}\leftrightarrow\cdots\leftrightarrow\ddot{\bf x}_m^{\text{\tiny(N)}}\}\big)
\end{equation}
and $\{\dot{\bf x}_m^{\text{\tiny(1)}}\leftrightarrow\dot{\bf x}_m^{\text{\tiny(2)}}\leftrightarrow\cdots\leftrightarrow\dot{\bf x}_m^{\text{\tiny(N)}}\}$ be the correspondences of $\{\ddot{\bf X}_m,\ddot{\bf P}_n\}$. Finally, let $\{\tilde{\bf x}_m^{\text{\tiny(1)}}\leftrightarrow\tilde{\bf x}_m^{\text{\tiny(2)}}\leftrightarrow\cdots\leftrightarrow\tilde{\bf x}_m^{\text{\tiny(N)}}\}$ be the optimal correspondences from $\{\dot{\bf x}_m^{\text{\tiny(1)}}\leftrightarrow\dot{\bf x}_m^{\text{\tiny(2)}}\leftrightarrow\cdots\leftrightarrow\dot{\bf x}_m^{\text{\tiny(N)}}\}$ by using~\eqref{eq:Isometry}. We use the final correspondences $\{\tilde{\bf x}_m^{\text{\tiny(1)}}\leftrightarrow\tilde{\bf x}_m^{\text{\tiny(2)}}\leftrightarrow\cdots\leftrightarrow\tilde{\bf x}_m^{\text{\tiny(N)}}\}$ to obtain the better solution $\{{\bf X}^{**}_m,{\bf P}^{**}_n\}$ as we want by
\begin{equation}\label{eq:Sol_twostar}
\{{\bf X}^{**}_m,{\bf P}^{**}_n\}~=~ \texttt{CF-WPfC}\big(\{\tilde{\bf x}_m^{\text{\tiny(1)}}\leftrightarrow\tilde{\bf x}_m^{\text{\tiny(2)}}\leftrightarrow\cdots\leftrightarrow\tilde{\bf x}_m^{\text{\tiny(N)}}\}\big)\,.
\end{equation}
The better of $\{{\bf X}^{**}_m,{\bf P}^{**}_n\}$ than $\{{\bf X}^{\circ}_m,{\bf P}^{*}_n\}$ is proven by the following proposition.

\begin{prop}\label{Prop:Opt_by_Iter}
The solution $\{\dot{\bf X}_m,\dot{\bf P}_n\}$ is better than $\{{\bf X}^{\circ}_m,{\bf P}^{\circ}_n\}$, and then the solution $\{{\bf X}^{**}_m,{\bf P}^{**}_n\}$ is better than $\{{\bf X}^{\circ}_n,{\bf P}^*_n\}$ based on the correctness of Fact~\ref{Fact:Orientation}.
\end{prop}
\begin{proof}
A relationship between two correspondences $\{\dot{\bf x}_m^{\text{\tiny(1)}}\leftrightarrow\dot{\bf x}_m^{\text{\tiny(2)}}\leftrightarrow\cdots\leftrightarrow\dot{\bf x}_m^{\text{\tiny(N)}}\}$ and $\{\ddot{\bf x}_m^{\text{\tiny(1)}}\leftrightarrow\ddot{\bf x}_m^{\text{\tiny(2)}}\leftrightarrow\cdots\leftrightarrow\ddot{\bf x}_m^{\text{\tiny(N)}}\}$ via the solution $\{\ddot{\bf X}_m,\ddot{\bf P}_n\}$ confirms that $\dot{\bf x}_m^{\text{\tiny(n)}} = \ddot{\bf x}_m^{\text{\tiny(n)}}$ for all $n$ and $m = 1,2,\ldots,5$. Thus, the objective value of $\{\ddot{\bf X}_m,\ddot{\bf P}_n\}$ is computed based on its correspondences as follows.
\begin{equation}\label{eq:Ineq_1}
\begin{split}
    d\big(\{\ddot{\bf X}_m,\ddot{\bf P}_n\}\big) &= \sum^M_{m=1}\sum^N_{n=1}\big\|\dot{\bf x}_m^{\text{\tiny(n)}} - \hat{\bf x}_m^{\text{\tiny(n)}}\big\|^2_2\\
    &= \sum^5_{m=1}\sum^N_{n=1}\big\|\ddot{\bf x}_m^{\text{\tiny(n)}} - \hat{\bf x}_m^{\text{\tiny(n)}}\big\|^2_2 + \sum^M_{m=6}\sum^N_{n=1}\big\|\dot{\bf x}_m^{\text{\tiny(n)}} - \hat{\bf x}_m^{\text{\tiny(n)}}\big\|^2_2\\
    &\leq \sum^5_{m=1}\sum^N_{n=1}\Big\|\frac{1}{2}(\bar{\bf x}_m^{\text{\tiny(n)}} - \hat{\bf x}_m^{\text{\tiny(n)}})\Big\|^2_2 + \sum^M_{m=6}\sum^N_{n=1}2\Big(\big\|\dot{\bf x}_m^{\text{\tiny(n)}} - \ddot{\bf x}_m^{\text{\tiny(n)}}\big\|^2_2 + \big\|\ddot{\bf x}_m^{\text{\tiny(n)}} - \hat{\bf x}_m^{\text{\tiny(n)}}\big\|^2_2\Big)\\
    &= \frac{1}{4}\sum^5_{m=1}\sum^N_{n=1}\big\|\bar{\bf x}_m^{\text{\tiny(n)}} - \hat{\bf x}_m^{\text{\tiny(n)}}\big\|^2_2 + \frac{1}{2}\sum^M_{m=6}\sum^N_{n=1}\big\|\bar{\bf x}_m^{\text{\tiny(n)}} - \hat{\bf x}_m^{\text{\tiny(n)}}\big\|^2_2 + 2\sum^M_{m=6}\sum^N_{n=1}\big\|\dot{\bf x}_m^{\text{\tiny(n)}} - \ddot{\bf x}_m^{\text{\tiny(n)}}\big\|^2_2\\
    &= \frac{1}{2}d\big(\{{\bf X}^{\circ}_m,{\bf P}^*_n\}\big) - \frac{1}{4}\sum^5_{m=1}\sum^N_{n=1}\big\|\bar{\bf x}_m^{\text{\tiny(n)}} - \hat{\bf x}_m^{\text{\tiny(n)}}\big\|^2_2 + 2\sum^M_{m=6}\sum^N_{n=1}\big\|\dot{\bf x}_m^{\text{\tiny(n)}} - \ddot{\bf x}_m^{\text{\tiny(n)}}\big\|^2_2\\
    &\leq \frac{1}{2}d\big(\{{\bf X}^{\circ}_m,{\bf P}^*_n\}\big) + 2\sum^M_{m=6}\sum^N_{n=1}\big\|\dot{\bf x}_m^{\text{\tiny(n)}} - \ddot{\bf x}_m^{\text{\tiny(n)}}\big\|^2_2
\end{split}
\end{equation}
On the other hand, since the projection matrices ${\bf P}^{\circ}_n$ and ${\bf P}^*_n$ have the same third rows, the correspondences of the solution $\{{\bf X}^{\circ}_m,\frac{1}{2}({\bf P}^{\circ}_n + {\bf P}^*_n)\}$ will be
\begin{equation*}
    \bigg\{\frac{1}{2}\big(\breve{\bf x}_m^{\text{\tiny(1)}} + \bar{\bf x}_m^{\text{\tiny(1)}}\big)~\leftrightarrow~\frac{1}{2}\big(\breve{\bf x}_m^{\text{\tiny(2)}} + \bar{\bf x}_m^{\text{\tiny(2)}}\big)~\leftrightarrow~\cdots~\leftrightarrow~ \frac{1}{2}\big(\breve{\bf x}_m^{\text{\tiny(N)}} + \bar{\bf x}_m^{\text{\tiny(N)}}\big)\bigg\}^M_{m=1\, .}
\end{equation*}
In addition, $\{{\bf X}^{\circ}_m,\frac{1}{2}({\bf P}^{\circ}_n + {\bf P}^*_n)\}$ satisfies the constraints of~\eqref{eq:Opt_CF_more6} w.r.t. $\{\ddot{\bf x}_m^{\text{\tiny(1)}}\leftrightarrow\ddot{\bf x}_m^{\text{\tiny(2)}}\leftrightarrow\cdots\leftrightarrow\ddot{\bf x}_m^{\text{\tiny(N)}}\}$ because
\begin{equation*}
    \qquad\quad\frac{1}{2}\big({\bf P}^{\circ}_n+{\bf P}^*_n\big)\begin{bmatrix}{\bf X}^{\circ}_m\\[1mm]1\end{bmatrix}\sim\frac{1}{2}\big(\breve{\bf x}_m^{\text{\tiny(n)}} + \bar{\bf x}_m^{\text{\tiny(n)}}\big) = \frac{1}{2}\big(\hat{\bf x}_m^{\text{\tiny(n)}} + \bar{\bf x}_m^{\text{\tiny(n)}}\big) = \ddot{\bf x}_m^{\text{\tiny(n)}} \qquad \text{for all}\quad n \quad \text{and}\quad m = 1,2\ldots,5\,.
\end{equation*}
Indeed, although $\{\ddot{\bf X}_m,\ddot{\bf P}_n\}$ is not the global optimal solution for~\eqref{eq:Opt_CF_more6} w.r.t. $\{\ddot{\bf x}_m^{\text{\tiny(1)}}\leftrightarrow\ddot{\bf x}_m^{\text{\tiny(2)}}\leftrightarrow\cdots\leftrightarrow\ddot{\bf x}_m^{\text{\tiny(N)}}\}$, we still expect that it is better than $\{{\bf X}^{\circ}_m,\frac{1}{2}({\bf P}^{\circ}_n + {\bf P}^*_n)\}$ for this optimization. It means that
\begin{equation}\label{eq:Ineq_2}
\begin{split}
    \sum^M_{m=6}\sum^N_{n=1}\big\|\dot{\bf x}_m^{\text{\tiny(n)}} - \ddot{\bf x}_m^{\text{\tiny(n)}}\big\|^2_2 &~\leq~  \sum^M_{m=6}\sum^N_{n=1}\Big\|\frac{1}{2}\big(\bar{\bf x}_m^{\text{\tiny(n)}} + \breve{\bf x}_m^{\text{\tiny(n)}}\big) - \ddot{\bf x}_m^{\text{\tiny(n)}}\Big\|^2_2\\
    &~=~\sum^M_{m=6}\sum^N_{n=1}\Big\|\frac{1}{2}\big(\bar{\bf x}_m^{\text{\tiny(n)}} + \breve{\bf x}_m^{\text{\tiny(n)}}\big) - \frac{1}{2}\big(\bar{\bf x}_m^{\text{\tiny(n)}} + \hat{\bf x}_m^{\text{\tiny(n)}}\big)\Big\|^2_2\\
    &~=~ \frac{1}{4}d\big(\{{\bf X}^{\circ}_m,{\bf P}^{\circ}_n\}\big)\,.
\end{split}
\end{equation}
\begin{algorithm}[t!]
\caption{Iterative solutions for WPfC (\textbf{WPfC})}
\label{Alg:Iteration}
{
{\bf Input:} $\big\{\hat{\bf x}^{\text{\tiny(1)}}_m \,\leftrightarrow\,\hat{\bf x}^{\text{\tiny(2)}}_m \,\leftrightarrow\,\cdots \,\leftrightarrow\,\hat{\bf x}^{\text{\tiny(N)}}_m\big\}^M_{m=1}$~:~ Correspondences.\\[-2mm]

{\bf Conditions:} $M\geq 6$ and $N\geq 5$.\\[-2mm]

{\bf Implementation:}\\[-2mm]

{\bf 1.} $\{{\bf X}^{\circ}_m,{\bf P}^{\circ}_n\}~\overset{\text{CF-WPfC alg.}}{\longleftarrow}~\big\{\hat{\bf x}^{\text{\tiny(1)}}_m \,\leftrightarrow\,\hat{\bf x}^{\text{\tiny(2)}}_m \,\leftrightarrow\,\cdots \,\leftrightarrow\,\hat{\bf x}^{\text{\tiny(N)}}_m\big\}$\,.\\[-2mm]

{\bf 2.} $\text{Eval} ~~\overset{\eqref{eq:Opt_WPfC}}{\longleftarrow}~~d\big(\{{\bf X}^{\circ}_m,{\bf P}^{\circ}_n\}\big)$: an objective value,\\[-2mm]

\hspace{3.5mm} $\{{\bf X}^{\star}_m,{\bf P}^{\star}_n\} ~~\longleftarrow~~\{{\bf X}^{\circ}_m,{\bf P}^{\circ}_n\}$\,.\\[-2mm]

{\bf 3.} Repeat several times the following steps:\\[-2mm]

\hspace{6mm}{\bf a.} $\{\breve{\bf x}_m^{\text{\tiny(1)}}\leftrightarrow\breve{\bf x}_m^{\text{\tiny(2)}}\leftrightarrow\cdots\leftrightarrow\breve{\bf x}_m^{\text{\tiny(N)}}\}~~\overset{\eqref{eq:Corresponences_Est}}{\longleftarrow}~~\{{\bf X}^{\circ}_m,{\bf P}^{\circ}_n\}$\,,\\[-2mm]

\hspace{6mm}{\bf b.} for each $n$,~~ $\{\bar{\bf x}_m^{\text{\tiny(n)}}\}^M_{m=1}~~\overset{\text{Absolute orientation}}{\longleftarrow}~~\{\breve{\bf x}_m^{\text{\tiny(n)}}\}^M_{m=1}$\,,\\[-2mm]

\hspace{6mm}{\bf c.} create a new correspondences $\{\ddot{\bf x}_m^{\text{\tiny(1)}}\leftrightarrow\ddot{\bf x}_m^{\text{\tiny(2)}}\leftrightarrow\cdots\leftrightarrow\ddot{\bf x}_m^{\text{\tiny(N)}}\}$~:~~$\ddot{\bf x}_m^{\text{\tiny(n)}} = \frac{1}{2}\big(\bar{\bf x}_m^{\text{\tiny(n)}} + \hat{\bf x}_m^{\text{\tiny(n)}}\big)$\,,\\[-2mm]

\hspace{6mm}{\bf d.} $\{{\bf X}^{*}_m,{\bf P}^{*}_n\}~\overset{\text{CF-WPfC alg.}}{\longleftarrow}~\big\{\ddot{\bf x}^{\text{\tiny(1)}}_m \,\leftrightarrow\,\ddot{\bf x}^{\text{\tiny(2)}}_m \,\leftrightarrow\,\cdots \,\leftrightarrow\,\ddot{\bf x}^{\text{\tiny(N)}}_m\big\}$\,,\\[-2mm]

\hspace{6mm}{\bf e.} if $d\big(\{{\bf X}^{*}_m,{\bf P}^{*}_n\}\big) < \text{Eval}$:~~ $\{{\bf X}^{\star}_m,{\bf P}^{\star}_n\} ~~\longleftarrow~~\{{\bf X}^{*}_m,{\bf P}^{*}_n\}$\\[-2mm]

\hspace{51.5mm}$\text{Eval} ~~\overset{\eqref{eq:Opt_WPfC}}{\longleftarrow}~~ d\big(\{{\bf X}^{*}_m,{\bf P}^{*}_n\}\big)$\,,\\[-2mm]

\hspace{6mm}{\bf f.} $\{{\bf X}^{\circ}_m,{\bf P}^{\circ}_n\}~~\longleftarrow~~\{{\bf X}^{*}_m,{\bf P}^{*}_n\}$\,.\\[-2mm]

{\bf Output:} $\big\{{\bf X}^{\star}_m, {\bf P}^{\star}_n\big\}$\,.\\[-1mm]
}
\end{algorithm}

\noindent Because of $d\big(\{{\bf X}^{\circ}_m,{\bf P}^*_n\}\big) \leq d\big(\{{\bf X}^{\circ}_m,{\bf P}^{\circ}_n\}\big)$,~\eqref{eq:Ineq_1} and~\eqref{eq:Ineq_2} yield
\begin{equation*}
    d\big(\{\ddot{\bf X}_m,\ddot{\bf P}_n\}\big) ~\leq~ \frac{1}{2}d\big(\{{\bf X}^{\circ}_m,{\bf P}^*_n\}\big) + \frac{1}{2}d\big(\{{\bf X}^{\circ}_m,{\bf P}^{\circ}_n\}\big) ~\leq~ d\big(\{{\bf X}^{\circ}_m,{\bf P}^{\circ}_n\}\big)\,.
\end{equation*}
Therefore $\{\ddot{\bf X}_m,\ddot{\bf P}_n\}$ is better than $\{{\bf X}^{\circ}_m,{\bf P}^{\circ}_n\}$, i.e.,
\begin{equation}\label{eq:Ineq_3}
    \sum^M_{m=1}\sum^N_{n=1}\big\|\dot{\bf x}_m^{\text{\tiny(n)}} - \hat{\bf x}_m^{\text{\tiny(n)}}\big\|^2_2 ~\leq~ \sum^M_{m=1}\sum^N_{n=1}\big\|\breve{\bf x}_m^{\text{\tiny(n)}} - \hat{\bf x}_m^{\text{\tiny(n)}}\big\|^2_2\, .
\end{equation}
Fact~\ref{Fact:Orientation} and~\eqref{eq:Ineq_3} imply
\begin{equation}\label{eq:Ineq_4}
    \sum^M_{m=1}\sum^N_{n=1}\big\|\tilde{\bf x}_m^{\text{\tiny(n)}} - \hat{\bf x}_m^{\text{\tiny(n)}}\big\|^2_2 ~\leq~ \sum^M_{m=1}\sum^N_{n=1}\big\|\bar{\bf x}_m^{\text{\tiny(n)}} - \hat{\bf x}_m^{\text{\tiny(n)}}\big\|^2_2\, .
\end{equation}
Finally, since
\[\{{\bf X}^{**}_m,{\bf P}^{**}_n\} = \texttt{CF-WPfC}(\{\tilde{\bf x}_m^{\text{\tiny(1)}}\leftrightarrow\tilde{\bf x}_m^{\text{\tiny(2)}}\leftrightarrow\cdots\leftrightarrow\tilde{\bf x}_m^{\text{\tiny(N)}}\})\]  \[\{{\bf X}^{\circ}_m,{\bf P}^{*}_n\} = \texttt{CF-WPfC}(\{\bar{\bf x}_m^{\text{\tiny(1)}}\leftrightarrow\bar{\bf x}_m^{\text{\tiny(2)}}\leftrightarrow\cdots\leftrightarrow\bar{\bf x}_m^{\text{\tiny(N)}}\})\,\]~\eqref{eq:Ineq_4} confirms that
$\{{\bf X}^{**}_m,{\bf P}^{**}_n\}$ is better than $\{{\bf X}^{\circ}_m,{\bf P}^{*}_n\}$. Proposition~\ref{Prop:Opt_by_Iter} is proven.
\end{proof}

The results given by Propositions~\ref{Prop:Opt_by_Iso} and~\ref{Prop:Opt_by_Iter} allow us to build an iterative process to find good solutions from the closed-form solution given by Algorithm~\ref{Alg:Closed_form}. This iterative process is called the \emph{Iterative solutions for WPfC} and presented by Algorithm~\ref{Alg:Iteration}. In this work, the output
\begin{equation*}
    \{{\bf X}^{\star}_m,{\bf P}^{\star}_n\}~=~\texttt{WPfC}\big(\{\hat{\bf x}_m^{\text{\tiny(1)}}\leftrightarrow\hat{\bf x}_m^{\text{\tiny(2)}}\leftrightarrow\cdots\leftrightarrow\hat{\bf x}_m^{\text{\tiny(N)}}\}^M_{m=1}\big)
\end{equation*}
of Algorithm~\ref{Alg:Iteration} is considered the final solution for the world points from correspondences (WPfC). Since the correctness of Fact~\ref{Fact:Orientation} is not always guaranteed, convergence of our proposed iterative process can be not occurred. We overcome this drawback by simply choosing the optimal solution if its objective value is the smallest at each time in step 3, and fix the number of times in step 3 by 20.

\newpage

\section{Creating Point clouds}\label{Sec:Point_Clouds}

\begin{figure*}[t]
  \centering
  \includegraphics[width=0.95\linewidth]{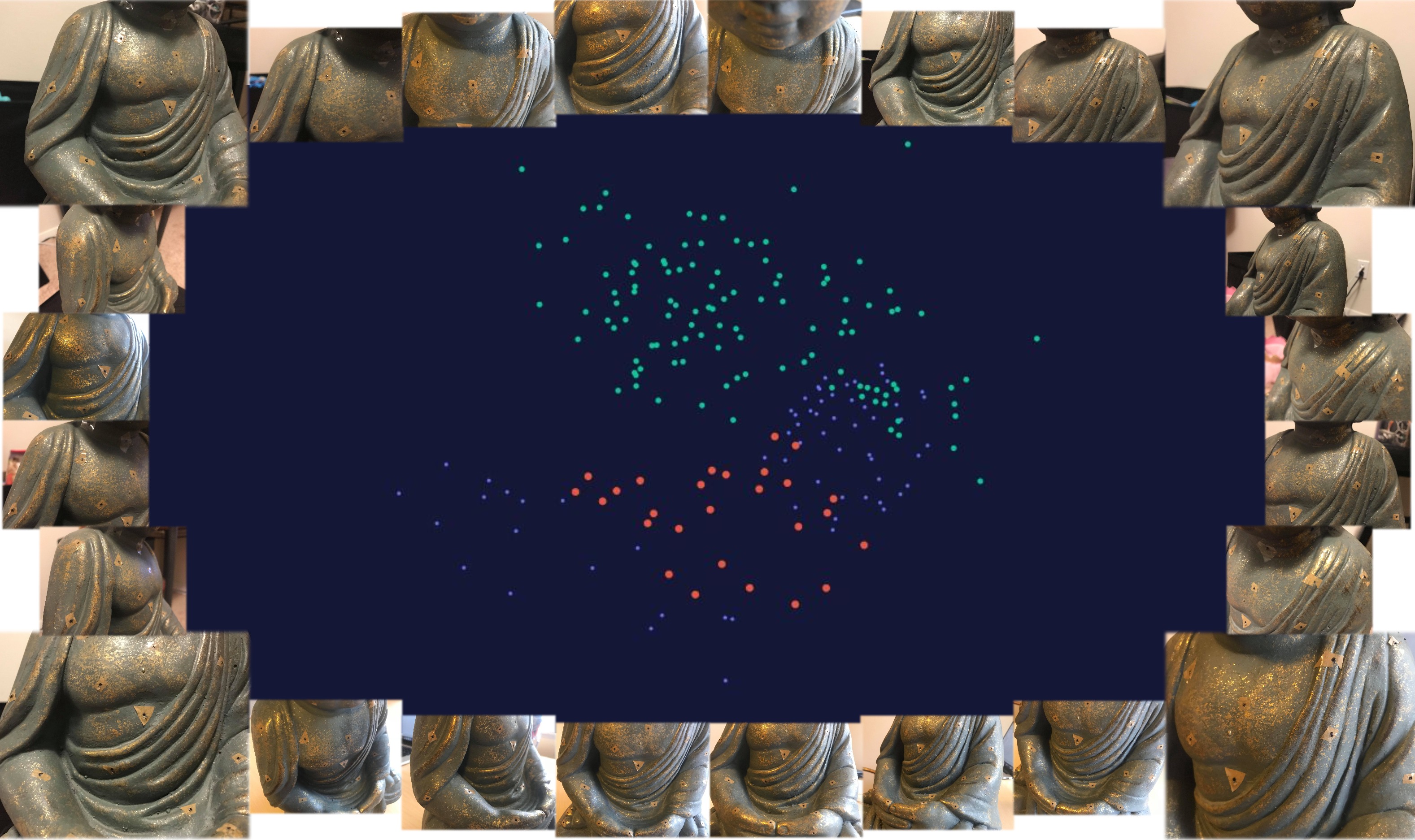}
  \caption{World points and camera positions from a 3D-reconstruction of the \emph{front body} of Buddha statue. This experiment has 116 images and 27 key-points. The results are estimated by Algorithm~\ref{Alg:Iteration} in which the \emph{red circles} are the estimated world points, the \emph{green circles} are the estimated camera positions. The \emph{small blue circles} are the other estimated world points from other experiments. They are added in the figure to simulate the 3D-shape of the Buddha statue. The Buddha statue is horizontal. Some examples of the correspondences can be found from the key-points on the twenty five images in the figure.}
  \label{Fig:Sec2}
\end{figure*}

We begin this section by summarizing our results that we had in the previous section. The results are simply explained by Figure~\ref{Fig:Sec2} that presents results from our project on 3D reconstructing \emph{front body of the Buddha statue}. Given 116 images, we determine 27 key-points and then correspondences. In our work, we use SURF~\citep{Bay2006} and SIFT~\citep{Lowe1999} to detect the key-points on each images, and then create the correspondences. The estimated correspondences are refined by~\citet{le2022multi} before using as the input of both Algorithms~\ref{Alg:Closed_form} and~\ref{Alg:Iteration}. The output of the Algorithm~\ref{Alg:Iteration} are the 27 world points indicated by the \emph{red circles} and the 116 camera positions indicated by the \emph{green circles}. Of course, from these 27 world points, we cannot recognize the Buddha statue. We need more points to get the 3D-shape of the statue, then construct its surface to complete its 3D-reconstruction.

This section studies how to create new points from the images, the world points and the camera positions derived by the previous section. These new points are called \emph{point clouds}. Before going further on the study of creating point clouds, we discuss some important tasks. First, points clouds should be almost everywhere on the object's surface. Second, points in point clouds will be evaluated with a metric. From there we know what is good and what is bad. Finally, a creating point clouds process should be simple, so we can recover bad points or missing surfaces by adding new images.

Discussing on the first task, we shall see a disadvantage of existence methods on 3D-reconstruction. In general, the existence methods build point clouds based on features they detect on images~\citep{Munkres1957,Lowe1999,Mikolajczyk2004,Bay2006,Myronenko2010}. Therefore, surfaces with weak or translucent features cannot be reconstructed. Of course, to overcome this disadvantage, we can change properties of features for successful point detection on these surfaces. In the viewpoint of this approach, features detection will depend on applications.

In this section, we propose a new approach to feature detection. A proposed feature of a point does not depend on neighbors of this point as the existence methods do. This feature is defined based on geometrical properties of this point to some given points. We name this feature by a \emph{geodesic feature} or shortly by a \emph{geo-feature}. Given two world points on the object's surface, a geodesic between these two points is defined as the shortest curve on the object's surface connecting two points~\citep{Busemann2012}. Then an unknown world point has a `\emph{geo-feature}' if there exist two know world points that this unknown world point lays on their geodesic. We shall introduce and study carefully the geo-feature in Subsection~\ref{Subsec:Geo_feature} after we list some needed notations for its introduction.

\subsection{Notations}\label{Subsec:Notations}

Given $K$ world points $\{{\bf X}_1,{\bf X}_2,\ldots,{\bf X}_K\}$, $N$ projection matrices $\{{\bf P}^{\star}_1,{\bf P}^{\star}_2,\ldots,{\bf P}^{\star}_N\}$ and their correspondences $\{{\bf x}_k^{\text{\tiny(1)}}\leftrightarrow{\bf x}_k^{\text{\tiny(2)}}\leftrightarrow\cdots\leftrightarrow{\bf x}_k^{\text{\tiny(N)}}\}^K_{k=1}$ as we can acquire from Section~\ref{Sec:World_points}. We use the symbol $\star$ for the projection matrices because all the projection matrices used in this section are from Algorithm~\ref{Alg:Iteration} in Section~\ref{Sec:World_points}, and use $K$ to replace $M$, the number of world points, because the number of world points will be increased in the process of creating point clouds $(K \gg M)$. Let us denote ${\bf Img}_n$ be a \emph{gray-scale image} from the $n$th camera. It is an $H\times W$ matrix and its elements are real values from zero to one. ${\bf Img}_n(h,w)$ or ${\bf Img}_n({\bf x}_k^{\text{\tiny(n)}})$ is the value of ${\bf Img}_n$ at the $(h,w)$-th pixel or the image point ${\bf x}_k^{\text{\tiny(n)}}$. Note that the image point ${\bf x}_k^{\text{\tiny(n)}}$ would be in the rectangle $[0,H]\times[0,W]$. Otherwise, we think that this image point does not exist. In addition, for convenience we accept the image point ${\bf x}_k^{\text{\tiny(n)}}$ as a real vector, not must be integer vector. When two coordinates of ${\bf x}_k^{\text{\tiny(n)}}$ are not integers, they are replaced by their closest integers to determine the value of ${\bf Img}_n({\bf x}_k^{\text{\tiny(n)}})$.

From any two world points ${\bf X}_k$, ${\bf X}_{k'}$ on the 3D object's surface, we denote $\mathcal{G}({\bf X}_k,{\bf X}_{k'})$ be a \emph{geodesic} between ${\bf X}_k$ and ${\bf X}_{k'}$, the shortest curve on the 3D object's surface connecting ${\bf X}_k$ and ${\bf X}_{k'}$. A motivation of this section is a study on how to reconstruct this geodesic. Given an image ${\bf Img}_n$ capturing both ${\bf X}_k$ and ${\bf X}_{k'}$, we define
\begin{equation}\label{eq:Image_Line}
{\bf Img}_n({\bf X}_k,{\bf X}_{k'}|L) ~=~ \begin{bmatrix}{\bf Img}_n({\bf x}_k^{\text{\tiny(n)}})\\[1mm]{\bf Img}_n\Big(\frac{1}{L-1}\left[(L-2){\bf x}_k^{\text{\tiny(n)}} + {\bf x}_{k'}^{\text{\tiny(n)}}\right]\Big)\\[2.5mm]{\bf Img}_n\Big(\frac{1}{L-1}\left[(L-3){\bf x}_k^{\text{\tiny(n)}} + 2{\bf x}_{k'}^{\text{\tiny(n)}}\right]\Big)\\\vdots\\[0.5mm]{\bf Img}_n\Big(\frac{1}{L-1}\left[{\bf x}_k^{\text{\tiny(n)}} + (L-2){\bf x}_{k'}^{\text{\tiny(n)}}\right]\Big)\\[2.5mm]{\bf Img}_n({\bf x}_{k'}^{\text{\tiny(n)}})\end{bmatrix}_{L\times 1}
\end{equation}
then normalize this vector, i.e., $\|{\bf Img}_n({\bf X}_k,{\bf X}_{k'}|L)\|_2 = 1$. This is the column vector of length $L$ that contains all image points on a straight line between the image points ${\bf x}_k^{\text{\tiny(n)}}$ of ${\bf X}_k$ and ${\bf x}_{k'}^{\text{\tiny(n)}}$ of ${\bf X}_{k'}$ on the image ${\bf Img}_n$. When $L = \|{\bf x}_k^{\text{\tiny(n)}} - {\bf x}_{k'}^{\text{\tiny(n)}}\|_{\infty}$, a $\mathcal{L}_{\infty}$-norm or maximum norm of the vector ${\bf x}_k^{\text{\tiny(n)}} - {\bf x}_{k'}^{\text{\tiny(n)}}$, each value at a pixel in the straight line between ${\bf x}_k^{\text{\tiny(n)}}$ and ${\bf x}_{k'}^{\text{\tiny(n)}}$ will correspond to one element in the vector ${\bf Img}_n({\bf X}_k,{\bf X}_{k'}|L)$, and when $L$ is double, each value at one pixel will correspond to two elements in the vector. We call this vector the \emph{image curve} between ${\bf X}_k$ and ${\bf X}_{k'}$ on ${\bf Img}_n$ and expect that this image curve is close to the image of the geodesic $\mathcal{G}({\bf X}_k,{\bf X}_{k'})$. Hence, from a lot of these image curves, we can obtain some information of the geodesic and reconstruct some points on it.

Next we introduce a concept of \emph{similar} for a group of images with respect to (w.r.t.) two world points.

\begin{defi}\label{Def:Similar_two_WPs}
Given a threshold $\theta\in(0,1)$ and two world points ${\bf X}_k, {\bf X}_{k'}$, a group of $J$ images $\{{\bf Img}_{n_1}$, ${\bf Img}_{n_2}$,$\ldots,{\bf Img}_{n_J}\}$ is called similar w.r.t. ${\bf X}_k, {\bf X}_{k'}$ and $\theta$ if
\begin{equation}\label{eq:Similar_two_WPs}
\big\|{\bf \Psi}^T{\bf \Psi} ~-~ {\bf 1}_{J\times J}\big\|_{\infty} ~\leq~ \theta\, ,
\end{equation}
where
\begin{equation}\label{eq:Psi}
{\bf \Psi} ~=~ \begin{bmatrix}{\bf Img}_{n_1}({\bf X}_k,{\bf X}_{k'}|L)^T \\[1mm] {\bf Img}_{n_2}({\bf X}_k,{\bf X}_{k'}|L)^T \\ \vdots \\ {\bf Img}_{n_J}({\bf X}_k,{\bf X}_{k'}|L)^T\end{bmatrix}^T_{L\times J}
\end{equation}
and
\begin{equation}\label{eq:Length}
L~=~\max\Big\{\big\|{\bf x}_k^{\text{\tiny$(n_1)$}} - {\bf x}_{k'}^{\text{\tiny$(n_1)$}}\big\|_{\infty}~,~\big\|{\bf x}_k^{\text{\tiny$(n_2)$}} - {\bf x}_{k'}^{\text{\tiny$(n_2)$}}\big\|_{\infty}~,\,\ldots\,,~\big\|{\bf x}_k^{\text{\tiny$(n_J)$}} - {\bf x}_{k'}^{\text{\tiny$(n_J)$}}\big\|_{\infty}\Big\}\,.
\end{equation}
\end{defi}

Since all columns ${\bf Img}_{n_j}({\bf X}_k,{\bf X}_{k'}|L)$ of ${\bf \Psi}$ are normalized, the $(i,j)$-th element of ${\bf \Psi}^T{\bf \Psi}$ is a Pearson correlation of two vectors ${\bf Img}_{n_i}({\bf X}_k,{\bf X}_{k'}|L)$ and ${\bf Img}_{n_j}({\bf X}_k,{\bf X}_{k'}|L)$. Thus, if the group $\{{\bf Img}_{n_1}, {\bf Img}_{n_2},\ldots,{\bf Img}_{n_J}\}$ is similar w.r.t. $\{{\bf X}_k, {\bf X}_{k'}, \theta\}$, then any two image curves of ${\bf X}_k$ and ${\bf X}_{k'}$ on two images $n_i$ and $n_j$ has a Pearson correlation at least $1-\theta$. Hence, when $\theta\simeq 0$, these curves are linearly dependent. This rare situation results two cases that (i) all these $J$ images capture two world points ${\bf X}_k$ and ${\bf X}_{k'}$ in the same direction, or (ii) the geodesic $\mathcal{G}({\bf X}_k,{\bf X}_{k'})$ is a straight line. We easily check the first case by comparing the camera positions and ${\bf X}_k, {\bf X}_{k'}$. If it is not the first but the second case, it is  useful for 3D reconstruction because all points in the straight line $\mathcal{G}({\bf X}_k,{\bf X}_{k'})$ can be reconstructed by these $J$ images. In practice, because of (i) non-straight line of the geodesic $\mathcal{G}({\bf X}_k,{\bf X}_{k'})$, (ii) distortions in images and others, it is  difficult to catch this situation with $\theta\simeq 0$. However, what is happen when the group $\{{\bf Img}_{n_1}, {\bf Img}_{n_2},\ldots,{\bf Img}_{n_J}\}$ is similar w.r.t. $\{{\bf X}_k, {\bf X}_{k'}, \theta\}$ with $\theta \simeq 0.3$? We expect that although this situation will not help us to reconstruct all points, it will help us to reconstruct some special points in $\mathcal{G}({\bf X}_k,{\bf X}_{k'})$.

Finally, we finish this subsection by introducing a metric for points in point clouds. Assuming that $\{{\bf x}^{\text{\tiny$(n_1)$}}\leftrightarrow{\bf x}^{\text{\tiny$(n_2)$}}\leftrightarrow\cdots\leftrightarrow{\bf x}^{\text{\tiny$(n_J)$}}\}$ is a correspondence found on the $J$-group $\{{\bf Img}_{n_1}, {\bf Img}_{n_2},\ldots,{\bf Img}_{n_J}\}$ based on the geo-feature as we study in Subsection~\ref{Subsec:Geo_feature}. From the given projection matrices $\{{\bf P}^{\star}_{n_1}, {\bf P}^{\star}_{n_2},\ldots,{\bf P}^{\star}_{n_J}\}$, we compute the $2J\times 4$ matrix $\mathcal{H}(\{{\bf P}^{\star}_{n_j}\})$ and determine a world point ${\bf X}$ as~\eqref{eq:Point_from_Camera}. The correct correspondence $\{\bar{\bf x}^{\text{\tiny$(n_1)$}}\leftrightarrow\bar{\bf x}^{\text{\tiny$(n_2)$}}\leftrightarrow\cdots\leftrightarrow\bar{\bf x}^{\text{\tiny$(n_J)$}}\}$ for the new world point is given by
\begin{equation}\label{eq:Image_point}
    \qquad\qquad\bar{\bf x}^{\text{\tiny$(n_j)$}}~=~ \left[\frac{{\bf p}^{\star}_{1,n_j}\big[{\bf X}^T\,,\,1\big]^T}{{\bf p}^{\star}_{3,n_j}\big[{\bf X}^T\,,\,1\big]^T}~,~\frac{{\bf p}^{\star}_{2,n_j}\big[{\bf X}^T\,,\,1\big]^T}{{\bf p}^{\star}_{3,n_j}\big[{\bf X}^T\,,\,1\big]^T}\right]^T\qquad \text{for}\quad j = 1,2\ldots,J\, .
\end{equation}
Thus, we use the \emph{mean differences} between these two correspondences to evaluate the new point ${\bf X}$. Concretely, we denote this metric by $\mathcal{E}_{\text{d}}$, and the metric of ${\bf X}$ is defined by
\begin{equation}\label{eq:Metric_point}
\mathcal{E}_{\text{d}}({\bf X}) ~=~ \frac{1}{J}\sum^J_{j=1}\big\|{\bf x}^{\text{\tiny$(n_j)$}} - \bar{\bf x}^{\text{\tiny$(n_j)$}}\big\|_{2\, .}
\end{equation}

Beside the metric $\mathcal{E}_{\text{d}}$, we evaluate the estimation of the world point ${\bf X}$ based on the number of images using for this estimation. We denote this metric by $\mathcal{E}_{\text{i}}$ and $\mathcal{E}_{\text{i}}({\bf X})$ will be $J$. Therefore, the new point ${\bf X}$ is good if $\mathcal{E}_{\text{d}}({\bf X})$ is low and $\mathcal{E}_{\text{i}}({\bf X})$ is high.

\subsection{Geodesic feature}\label{Subsec:Geo_feature}

Assuming that the group of $J$ images $\{{\bf Img}_{n_1}, {\bf Img}_{n_2},\ldots,{\bf Img}_{n_J}\}$ is similar w.r.t. two world points ${\bf X}_k, {\bf X}_{k'}$ and a threshold $\theta$. Given a length $\ell < L$ where $L$ is given by~\eqref{eq:Length}, we are interested with the following optimization
\begin{equation}\label{eq:Opt_find_feature}
\arg\hspace{-5.5mm}\min_{\hspace{-6mm}1\leq l_1,l_2,\ldots,l_J\leq L-\ell}\Big\|{\bf \Psi}^T_{l_1l_2\ldots l_J}{\bf \Psi}_{l_1l_2\ldots l_J}~-~{\bf 1}_{J\times J}\Big\|_{\infty}
\end{equation}
where ${\bf \Psi}_{l_1l_2\ldots l_J}$ is a $\ell\times J$ matrix that its $j$-th column is a sub-column beginning at $l_j$ and ending at $l_j+\ell-1$ of the $j$-th column of ${\bf \Psi}$, i.e.,
\begin{equation}
    \big[{\bf \Psi}_{l_1l_2\ldots l_J}\big]_j ~=~ \big[{\bf \Psi}(l_j,j)\,,\,{\bf \Psi}(l_j+1,j)\,,\,{\bf \Psi}(l_j+2,j)\,,\,\ldots\,,\,{\bf \Psi}(l_j+\ell-1,j)\big]^T
\end{equation}
with $\big[{\bf \Psi}_{l_1l_2\ldots l_J}\big]_j$ means the $j$-th column of ${\bf \Psi}_{l_1l_2\ldots l_J}$ and ${\bf \Psi}(i,j)$ mean the $(i,j)$-th element of ${\bf \Psi}$. If $\{l^*_1, l^*_2,\ldots,l^*_J\}$ is the optimal solution for~\eqref{eq:Opt_find_feature}, it is obvious that
\begin{equation*}
    \qquad\text{Corr}\Big(\big[{\bf \Psi}_{l^*_1l^*_2\ldots l^*_J}\big]_i\,,\,\big[{\bf \Psi}_{l^*_1l^*_2\ldots l^*_J}\big]_j\Big) ~\gg~ \text{Corr}\big({\bf \Psi}_i\,,\,{\bf \Psi}_j\big) > 1-\theta\,\qquad \text{for all}\quad i,j\,,
\end{equation*}
where ${\bf \Psi}_j = {\bf Img}_{n_j}({\bf X}_k,{\bf X}_{k'}|L)$ be the $j$-th column of ${\bf \Psi}$. We expect that when $\theta \simeq 0.3$ and $\ell \simeq L/10$, the correlation between two vectors $[{\bf \Psi}_{l^*_1l^*_2\ldots l^*_J}]_i$ and $[{\bf \Psi}_{l^*_1l^*_2\ldots l^*_J}]_j$ is at least 0.9 for all $i,j$. Then as our previous discussion, all the sub image curves $\big\{[{\bf \Psi}_{l^*_1l^*_2\ldots l^*_J}]_j\big\}^J_{j=1}$ are close to images of a sub-geodesic of $\mathcal{G}({\bf X}_k,{\bf X}_{k'})$. In addition, this sub-geodesic is nearly straight line. Indeed, in theory we can reconstruct all points in this sub-geodesic. In our work, we only reconstruct the mid-point of this sub-geodesic based on a correspondence from the mid-points of $\big\{[{\bf \Psi}_{l^*_1l^*_2\ldots l^*_J}]_j\big\}^J_{j=1}$. Concretely, the following definition of the geo-feature explains how we find a new world point to create the point clouds.

\begin{defi}\label{Def:Geo-feature}
Given a similar group $\{{\bf Img}_{n_1},{\bf Img}_{n_2},\ldots,{\bf Img}_{n_J}\}$ w.r.t. $\{{\bf X}_k, {\bf X}_{k'}, \theta\}$, let $\{l^*_1,l^*_2,\ldots,l^*_J\}$ be the optimal solution for~\eqref{eq:Opt_find_feature} corresponding to $\ell$. A correspondence $\{{\bf x}^{\text{\tiny$(n_1)$}}\leftrightarrow{\bf x}^{\text{\tiny$(n_2)$}}\leftrightarrow\cdots\leftrightarrow{\bf x}^{\text{\tiny$(n_J)$}}\}$ where
\begin{equation}\label{eq:Geo-feature}
{\bf x}^{\text{\tiny$(n_j)$}}~=~\frac{L - 1 - (l^*_j \,+\, \ell/2)}{L-1}{\bf x}_k^{\text{\tiny$(n_j)$}} + \frac{l^*_j + \ell/2}{L-1}{\bf x}_{k'}^{\text{\tiny$(n_j)$}}
\end{equation}
is called a geodesic feature (geo-feature) w.r.t. ${\bf X}_k, {\bf X}_{k'}, \theta$ and $\ell$.
\end{defi}

When the geo-feature $\{{\bf x}^{\text{\tiny$(n_1)$}}\leftrightarrow{\bf x}^{\text{\tiny$(n_2)$}}\leftrightarrow\cdots\leftrightarrow{\bf x}^{\text{\tiny$(n_J)$}}\}$ is found, we use~\eqref{eq:Point_from_Camera}, together with the projection matrix $\{{\bf P}^{\star}_{n_j}\}$, to create a new world point ${\bf X}$. If the metrics $\mathcal{E}_d$ and $\mathcal{E}_i$ of the new point is small and large, respectively, we add this point in the point clouds and reconstruct other points in the geodesic $\mathcal{G}({\bf X}_k,{\bf X}_{k'})$ by applying the same process with replacing the couple $\{{\bf X}_k,{\bf X}_{k'}\}$ by $\{{\bf X}_k,{\bf X}\}$ or $\{{\bf X},{\bf X}_{k'}\}$.

\subsection{Algorithm}\label{Subsec:Alg}

An initialization of our process on creating the point clouds is a group of $M$ world points $\{{\bf X}^{\star}_1, {\bf X}^{\star}_2, \ldots, {\bf X}^{\star}_M\}$ and $N$ projection matrices $\{{\bf P}^{\star}_1, {\bf P}^{\star}_2, \ldots, {\bf P}^{\star}_N\}$, the output of Algorithm~\ref{Alg:Iteration}. We choose all a couple of two world points ${\bf X}^{\star}_k$ and ${\bf X}^{\star}_{k'}$, and use $N$ images ${\bf Img}_1, {\bf Img}_2, \ldots, {\bf Img}_N$, together with two groups of image points $\{{\bf x}^{\text{\tiny(n)}}_k\}^N_{n=1}$ of ${\bf X}^{\star}_k$ and $\{{\bf x}^{\text{\tiny(n)}}_{k'}\}$ of ${\bf X}^{\star}_{k'}$ computing based on the projection matrices $\{{\bf P}^{\star}_n\}^N_{n=1}$, to find as much as possible geo-features w.r.t. this couple.

Beside a criterion of geo-feature used as a sign of good point for the point clouds, we can use a small value of $\mathcal{E}_d$ and a large value of $\mathcal{E}_i$ of a point to guarantee that this is a good point. For example, in our project of the Buddha's statue 3D reconstruction, if a point has its $\mathcal{E}_d$ value at most 25 and its $\mathcal{E}_i$ value at least 20, i.e., this point is estimated by at least 20 images and a mean difference between image points of this point before and after estimation is smaller than 25 pixels, would be a good point. Thus, we can create lots of candidates then use their metrics $\mathcal{E}_d$ and $\mathcal{E}_i$ to recognize good candidate to add in the point clouds. Based on this viewpoint, we shall simplify a process of finding geo-features based on the \emph{similar groups} given by Definition~\ref{Def:Similar_two_WPs} and the optimal solutions for~\eqref{eq:Opt_find_feature} by a simple process of finding \emph{candidates} that

\begin{defi}\label{Def:Candidate_Geofeature}
Given a group of $J$ images $\{{\bf Img}_{n_1}\,,\,{\bf Img}_{n_2}, \ldots,{\bf Img}_{n_J}\}$, a correspondence $\{{\bf x}^{\text{\tiny$(n_1)$}}\leftrightarrow{\bf x}^{\text{\tiny$(n_2)$}}\leftrightarrow\cdots\leftrightarrow{\bf x}^{\text{\tiny$(n_J)$}}\}$ is called a candidate of a geo-feature w.r.t. $\{{\bf X}_{k}, {\bf X}_{k'}, \theta, \ell\}$ if
\begin{enumerate}
\item for all $1\leq i\neq j \leq J$, a group of two images $\{{\bf Img}_{n_i}\,,\,{\bf Img}_{n_j}\}$ is similar w.r.t. $\{{\bf X}_k,{\bf X}_{k'},\theta\}$\,,
\item ${\bf x}^{\text{\tiny$(n_j)$}}$ is in the straight line between ${\bf x}_k^{\text{\tiny$(n_j)$}}$ and ${\bf x}_{k'}^{\text{\tiny$(n_j)$}}$, i.e., there exists an integer number $l_j$ between $\ell/2$ to $L-1-\ell/2$ that
\begin{equation*}
    {\bf x}^{\text{\tiny$(n_j)$}}~=~\frac{L - 1 - l_j}{L-1}{\bf x}_k^{\text{\tiny$(n_j)$}} + \frac{l_j}{L-1}{\bf x}_{k'}^{\text{\tiny$(n_j)$}}\,,
\end{equation*}
\item and for all $1\leq i \neq j \leq J$,
\begin{equation*}
    \text{Corr}\Big(\psi\big({\bf x}^{\text{\tiny$(n_i)$}}\big)\,,\,\psi\big({\bf x}^{\text{\tiny$(n_j)$}}\big)\Big) ~\geq~ 1-\theta\,,
\end{equation*}
where
\begin{equation}\label{eq:Image_Line_ell}
    \psi\big({\bf x}^{\text{\tiny$(n_j)$}}\big)~=~\psi\big({\bf x}_k^{\text{\tiny$(n_j)$}},{\bf x}_{k'}^{\text{\tiny$(n_j)$}},l_j\big)~\defeq~\begin{bmatrix}\frac{L-1-(l_j-\ell/2)}{L-1}{\bf x}_k^{\text{\tiny$(n_j)$}} + \frac{l_j - \ell/2}{L-1}{\bf x}_{k'}^{\text{\tiny$(n_j)$}}\\[2.5mm]\frac{L-1-(l_j-\ell/2+1)}{L-1}{\bf x}_k^{\text{\tiny$(n_j)$}} + \frac{l_j - \ell/2 +1}{L-1}{\bf x}_{k'}^{\text{\tiny$(n_j)$}}\\[1mm]\vdots\\[1mm]\frac{L-1-(l_j+\ell/2-1)}{L-1}{\bf x}_k^{\text{\tiny$(n_j)$}} + \frac{l_j + \ell/2-1}{L-1}{\bf x}_{k'}^{\text{\tiny$(n_j)$}}\\[2.5mm]\frac{L-1-(l_j+\ell/2)}{L-1}{\bf x}_k^{\text{\tiny$(n_j)$}} + \frac{l_j + \ell/2}{L-1}{\bf x}_{k'}^{\text{\tiny$(n_j)$}}\end{bmatrix}_{\ell\times 1}
\end{equation}
be the $\ell\times 1$ vector in the straight line between ${\bf x}_k^{\text{\tiny$(n_j)$}}$ and ${\bf x}_{k'}^{\text{\tiny$(n_j)$}}$ with ${\bf x}^{\text{\tiny$(n_j)$}}$ be a middle.
\end{enumerate}
\end{defi}

Although motivating to find candidates for geo-features of $J$ images ($J \gg 2$), Definition~\ref{Def:Candidate_Geofeature} allows us to find these candidates by doing independently with any two images from $J$ images. More precisely, given a couple of two images ${\bf Img}_{n_i}$ and ${\bf Img}_{n_j}$, first we check
\begin{equation*}\text{Corr}\big({\bf Img}_{n_i}({\bf X}_k,{\bf X}_{k'}|L)\,,\,{\bf Img}_{n_j}({\bf X}_k,{\bf X}_{k'}|L)\big) ~\geq~ 1-\theta
\end{equation*}
where $L = \max\big\{\|{\bf x}_{k}^{\text{\tiny$(n_i)$}} - {\bf x}_{k'}^{\text{\tiny$(n_i)$}}\|_{\infty}\,,\,\|{\bf x}_{k}^{\text{\tiny$(n_j)$}} - {\bf x}_{k'}^{\text{\tiny$(n_j)$}}\|_{\infty}\big\}$ to confirm that the group $\{{\bf Img}_{n_i}\,,\,{\bf Img}_{n_j}\}$ is similar w.r.t. $\{{\bf X}_k,{\bf X}_{k'},\theta\}$. Second, we find all $l_i, l_j \in \{0,1,\ldots,L-1\}$ such that the third condition of Definition~\ref{Def:Candidate_Geofeature} occurs. Then we store the image points
\begin{equation*}
\begin{split}
{\bf x}^{\text{\tiny$(n_i)$}} ~&=~ \frac{L-1-l_i}{L-1}{\bf x}_k^{\text{\tiny$(n_i)$}} + \frac{l_i}{L-1}{\bf x}_{k'}^{\text{\tiny$(n_i)$}}\qquad \text{to}\qquad \mathcal{I}_{n_i}\,,\\[2mm]
{\bf x}^{\text{\tiny$(n_j)$}} ~&=~ \frac{L-1-l_j}{L-1}{\bf x}_k^{\text{\tiny$(n_j)$}} + \frac{l_j}{L-1}{\bf x}_{k'}^{\text{\tiny$(n_j)$}}\qquad \text{to}\qquad \mathcal{I}_{n_j}\,,
\end{split}
\end{equation*}
where $\mathcal{I}_{n_i},\, \mathcal{I}_{n_j}$ be sets of all candidates for the geo-feature w.r.t. $\{{\bf X}_k,{\bf X}_{k'},\theta,\ell\}$ on the images $n_i$ and $n_j$, respectively, and remember that ${\bf x}^{\text{\tiny$(n_i)$}}$ in $\mathcal{I}_{n_i}$ matches to ${\bf x}^{\text{\tiny$(n_j)$}}$ in $\mathcal{I}_{n_j}$. Here we have a notice that if (i) ${\bf x}_1^{\text{\tiny$(n_j)$}} \in {\bf Img}_{n_j}$ matches to ${\bf x}^{\text{\tiny$(n_i)$}} \in {\bf Img}_{n_i}$, (ii) another ${\bf x}_2^{\text{\tiny$(n_j)$}} \in {\bf Img}_{n_j}$ matches to ${\bf x}^{\text{\tiny$(n_{i'})$}} \in {\bf Img}_{n_{i'}}$, and (iii) ${\bf x}_1^{\text{\tiny$(n_j)$}}\neq{\bf x}_2^{\text{\tiny$(n_j)$}}$ but they are very close then we must conclude ${\bf x}_1^{\text{\tiny$(n_j)$}}$ matches to ${\bf x}^{\text{\tiny$(n_{i'})$}}$ and ${\bf x}_2^{\text{\tiny$(n_j)$}}$ matches to ${\bf x}^{\text{\tiny$(n_{i})$}}$. In addition, we should replace two points ${\bf x}_1^{\text{\tiny$(n_j)$}}$ and ${\bf x}_2^{\text{\tiny$(n_j)$}}$ by a new point $\bar{\bf x}^{\text{\tiny$(n_j)$}}$ and remember that the new point $\bar{\bf x}^{\text{\tiny$(n_j)$}}$ matches to both ${\bf x}^{\text{\tiny$(n_i)$}}$ on the image $n_i$ and ${\bf x}^{\text{\tiny$(n_{i'})$}}$ on the image $n_{i'}$. To solve this problem, we shall cluster the set $\mathcal{I}_{n_j}$ to group image points that are close together into a group. Note that all image points in $\mathcal{I}_{n_j}$ are in a straight line between ${\bf x}_k^{\text{\tiny$(n_j)$}}$ and ${\bf x}_{k'}^{\text{\tiny$(n_j)$}}$. Thus, we shall use the \emph{Jenks natural breaks optimization}~\citep{Jenks1967} to cluster all image points in $\mathcal{I}_{n_j}$. We have a mention that in one-dimensional clustering, the Jenks natural breaks optimization seems to be \emph{$k$-means clustering}~\citep{Hartigan1979}.

\begin{algorithm}[t!]
\caption{Creating Point Clouds from two given world points (\textbf{CrPC})}
\label{Alg:Point_Clouds}
{
{\bf Input:} $\big\{{\bf P}^{\star}_1, {\bf P}^{\star}_2, \ldots, {\bf P}^{\star}_N\big\}$~:~ Projection matrices,\\[-2mm]

\hspace{12mm} ${\bf X}_k$ ~and~ ${\bf X}_{k'}$\hspace{7.25mm}:~ Two given world points,\\[-2mm]

\hspace{12mm} $\theta$ ~and~ $\ell$\hspace{14.5mm}:~ Thresholds for candidates detection.\\[-2mm]

{\bf Implementation:}\\[-2mm]

{\bf 1.} $\mathcal{PC} ~\longleftarrow~ \emptyset$ : \quad a point clouds set collecting all reconstructed points,\\[-2mm]

\hspace{5mm}$\{{\bf x}_k^{\text{\tiny(n)}}\}^N_{n=1} \overset{\text{Image points}~\eqref{eq:Image_point}}{\longleftarrow} \big\{{\bf X}_k\,;\,\{{\bf P}^{\star}_n\}\big\}$\qquad\text{and}\qquad$\{{\bf x}_{k'}^{\text{\tiny(n)}}\}^N_{n=1} \overset{\text{Image points}~\eqref{eq:Image_point}}{\longleftarrow} \big\{{\bf X}_{k'}\,;\,\{{\bf P}^{\star}_n\}\big\}$\,,\\[-2mm]

\hspace{5mm}$\mathcal{I}_{1} \longleftarrow \emptyset$~,\quad $\mathcal{I}_{2} \longleftarrow \emptyset$~, ~$\ldots$~,\quad $\mathcal{I}_{N} \longleftarrow \emptyset$\,.\\[-2mm]

{\bf 2.} For all $1\leq n_i \neq n_j \leq N$: \\[-2mm]

\hspace{6mm}{\bf 2.1.} $L \longleftarrow \max\big\{\|{\bf x}_{k}^{\text{\tiny$(n_i)$}} - {\bf x}_{k'}^{\text{\tiny$(n_i)$}}\|_{\infty}\,,\,\|{\bf x}_{k}^{\text{\tiny$(n_j)$}} - {\bf x}_{k'}^{\text{\tiny$(n_j)$}}\|_{\infty}\big\}$\,.\\[-2mm]

\hspace{6mm}{\bf 2.2.} ${\bf Img}_{n_i}({\bf X}_k,{\bf X}_{k'}|L)$\quad and \quad ${\bf Img}_{n_j}({\bf X}_k,{\bf X}_{k'}|L) ~\longleftarrow~~\eqref{eq:Image_Line}$\,.\\[-2mm]

\hspace{6mm}{\bf 2.3.} If~ $\text{Corr}\big({\bf Img}_{n_i}({\bf X}_k,{\bf X}_{k'},|L)\,,\,{\bf Img}_{n_j}({\bf X}_k,{\bf X}_{k'}|L)\big)~>~ 1-\theta$ :\\[-2mm]

\hspace{16mm}for all $\ell/2 \leq l_i, l_j \leq L-1-\ell/2$:\\[-2mm]

\hspace{16mm}{\bf a.} $\psi\big({\bf x}_k^{\text{\tiny$(n_i)$}},{\bf x}_{k'}^{\text{\tiny$(n_i)$}},l_i\big)$\quad and \quad$\psi\big({\bf x}_k^{\text{\tiny$(n_j)$}},{\bf x}_{k'}^{\text{\tiny$(n_j)$}},l_j\big)~\longleftarrow~\eqref{eq:Image_Line_ell}$\\[-2mm]

\hspace{16mm}{\bf b.} if~ $\text{Corr}\big(\psi({\bf x}_k^{\text{\tiny$(n_i)$}},{\bf x}_{k'}^{\text{\tiny$(n_i)$}},l_i)\,,\,\psi({\bf x}_k^{\text{\tiny$(n_j)$}},{\bf x}_{k'}^{\text{\tiny$(n_j)$}},l_j)\big) ~>~1-\theta$ :\\

\hspace{26mm}{\bf i.}~~ ${\bf x}^{\text{\tiny$(n_i)$}} ~\longleftarrow~ \frac{1}{L-1}\big[(L-1-l_i){\bf x}_k^{\text{\tiny$(n_i)$}} + l_i{\bf x}_{k'}^{\text{\tiny$(n_i)$}}\big]$~ \quad and\quad add ${\bf x}^{\text{\tiny$(n_i)$}}$ ~to~ $\mathcal{I}_{n_i}$\\

\hspace{26mm}{\bf ii.}~ ${\bf x}^{\text{\tiny$(n_j)$}} ~\longleftarrow~ \frac{1}{L-1}\big[(L-1-l_j){\bf x}_k^{\text{\tiny$(n_j)$}} + l_j{\bf x}_{k'}^{\text{\tiny$(n_j)$}}\big]$ \quad and\quad add ${\bf x}^{\text{\tiny$(n_j)$}}$ ~to~ $\mathcal{I}_{n_j}$\\

\hspace{26mm}{\bf iii.} Remember that ``${\bf x}^{\text{\tiny$(n_i)$}}$ matches to ${\bf x}^{\text{\tiny$(n_j)$}}$''.\\[-2mm]

{\bf 3.} For each~ $n\in\{1,2,\ldots,N\}$,\quad $\big\{\bar{\bf x}_1^{\text{\tiny(n)}}, \bar{\bf x}_2^{\text{\tiny(n)}},\ldots,\bar{\bf x}_{\kappa_n}^{\text{\tiny(n)}}\big\}$ $\overset{\text{~Jenks natural breaks optimization~}}{\longleftarrow}$ $\mathcal{I}_n$\,.\\[-2mm]

{\bf 4.} $\mathcal{C}({\bf X}_k, {\bf X}_{k'},\theta,\ell) ~\longleftarrow~$~\eqref{eq:Match_Cluster} and~\eqref{eq:Candidate_Set}.\\[-2mm]

{\bf 5.} For each candidate $\big\{\bar{\bf x}_{h_1}^{\text{\tiny$(n_1)$}}\leftrightarrow\bar{\bf x}_{h_2}^{\text{\tiny$(n_2)$}}\leftrightarrow\cdots\leftrightarrow\bar{\bf x}_{h_J}^{\text{\tiny$(n_J)$}}\big\}$ ~in~ $\mathcal{C}({\bf X}_k, {\bf X}_{k'},\theta,\ell)$\\[-2mm]

\hspace{6mm}{\bf 5.1.} Use $J$ projection matrices $\{{\bf P}^{\star}_{n_1},{\bf P}^{\star}_{n_2},\ldots,{\bf P}^{\star}_{n_J}\}$ and~\eqref{eq:Point_from_Camera} to create a new world point ${\bf X}$\,,\\[-2mm]

\hspace{6mm}{\bf 5.2.} $\mathcal{E}_d({\bf X}) ~\longleftarrow~~\eqref{eq:Image_point}$ and~\eqref{eq:Metric_point}\,,\qquad $\mathcal{E}_i({\bf X}) = J$\,,\\[-2mm]

\hspace{6mm}{\bf 5.3.} If $\mathcal{E}_d({\bf X})$ and $\mathcal{E}_i({\bf X})$ satisfy the conditions of the point clouds, add ${\bf X}$ to $\mathcal{PC}$.\\[-2mm]

{\bf Output:} $\mathcal{PC}$ ~-~ Point clouds.\\[-1mm]
}
\end{algorithm}

Assuming that the Jenks natural breaks optimization clusters $\mathcal{I}_{n_j}$ into $\kappa_j$ classes $\{{\bf I}_{1}, {\bf I}_2, \ldots,{\bf I}_{\kappa_j}\}$, we shall replace all the image points ${\bf x}^{\text{\tiny$(n_j)$}}$ in ${\bf I}_h$ by their mean $\bar{\bf x}_h^{\text{\tiny$(n_j)$}}$ given by
\begin{equation}\label{eq:Imge_Cluster}
\bar{\bf x}_{h}^{\text{\tiny$(n_j)$}} = \frac{1}{|{\bf I}_{h}|}\sum_{{\bf x}^{\text{\tiny$(n_j)$}}\in{\bf I}_h}{\bf x}^{\text{\tiny$(n_j)$}}
\end{equation}
where $|{\bf I}_{h}|$ is the number of image points in ${\bf I}_{h}$. Then on the image ${\bf Img}_{n_j}$, at the straight line connecting two image points ${\bf x}_k^{\text{\tiny$(n_j)$}}$ and ${\bf x}_{k'}^{\text{\tiny$(n_j)$}}$, we only use $\kappa_j$ new image points $\bar{\bf x}^{\text{\tiny$(n_j)$}}_1, \bar{\bf x}^{\text{\tiny$(n_j)$}}_2, \ldots, \bar{\bf x}_{\kappa_j}^{\text{\tiny$(n_j)$}}$ to find candidates for geo-features as defined by Definition~\ref{Def:Candidate_Geofeature}. The candidates will be found based on the matches of elements in $J$ groups $\big\{\bar{\bf x}^{\text{\tiny$(n_j)$}}_1, \bar{\bf x}^{\text{\tiny$(n_j)$}}_2, \ldots, \bar{\bf x}_{\kappa_j}^{\text{\tiny$(n_j)$}}\big\}^J_{j=1}$ corresponding to $J$ groups of images $\{{\bf Img}_{n_1}, {\bf Img}_{n_2}, \ldots, {\bf Img}_{n_J}\}$ that
\begin{equation}\label{eq:Match_Cluster}
\begin{cases}
{\bf x}^{\text{\tiny$(n_i)$}} \quad \text{is replaced by} \quad \bar{\bf x}_{h_i}^{\text{\tiny$(n_i)$}}\\[2mm]
{\bf x}^{\text{\tiny$(n_j)$}} \quad \text{is replaced by} \quad \bar{\bf x}_{h_j}^{\text{\tiny$(n_j)$}}\\[2mm]
{\bf x}^{\text{\tiny$(n_i)$}} \quad \text{matches to} \hspace{4.5mm}\quad {\bf x}^{\text{\tiny$(n_j)$}}
\end{cases} \qquad \Rightarrow \qquad \bar{\bf x}_{h_i}^{\text{\tiny$(n_i)$}} \quad \text{matches to} \quad \bar{\bf x}_{h_j}^{\text{\tiny$(n_j)$}}\, .
\end{equation}
Concretely, we consider the following set of correspondences
\begin{equation}\label{eq:Candidate_Set}
\mathcal{C}\big({\bf X}_k,{\bf X}_{k'},\theta,\ell\big) ~\defeq~ \Bigg\{\big\{\bar{\bf x}_{h_1}^{\text{\tiny$(n_1)$}}\leftrightarrow\bar{\bf x}_{h_2}^{\text{\tiny$(n_2)$}}\leftrightarrow\cdots\leftrightarrow\bar{\bf x}_{h_J}^{\text{\tiny$(n_J)$}}\big\}~\Bigg|~\begin{matrix}\bar{\bf x}_{h_j}^{\text{\tiny$(n_j)$}}\in\big\{\bar{\bf x}^{\text{\tiny$(n_j)$}}_1, \bar{\bf x}^{\text{\tiny$(n_j)$}}_2, \ldots, \bar{\bf x}_{\kappa_j}^{\text{\tiny$(n_j)$}}\big\}~~\text{for all}~~j\,~~\\[1mm]
\bar{\bf x}_{h_i}^{\text{\tiny$(n_i)$}} \quad \text{matches to}\quad \bar{\bf x}_{h_j}^{\text{\tiny$(n_j)$}}\hspace{7.25mm}\quad\text{for all}~~i,j\,\end{matrix}\Bigg\}
\end{equation}
as candidates for the geo-features w.r.t. $\{{\bf X}_k,{\bf X}_{k'},\theta,\ell\}$. From these candidates, new points are created based on the known projection matrices ${\bf P}^{\star}_{n_1}, {\bf P}^{\star}_{n_2}, \ldots, {\bf P}^{\star}_{n_J}$ and the formula~\eqref{eq:Point_from_Camera}. Depending on applications, the optimal metric values $\mathcal{E}_d$ and $\mathcal{E}_i$ are used to decide whether or not to create these new points to the point clouds.

\newpage

Finally, Algorithm~\ref{Alg:Point_Clouds}, named the \emph{creating point clouds} (\texttt{CrPC}), presents step-by-step the process of creating point clouds based on the two given world points ${\bf X}_k, {\bf X}_{k'}$ and the parameters $\theta$ and $\ell$.

\section{Point Clouds from Correspondences via an experiment}\label{Sec:PC_from_Correspondences}

Section~\ref{Sec:Point_Clouds} proposes the algorithm named \texttt{CrCP} and the theory of \emph{reconstructing geodesic based on geo-feature} for this algorithm. New world points in the point clouds are created based on some given world points found by Algorithm \texttt{WPfC} and the correspondences. This section will describe an optimal method to create the point clouds from the correspondences such that these point clouds would be everywhere in the surface of the reconstructed object. This is the most important feature of the point clouds we hope to obtain when we discuss it at the beginning of Section~\ref{Sec:Point_Clouds}. The description is based on a real experiment in our project of the Buddha statue 3D reconstruction. In this experiment, we use 116 images of the \emph{front body} of the Buddha statue to 3D reconstruct the front body part of this statue. In the front body of the Buddha statue, we have 27 special points that help us on estimating the correspondences, reconstructing 3D points of these 27 special points, and estimating the projection matrices for the 116 images as we done and figured out in Figure~\ref{Fig:Sec2}.

This section begins with Figure~\ref{Fig:Images_Keypoints} that gives 38 from 116 images we use for this experiment. In these 38 images, we indicate the key points from the 27 special points by the red dots. These key-points are used to determine the correspondences. We determine these key-points based on our work in~\citet{le2022multi}. The average special points captured by each image is 9.4. There are only two from 116 images capture 17 (maximum) and three images capture 6 (minimum) from 27 special points. Two images capturing 17 special points are two big images in Figure~\ref{Fig:Images_Keypoints}.

\begin{figure*}[t]
  \centering
  \includegraphics[width=0.95\linewidth]{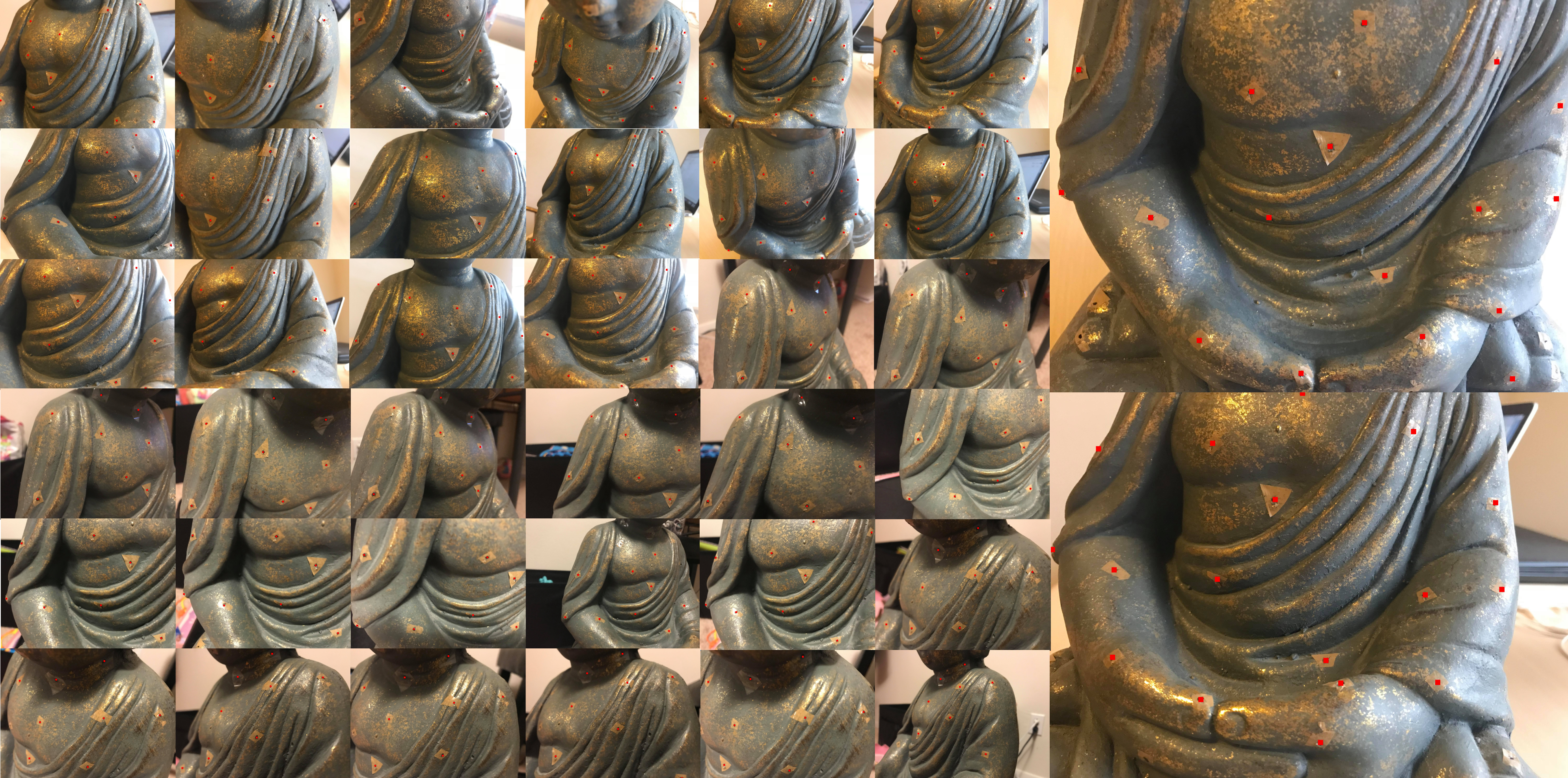}
  \caption{Images with key points (\emph{Red dots}) for the front body of Buddha statue experiment. Thirty eight (38) images from one hundred and sixteen (116) images of the experiment.  Two big images capture seventeen (17) from twenty seven (27) special points in the surface of the Buddha statue's front body.}
  \label{Fig:Images_Keypoints}
\end{figure*}

We denote 27 special points by $\{{\bf X}_1, {\bf X}_2, \ldots, {\bf X}_{27}\}$ and use Algorithm \texttt{WPfC} together with correspondences from the key points to estimate them. To fix the ambiguity of the 3D reconstruction under any projective transformation as presented by Proposition~\ref{Prop:Ambiguity}, we need to know exact relative positions of at least four points from these 27 special points. Indeed, in our 3D reconstruction, we shall determine the relative positions of four points $\{{\bf X}_1, {\bf X}_2, {\bf X}_3, {\bf X}_4\}$ by manually measuring six Euclidean distances $\|{\bf X}_i - {\bf X}_j\|_2,~(1\leq i < j \leq 4)$ and then using the \emph{multidimensional scaling} (MDS)~\citep{Schonemann1970} to acquire their relative positions. If we need to increase an accuracy, we can determine the relative positions of five special points by measuring ten Euclidean distances or the relative positions of six special points by measuring fifteen Euclidean distances. Next without loss of generality, we assume that the 3D positions of ${\bf X}_1, {\bf X}_2, {\bf X}_3$ and ${\bf X}_4$ are known. We consider all $N$ images $\{{\bf Img}_{i_1}, {\bf Img}_{i_2},\ldots,{\bf Img}_{i_N}\}$ such that there are $N$-view correspondences of $M$ world points and $\{{\bf X}_1, {\bf X}_2, {\bf X}_3, {\bf X}_4\}$ are ones of these $M$ world points. Using Algorithm \texttt{WPfC} with these $N$-view correspondences of $M$ world points, we reconstruct all other world points to increase the number of known world points from four to $M$. From the $M$ known world points, we continue this process until we have a total of 27 estimated world points ${\bf X}^{\star}_1, {\bf X}^{\star}_2,\ldots,{\bf X}^{\star}_{27}$. Of course, we also derive all projection matrices ${\bf P}^{\star}_1, {\bf P}^{\star}_2,\ldots,{\bf P}^{\star}_{116}$ of 116 images. The 27 estimated world points and 116 camera positions corresponding to 116 projection matrices are figured out in Figure~\ref{Fig:Sec2}.

\begin{figure*}[t]
  \centering
  \includegraphics[width=0.95\linewidth]{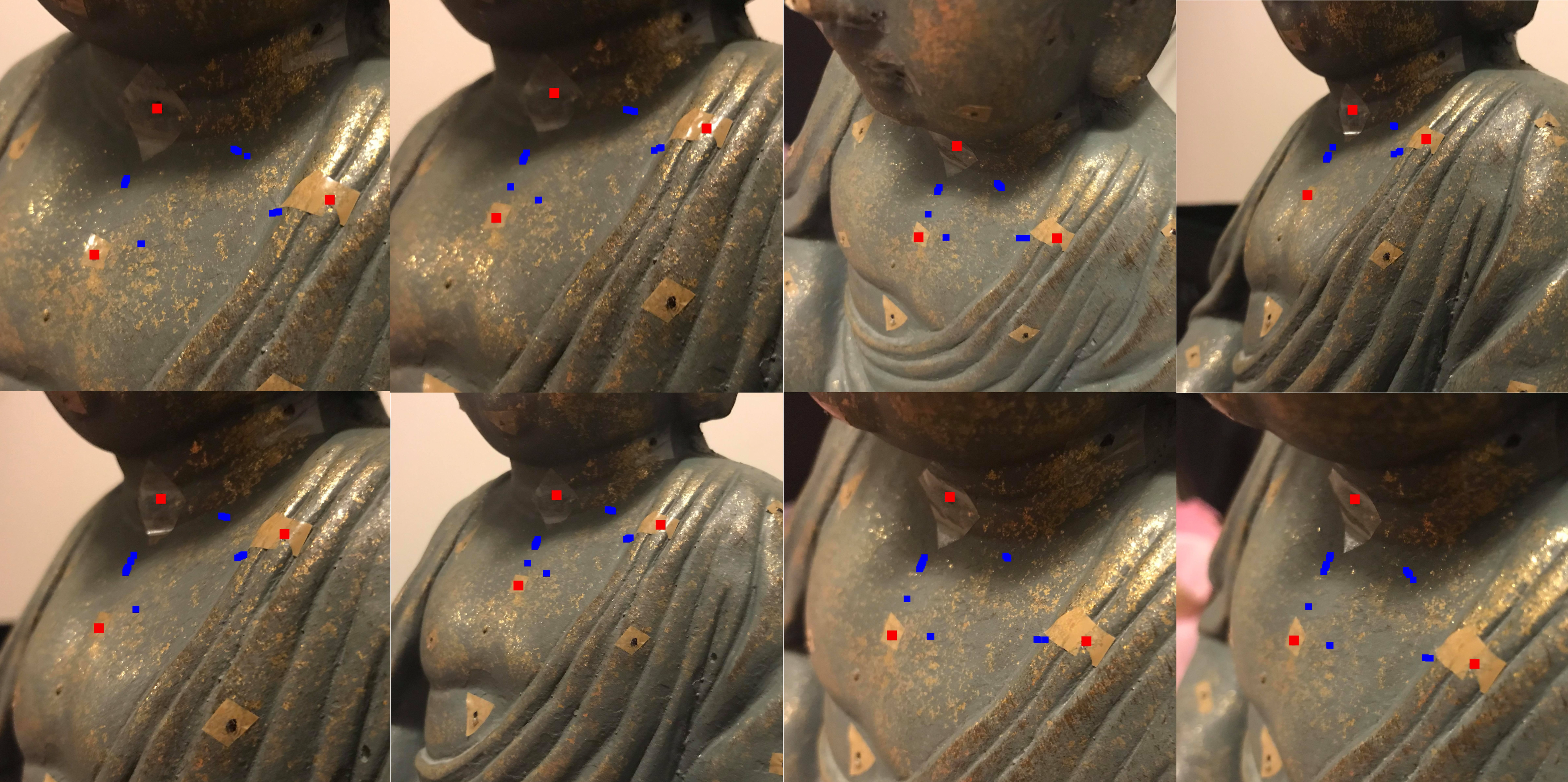}
  \caption{Zoom out of eight example images capturing three world points ${\bf X}^{\star}_4, {\bf X}^{\star}_5, {\bf X}^{\star}_{26}$. Some of geo-features w.r.t. these three world points and the parameters $\theta = 0.3,~\ell = L/10$ are the \emph{blue dots} in each images.}
  \label{Fig:Geo_features}
\end{figure*}

\begin{figure*}[t]
  \centering
  \includegraphics[width=0.95\linewidth]{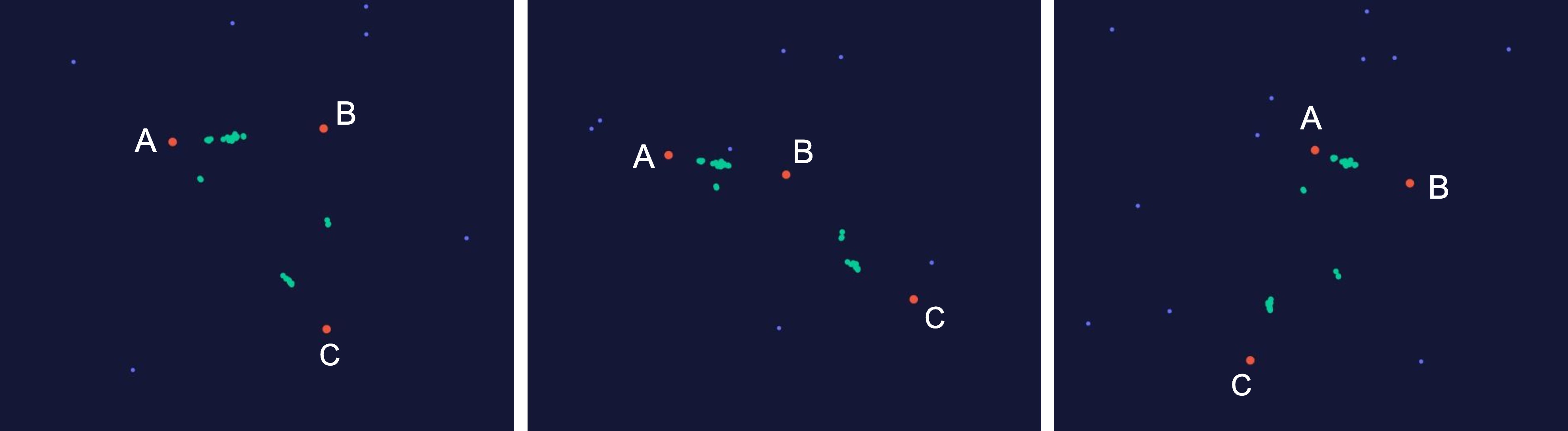}
  \caption{Three different views in three-dimensional space of the old world points ${\bf X}^{\star}_4, {\bf X}^{\star}_5, {\bf X}^{\star}_{26}$ (\emph{red dots}) and 37 new world points (\emph{green dots}) found based on these three old world points. The \emph{white A} indicates the world point ${\bf X}^{\star}_4$, the \emph{white B} indicates the world point ${\bf X}^{\star}_5$, and the \emph{white C} indicates the world point ${\bf X}^{\star}_{26}$.}
  \label{Fig:Wpoints_Est}
\end{figure*}

To generate the point clouds, we take two arbitrary world points ${\bf X}^{\star}_k$ and ${\bf X}^{\star}_{k'}$ from the 27 estimated world points and find all candidates for the geo-features w.r.t. $\{{\bf X}^{\star}_k\,,\,{\bf X}^{\star}_{k'}\,,\,\theta = 0.3\,,\, \ell = L/10\}$. For example, with a couple $\{{\bf X}^{\star}_4\,,\,{\bf X}^{\star}_5\}$, there are 48 images capturing this couple and from them we find a total 12 geo-features. Note that there are some of 48 images capture 10 but some others capture only 3 from these 12 geo-features. Let us consider other examples with two couples $\{{\bf X}^{\star}_4\,,\,{\bf X}^{\star}_{26}\}$ and $\{{\bf X}^{\star}_5\,,\,{\bf X}^{\star}_{26}\}$. There are 42 images capturing the first couple and 22 images capturing the second one. We find 26 new geo-features w.r.t. the first couple but only 8 new geo-features w.r.t. the second couple. Eight example images that capture all three world points ${\bf X}^{\star}_4, {\bf X}^{\star}_5, {\bf X}^{\star}_{26}$ together with some new geo-features (\emph{blue dots}) found from these three world points are in Figure~\ref{Fig:Geo_features}. Normally, from these (12 + 26 + 8 =) 46 geo-features w.r.t. three world points $\{{\bf X}^{\star}_4, {\bf X}^{\star}_5, {\bf X}^{\star}_{26}\}$, we shall have 46 new world points. However, since we use the constraints that $\mathcal{E}_d(\text{new world point}) \leq 40$ and $\mathcal{E}_i(\text{new world point}) \geq 7$ to recognize good world points, we only have 37 new world points from these 46 geo-features. The new world points (\emph{green dots}), together with three old world points ${\bf X}^{\star}_4, {\bf X}^{\star}_5, {\bf X}^{\star}_{26}$ (\emph{red dots}) are plotted by Figure~\ref{Fig:Wpoints_Est}. 

\begin{figure*}[t]
  \centering
  \includegraphics[width=0.95\linewidth]{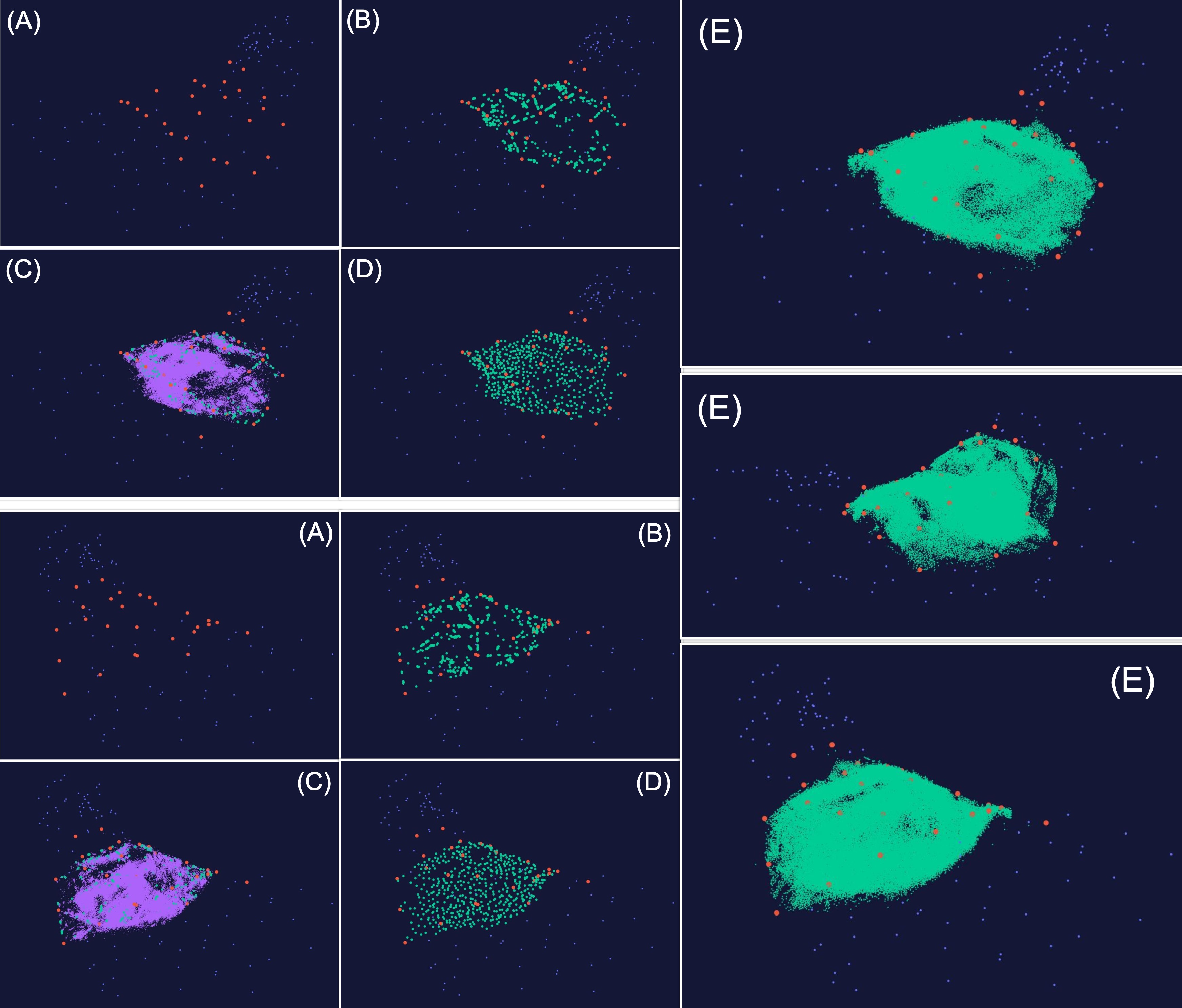}

\vspace{-0.05in}

  \caption{A process of creating point clouds from given world points: \textbf{(A)} \emph{Red dots} are some given world points (called \emph{level 0}), \textbf{(B)} \emph{Green dots} are point clouds (called \emph{level 1}) creating from level 0, \textbf{(C)} \emph{Purple dots} are point clouds (called \emph{level 2}) generating from level 1, and \textbf{(E)} \emph{Green dots} are point clouds (called \emph{level 3}) creating from the point clouds in \textbf{(D)} that refined by level 2, \textbf{(F)} and so on. The process is simulated by an experiment of reconstructing the front body of the Buddha statue.}
  \label{Fig:Creating_PClouds}
\end{figure*}

The above paragraph explains a process of creating new world points from old world points. Figure~\ref{Fig:Creating_PClouds} gives a sufficient description of a process of generating point clouds. In this figure, sub-figures (A) are the 27 estimated world points $\{{\bf X}^{\star}_1, {\bf X}^{\star}_2, \ldots, {\bf X}^{\star}_{27}\}$ (\emph{red circles}) in two different views. From these given 27 world points, we generate 785 new world points as we do with three world points $\{{\bf X}^{\star}_4,{\bf X}^{\star}_5, {\bf X}^{\star}_{26}\}$ in the previous paragraph. The old 27 world points (\emph{red circles}) and 785 new world points (\emph{green circles}) are plotted in the sub-figures (B). We continue this process with (27 + 785 =) 812 given world points to create more than 400,000 new world points as \emph{purple circles} in the sub-figures (C). Although a huge number of world points are generated, there are some areas in the surface that cannot reconstruct. In another word, our point clouds are not everywhere of the object's surface. Indeed, we continue our process of creating the point clouds by collecting the best world points in the current point clouds to generate new world points. Specifically, we choose the best world points such that they have minimal values of $\mathcal{E}_d$ and no two points in them are closer than 0.5 centimeter. There are 522 world points chosen based on this algorithm. These 522 world points (\emph{green circles}) and 27 initialization (\emph{red circles}) are given in the sub-figures (D). Comparing with 785 world points in the sub-figures (B), the new 522 world points cover the surface better. With these 522 world points, Algorithm \texttt{CrPC}, together with the constraints $\mathcal{E}_d(\cdot) \leq 40$ and $\mathcal{E}_i(\cdot)\geq 7$, generates more than 2 millions new world points. From the new world points, we choose around the 600,000 best ones such that no two points in them are closer than 0.05 centimeter. The 600,000 best ones are plotted in the sub-figures (E) as the \emph{green dots}. To see how the final point clouds cover the area between the hands and the front body, we add another view of these point clouds (the middle sub-figure (E)).

There is one interesting point we can see in Figure~\ref{Fig:Creating_PClouds} to evaluate how is efficient of Algorithm \texttt{CrPC} in reconstructing the front body of Buddha statue. From the initialization with 27 world points, there are two world points in the neck of the statue that only have twelve images capture them. Indeed, first we guess Algorithm \texttt{CrPC} cannot help us to reconstruct the area in the neck. It is still correct when we see results in the sub-figures (B), (C) and (D). However, the results in the sub-figures (E) show that there are some new world points in the neck area that means Algorithm \texttt{CrPC} can reconstruct this area. This result also yields that the 522 world points in the sub-figures (D) is really better than the 783 world points in the sub-figures (B). Therefore, we believe that from the point clouds in the sub-figures (E), we find the best ones as we do with 522 world points in the sub-figures (D) from the point clouds in the sub-figures (C), then the new world points may be better than the 522 world points. Thus, we can have better point clouds than the point clouds given in the sub-figures (E).

\section{Evaluation: The Buddha statue 3D reconstruction}\label{Sec:Evaluation}

\subsection{The Buddha statue experiments}\label{Subsec:Experiments}

\begin{figure*}[t]
  \centering
  \includegraphics[width=0.85\linewidth]{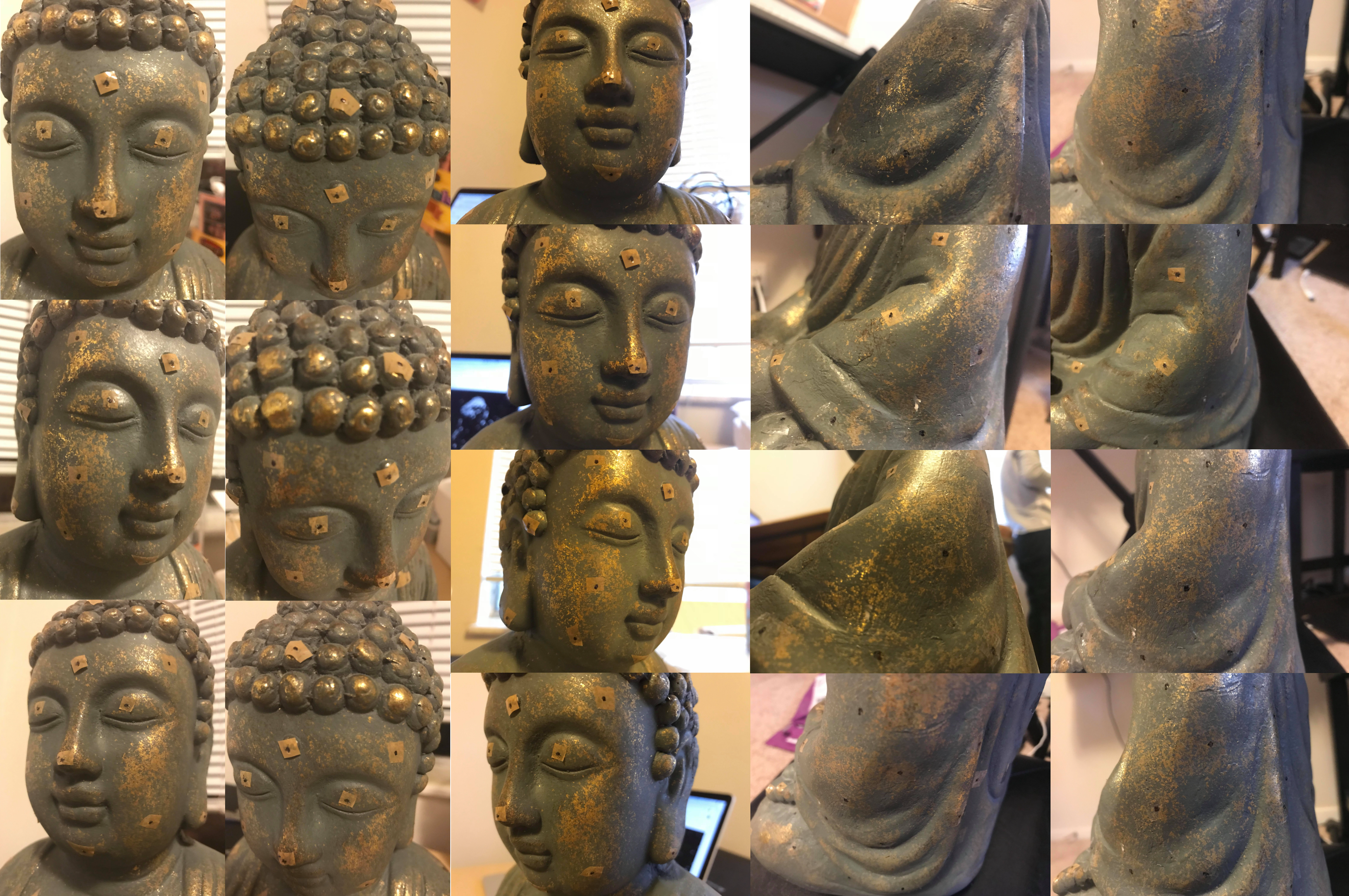}

\vspace{-0.05in}

  \caption{Thirteen example images from 160 images of the \emph{face experiment} and eight example images from 68 images of the \emph{right hand experiment}.}
  \label{Fig:Experiments}\vspace{-0.05in}
\end{figure*}

Our research on creating 3D point clouds from a few accurate 2D correspondences is validated by the 3D Buddha statue reconstruction project as we describe in Section~\ref{Sec:Introduction} by Figure~\ref{Fig:Main}. We study to 3D reconstruct the Buddha statue using 402 images of this statue. First, we mark on the surface of the Buddha statue 112 special small points. These special points are around the surface. For instance, there are 27 points on the front body of Buddha statue as we see in the \emph{front body experiment} described in Section~\ref{Sec:PC_from_Correspondences}. Second, the 402 images of the Buddha statue were taken such that each image captures at least six and at most 20 special points. From these 402 images, we divide our project into twelve experiments corresponding to 3D reconstructions of twelve areas of the Buddha statue. The first experiment is the \emph{front body experiment} with 116 images capturing 27 special world points. A task of this experiment is reconstructing the front body surface of the Buddha statue as sub-figures (E) of Figure~\ref{Fig:Creating_PClouds} show. The second experiment, named the \emph{face experiment}, has 119 images that capture 15 special world points, and other experiments like the \emph{head experiment}, the \emph{back experiment}, the \emph{left hand experiment}, the \emph{right hand experiment}, etc. Figure~\ref{Fig:Experiments} gives 13 example images from 119 images of the face experiment and 8 example images from 68 images of the right hand experiment.

\begin{figure*}[t]
  \centering
  \includegraphics[width=0.95\linewidth]{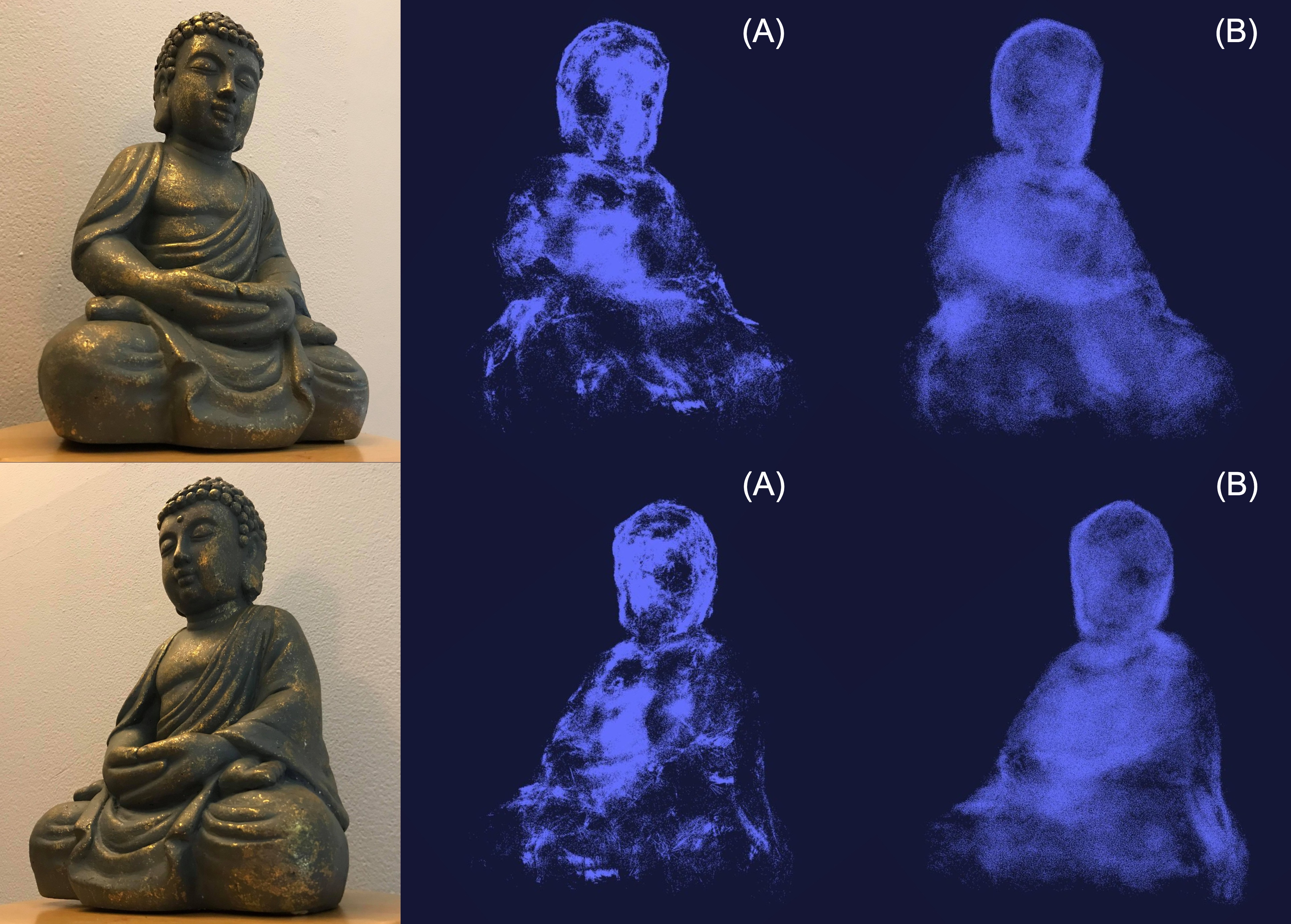}
  \caption{The 3D Buddha statue reconstruction along with two similar view images: (A) The refinement point clouds corresponding to $\epsilon = 7$ pixels. (B) The refinement point clouds corresponding to distance $\delta = 0.05$ centimeter. \emph{Blue dots} are points in the point clouds.}
  \label{Fig:Total_PClouds_1}
\end{figure*}

Thanks to the special points on the surface of the statue and also on the images, we derive accurate correspondences by using the works from~\citet{Lowe1999,Bay2006,le2022multi}. Next we use Algorithms \texttt{WPfC} and \texttt{CrPC} to estimate the world points, the projection matrices and create the point clouds. The 3D reconstruction of the Buddha statue are the point clouds from the twelve experiments. From the huge point clouds that are everywhere in the surface of the Buddha statue, we can refine them to get better a 3D Buddha statue reconstruction by the following process: Given the huge original point clouds $\mathcal{PC}$ and a threshold $\delta > 0$, a refined point clouds $\mathcal{PC}_{\delta}$ is called a \emph{refinement of $\mathcal{PC}$ corresponding to distance $\delta$} if
\begin{enumerate}
\item $\mathcal{PC}_{\delta}$ is a subset of $\mathcal{PC}$, i.e., $\mathcal{PC}_{\delta}\subset\mathcal{PC}$,
\item for all refined point ${\bf X}_{\delta}\in\mathcal{PC}_{\delta}$, we always have $\mathcal{E}_d({\bf X}_{\delta}) \leq \mathcal{E}_d({\bf X}),~\forall\, {\bf X}\in\mathcal{PC}$ and $\|{\bf X} - {\bf X}_\delta\|_2\leq \delta$, \item for all ${\bf X}_{\delta}, {\bf X}'_{\delta}\in\mathcal{PC}_{\delta},~\|{\bf X}_{\delta} - {\bf X}'_{\delta}\|_2\geq \delta$.
\end{enumerate}
Examples of the refinement of point clouds corresponding to $\delta$ are in Figure~\ref{Fig:Creating_PClouds} that the point clouds in sub-figures (D) is the refinement of the point clouds in sub-figures (C) corresponding to $\delta = 0.5$ centimeter, and the point clouds in sub-figures (E) is the refinement of the final point clouds we find from the front body experiment corresponding $\delta = 0.05$ centimeter. Note that the refinement of the point clouds corresponding to distance $\delta$ maintains the \emph{everywhere} property of the original point clouds. However, we do not guarantee all the points in the refined point clouds $\mathcal{PC}_{\delta}$ are good. If some applications on 3D reconstruction require all good points in the point clouds, we can refine the point clouds $\mathcal{PC}$ by $\mathcal{PC}_{\epsilon}$, the \emph{refinement of $\mathcal{PC}$ corresponding to pixel $\epsilon$} that is defined as follows. Given a small positive real number $\epsilon$, the refined point clouds $\mathcal{PC}_{\epsilon}$ is a subset of $\mathcal{PC}$ such that for all ${\bf X}_{\epsilon}\in\mathcal{PC}_{\epsilon}$, we always get $\mathcal{E}_d({\bf X}_{\epsilon}) \leq \epsilon$. Figure~\ref{Fig:Total_PClouds_1} gives examples of the refinement point clouds corresponding to distance $\delta = 0.1$ centimeter and the refinement point clouds corresponding to $\epsilon = 5$ pixels of our point clouds for the 3D Buddha statue reconstruction.


\subsection{Evaluations}\label{Subsec:Evaluations}

There are two obvious steps in our research. The first is the estimate of the world points and projection matrices from the correspondences, and the second is the estimate of the point clouds from the estimated world points and estimated projection matrices obtained in the first step. Thus, the evaluation of our work will be based on the estimates of the world points, the projection matrices and the point clouds. Note that the estimate in the second step depends on the estimate in the first step. In addition, since the point clouds are created based on the estimate of the geo-features, the second step evaluation can be studied from the estimate of the geo-features.

In general, a correct evaluation of any estimator will be based on differences between ground-truths and their estimations by this estimator. However, in the 3D reconstruction study, it is difficult to acquire ground-truths for world points, projection matrices and point clouds. For example, to have the ground-truths for 27 world points in the front body experiment of the Buddha statue as we discuss in Section~\ref{Sec:PC_from_Correspondences}, we need to measure carefully the Euclidean distances of at least one hundred pairs of 3D points from these 27 world points, and the errors on these manual measurements would be less than 0.1 centimeter. It is a hard work. When the number of world points is the hundreds, it is not possible to obtain their good ground-truths by manually measuring the Euclidean distances. There are similar situations for acquiring the ground-truths of projection matrices, point clouds and geo-features. Since our research faces on the hundreds world points, the hundreds projection matrices, the millions point clouds and geo-features, we do not have the ground-truths for the world points, projection matrices, point clouds and geo-features. Therefore, we replace the traditional error by the \emph{relative error} to valuate our research. An relative error for the output of Algorithm \texttt{WPfC} is the objective value of the optimization~\eqref{eq:Opt_WPfC}, and an relative error for the new world point from Algorithm \texttt{CrPC} is computed based on two metrics $\mathcal{E}_d$ and $\mathcal{E}_i$ of this world point. Concretely, assuming that a group $\{{\bf X}^{\star}_m,{\bf P}^{\star}_n\}$ of $M$ estimated world points and $N$ estimated projection matrices is an output of Algorithm \texttt{WPfC} with respect to an input the correspondences $\{\hat{\bf x}^{\text{\tiny(1)}}_m\leftrightarrow\hat{\bf x}^{\text{\tiny(2)}}_m\leftrightarrow\cdots\leftrightarrow\hat{\bf x}^{\text{\tiny(N)}}_m\}^M_{m=1}$, then its relative error is given by
\begin{equation}\label{eq:Relative_Eval_WP_PM}
\text{Error}_{\text{Rel}}\big(\{{\bf X}^{\star}_m,{\bf P}^{\star}_n\}\big)~=~\frac{1}{MN}\sum^M_{m=1}\sum^N_{n=1}\left\|\hat{\bf x}^{\text{\tiny(n)}}_m \,-\,\left[\frac{{\bf p}^{\star}_{1,n}\big[{\bf X}^{\star T}_m\,,\,1\big]^T}{{\bf p}^{\star}_{3,n}\big[{\bf X}^{\star T}_m\,,\,1\big]^T}~,~\frac{{\bf p}^{\star}_{2,n}\big[{\bf X}^{\star T}_m\,,\,1\big]^T}{{\bf p}^{\star}_{3,n}\big[{\bf X}^{\star T}_m\,,\,1\big]^T}\right]^T\right\|_2
\end{equation}
the mean differences between the correspondences used to estimated $\{{\bf X}^{\star}_m,{\bf P}^{\star}_n\}$ and other correspondences computed by $\{{\bf X}^{\star}_m,{\bf P}^{\star}_n\}$. Given a new world point ${\bf X}$ in the point clouds that is acquired by Algorithm \texttt{CrPC} corresponding to a geo-feature $\{\bar{\bf x}^{\text{\tiny$(n_1)$}}_{h_1}\leftrightarrow\bar{\bf x}^{\text{\tiny$(n_2)$}}_{h_2}\leftrightarrow\cdots\leftrightarrow\bar{\bf x}^{\text{\tiny$(n_J)$}}_{h_J}\}$, its relative error is given by
\begin{equation}\label{eq:Relative_Eval_PClouds}
\begin{split}
\text{Error}_{\text{Rel}}\big({\bf X}\big)~&=~\bigg[\frac{\mathcal{E}_d({\bf X})}{\sqrt{H^2+W^2}}\bigg]^{\mathcal{E}_i({\bf X})}\\
&=~ \left[\frac{1}{J\sqrt{H^2+W^2}}\sum^J_{j=1}\left\|\bar{\bf x}^{\text{\tiny$(n_j)$}}_{h_j} - \left[\frac{{\bf p}^{\star}_{1,n_j}\big[{\bf X}^T\,,\,1\big]^T}{{\bf p}^{\star}_{3,n_j}\big[{\bf X}^T\,,\,1\big]^T}~,~\frac{{\bf p}^{\star}_{2,n_j}\big[{\bf X}^T\,,\,1\big]^T}{{\bf p}^{\star}_{3,n_j}\big[{\bf X}^T\,,\,1\big]^T}\right]^T\right\|_2\right]^{J}\, ,
\end{split}
\end{equation}
where $H,W$ are two side of images. The value $\sqrt{H^2+W^2}$ is used in the relative error of ${\bf X}$ to ensure that the error $\mathcal{E}_d({\bf X})/\sqrt{H^2+W^2}$ is always less than one. Then, the relative error of ${\bf X}$ will be smaller as $\mathcal{E}_i({\bf X})$ is larger. Thus, two points have the same $\mathcal{E}_d$ errors, the one with larger $\mathcal{E}_i$ value will be better.

\begin{figure*}[t]
  \centering
  \includegraphics[width=0.95\linewidth]{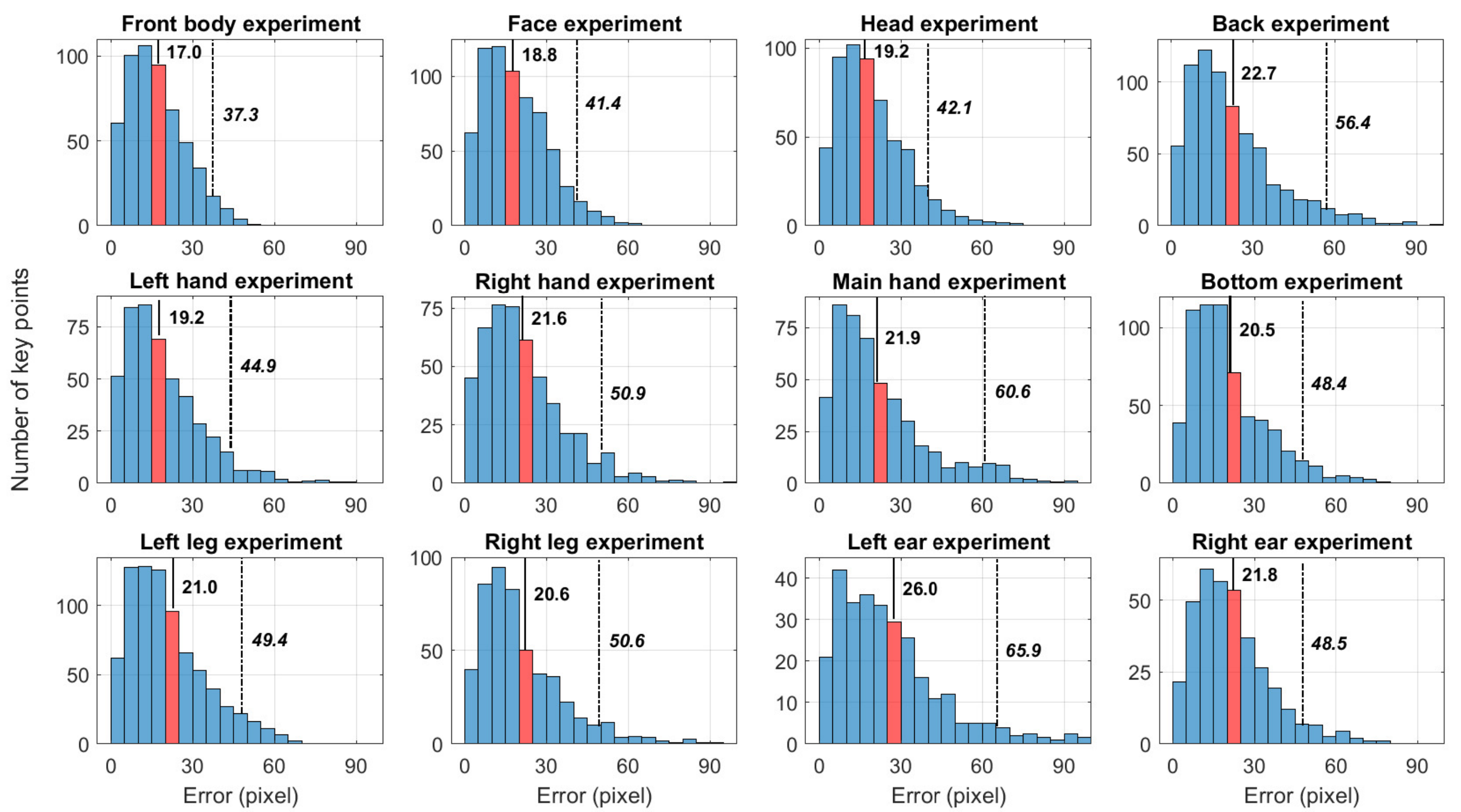}
  \caption{Twelve histograms of \emph{differences between before \& after key points} from twelve experiments in our 3D Buddha statue reconstruction project. The \emph{red bar} are used to indicate the \emph{mean} values. Two real numbers in each sub-figure are the mean and 95\%-percentile values (mean value $<$ 95\%-percentile value).}
  \label{Fig:Hist_Experiments}
\end{figure*}

Regarding the relative errors, we hope that the smaller this relative error, the smaller its ground-truth-based error (denoted by $\text{Error}_{\text{GT}}$). More precisely, we expect to find a function $f$ such that
\begin{equation}\label{eq:Func_Error_WP}
\begin{split}
&\text{Error}_{\text{GT}}\big(\{{\bf X}^{\star}_m,{\bf P}^{\star}_n\}\big) ~\leq~ f\big[\text{Error}_{\text{Rel}}(\{{\bf X}^{\star}_m,{\bf P}^{\star}_n\})\big]\\[2mm] &f\big[\text{Error}_{\text{Rel}}(\{{\bf X}^{\star}_m,{\bf P}^{\star}_n\})\big] ~\longrightarrow~ 0 \quad \text{as} \quad \text{Error}_{\text{Rel}}(\{{\bf X}^{\star}_m,{\bf P}^{\star}_n\})~\longrightarrow~ 0\,,
\end{split}
\end{equation}
for all estimated groups $\{{\bf X}^{\star}_m,{\bf P}^{\star}_n\}$, and another function $g$ such that
\begin{equation}\label{eq:Func_Error_PC}
\begin{split}
&\text{Error}_{\text{GT}}({\bf X}) ~\leq~ g\big[\text{Error}_{\text{Rel}}({\bf X})\big]\\[2mm] &g\big[\text{Error}_{\text{Rel}}({\bf X})\big] ~\longrightarrow~ 0 \quad \text{as} \quad \text{Error}_{\text{Rel}}({\bf X}) ~\longrightarrow~0\,,
\end{split}
\end{equation}
for all the point ${\bf X}$ in the point clouds. When $f$ and $g$ are derived, the study of 3D reconstruction can be evaluated without the ground-truths for world points, projection matrices and point clouds. Unfortunately, to the best of our knowledge, these two functions have yet to be found. Finding these two functions will be our next work.

\begin{figure*}[t!]
  \centering
  \includegraphics[width=0.95\linewidth]{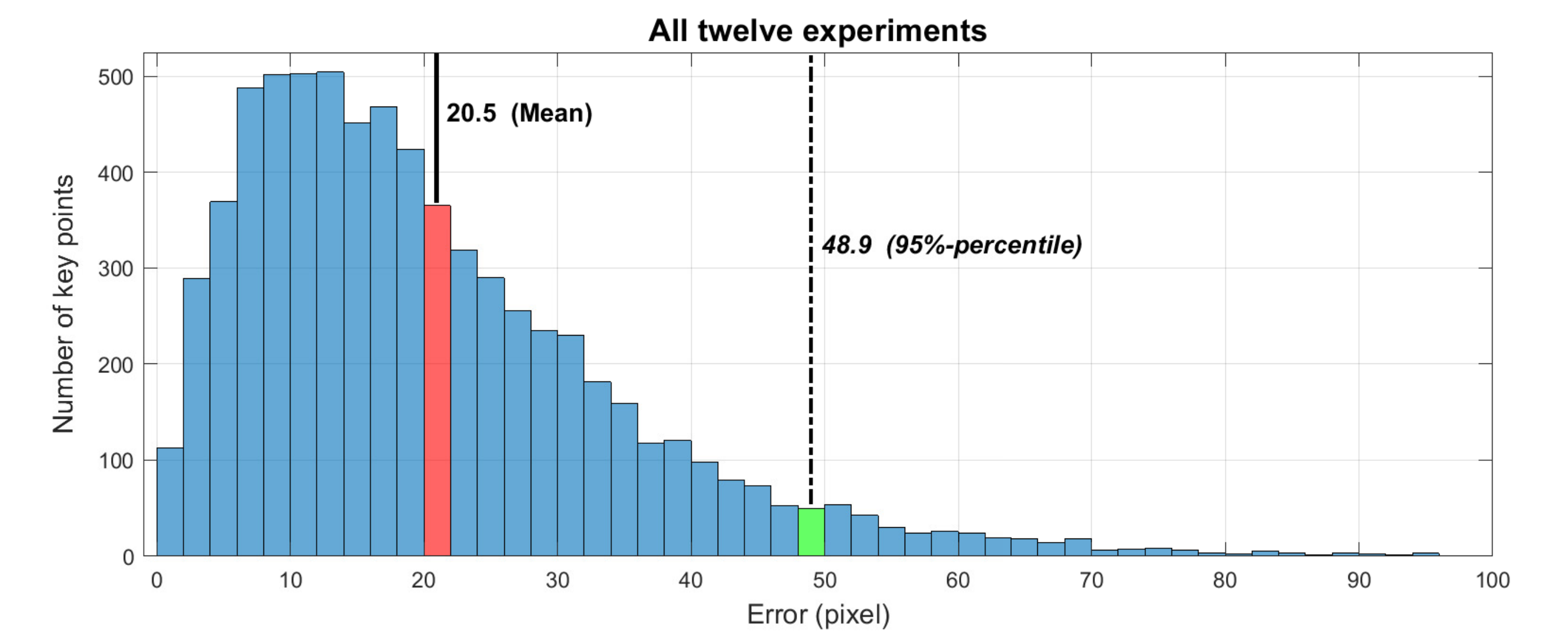}

\vspace{-0.05in}

  \caption{Histogram of \emph{differences between before \& after key points} from all key points in our 3D Buddha statue reconstruction project. The \emph{red bar} and \emph{green bar}, together with two real numbers, are used to indicate the \emph{mean} and 95\%\emph{-percentile} values.}
  \label{Fig:Hist_Total}
\end{figure*}
\begin{figure*}[t!]
  \centering
  \includegraphics[width=0.95\linewidth]{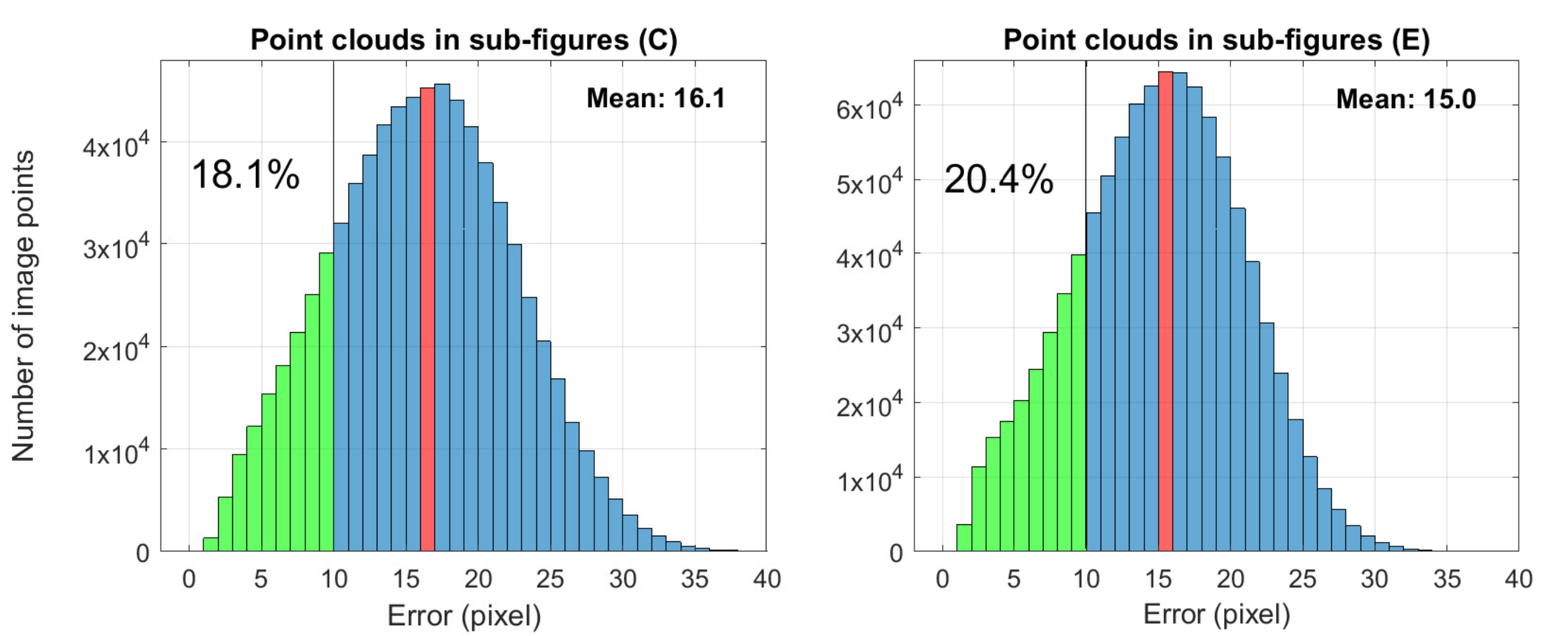}

 \vspace{-0.05in}

  \caption{Histograms of \emph{differences between before \& after image points} from the point clouds simulated by sub-figures (C) and the others simulated by sub-figures (E) in the front body experiment in Figure~\ref{Fig:Creating_PClouds}.}
  \label{Fig:Hist_Frontbody}
\end{figure*}

\begin{figure*}[t!]
  \centering
  \includegraphics[width=0.83\linewidth]{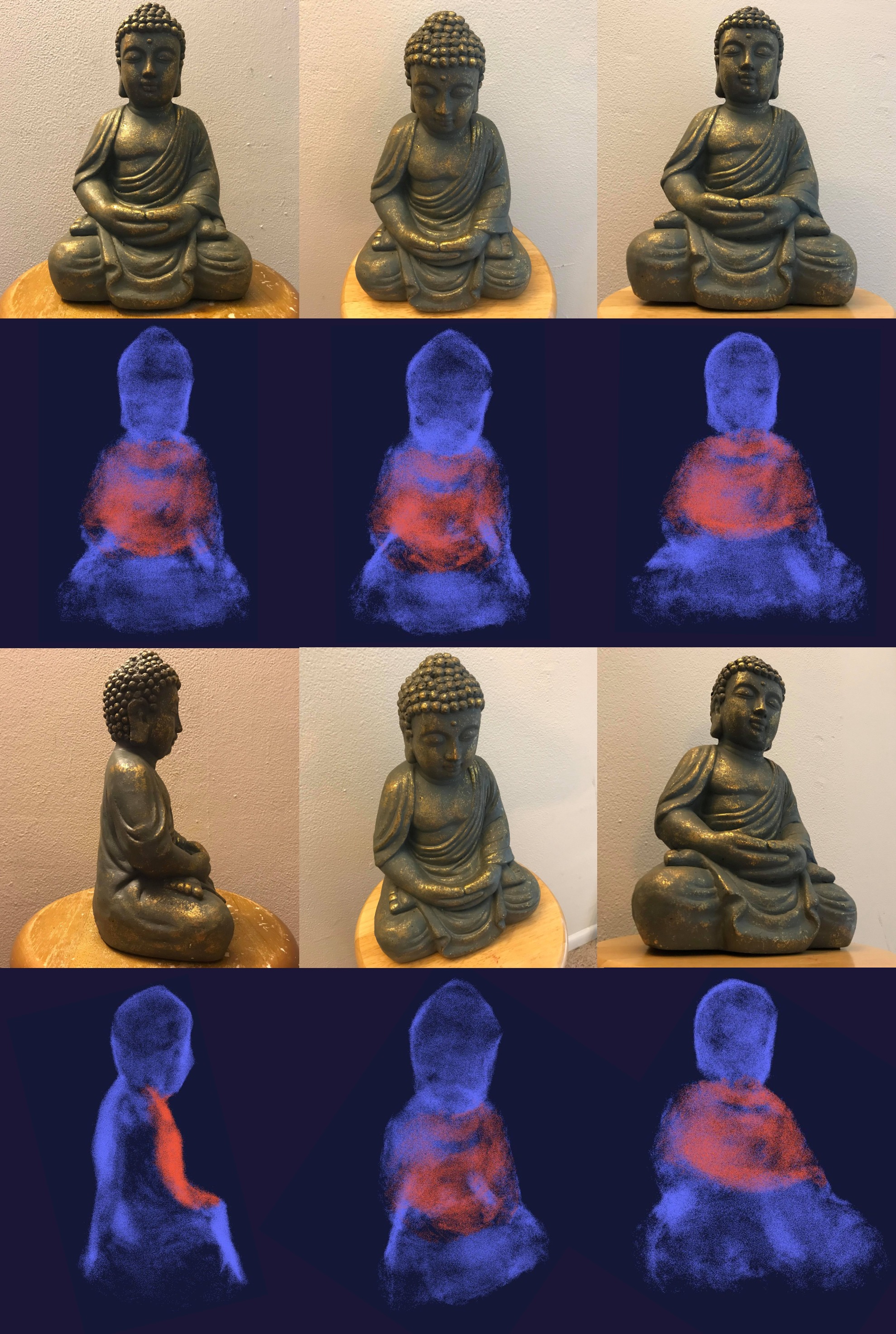}
  \caption{The 3D Buddha statue reconstruction along with six similar view images. \emph{Blue dots} are the refinement point clouds corresponding to distance $\delta = 0.05$ centimeter. \emph{Red dots} are these refinement point clouds but only from the \emph{front body experiment}.}
  \label{Fig:Total_PClouds_2}
\end{figure*}

If each Euclidean distance in~\eqref{eq:Relative_Eval_WP_PM} is called by a \emph{difference between before \& after key points} corresponding to the world point ${\bf X}^{\star}_m$ and the projection matrix ${\bf P}^{\star}_n$, the small or large of $\text{Error}_{\text{Rel}}(\{{\bf X}^{\star}_m,{\bf P}^{\star}_n\})$ can be seen from all the differences between before \& after key points. Figure~\ref{Fig:Hist_Experiments} presents twelve histograms of the differences between before \& after key points from twelve experiments. For the (1,1)th sub-figure of Figure~\ref{Fig:Hist_Experiments} corresponding to the front body experiment, we can conclude that most $\{{\bf X}^{\star}_m,{\bf P}^{\star}_n\}$ derived by this experiment will have relative error close to 17.0 pixels and less than 37.3 pixels. Similarly, relative errors of most $\{{\bf X}^{\star}_m,{\bf P}^{\star}_n\}$ from the face experiment will be close to 18.8 pixels and less than 41.4 pixels. Figure~\ref{Fig:Hist_Total} plots the similar histogram for all the key points in our project. This histogram claims that the average of differences between before \& after key points is 20.5 pixels, and 95\% of them are less than 48.9 pixels. Other than that, there's no difference greater than 100 pixels. Our images have the size $3024 \times 4032$ ($\text{pixel}^2$). Indeed, a difference of 20-pixels is very small and we hardly notice this difference with the naked eyes. This achievement validates not only a successful of the proposed algorithm \texttt{WPfC} on estimating the world points and projection matrices based on their correspondences, but also a successful of works on~\citet{Lowe1999,Bay2006,le2022multi} on detecting the key-points and the correspondences.

\begin{figure*}[t!]
  \centering
  \includegraphics[width=0.95\linewidth]{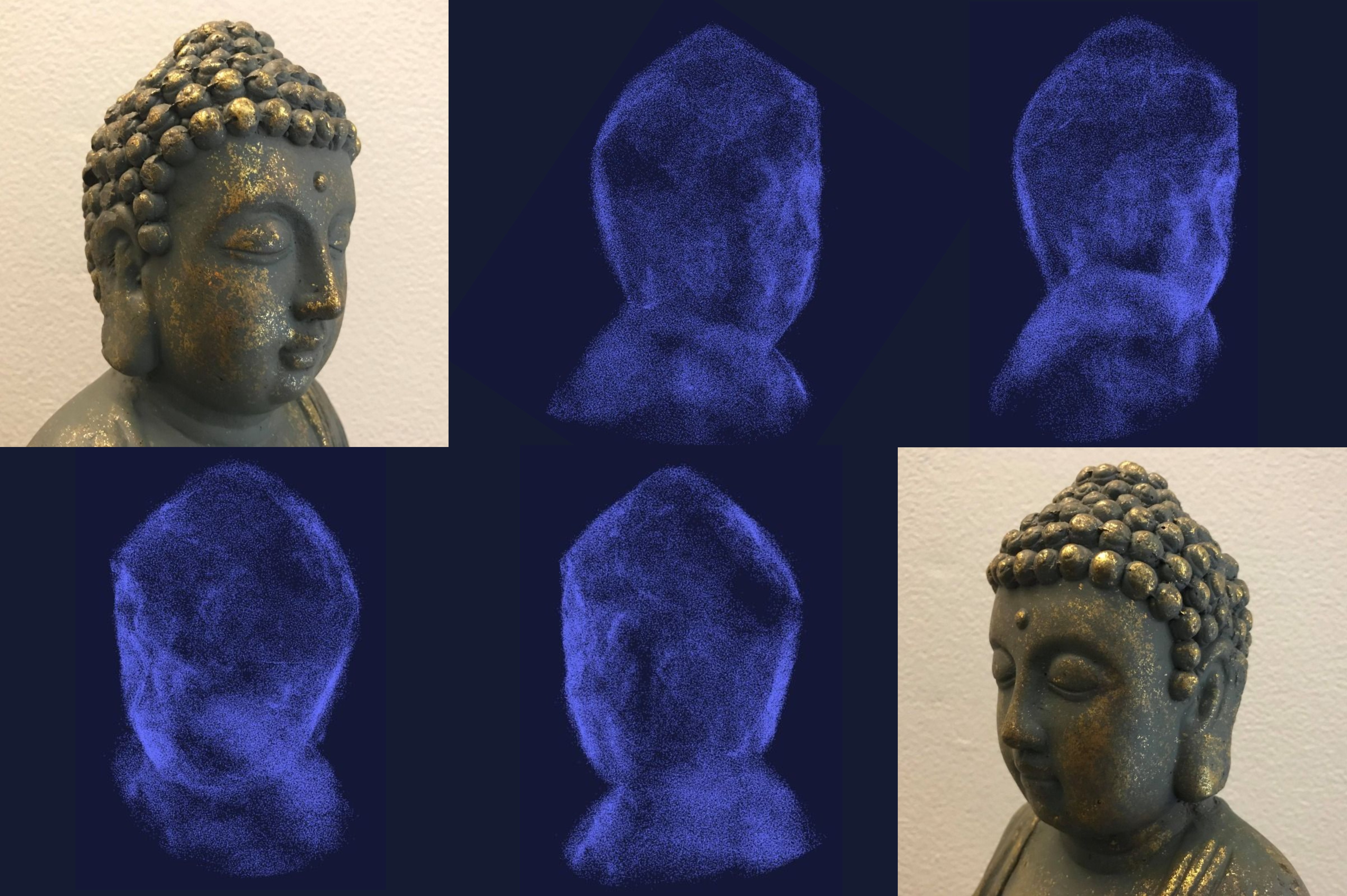}
  \caption{The \emph{face} of Buddha statue with four different views of the point clouds. The point clouds are a combine of the refinement corresponding to distance $\delta = 0.05$ centimeter and the refinement corresponding to $\epsilon = 5$ pixels.}
  \label{Fig:Face_PClouds}
\end{figure*}

We use a histogram of all Euclidean distances in~\eqref{eq:Relative_Eval_PClouds} as we do with~\eqref{eq:Relative_Eval_WP_PM} to evaluate both the relative errors $\text{Error}_{\text{Rel}}({\bf X})$ and Algorithm \texttt{CrPC}. Considering the front body experiment, Figure~\ref{Fig:Hist_Frontbody} presents two histograms of these Euclidean distances corresponding to the point clouds in the sub-figures (C) and the point clouds in the sub-figures (E). Note that since we do not add a new point with an $\mathcal{E}_d$ value greater than 40 to our point clouds, most points in the point clouds have differences before \& after image points less than 40 pixels. Indeed, although the average value of the differences before \& after image points (15.5 pixels) is smaller than the average value of the differences before \& after key points (17.0 pixels), while the errors of key points effect to the errors of image points, we cannot conclude anything. However, there is an interesting point that we would like to discuss. Since there are lots of key points where their errors are more than 20 pixels as we see in Figure~\ref{Fig:Hist_Total}, we were worried at first that the constraint $\mathcal{E}_d(\cdot)\leq 40$ would limit the number of points in the point clouds, and the point clouds cannot cover everywhere of the object's surface. In spite of that the size of our point clouds is very large and they are everywhere the the surface as simulated by the sub-figures (E) in Figure~\ref{Fig:Creating_PClouds}. Furthermore, since there are more good points ($\mathcal{E}_d$-values are less than 10) in the sub-figures (D) than in the sub-figures (B), the number of good points in (E) is very large (20.4\% of two millions points) compared to (C) (18.1\% of one million points) (see in Figure~\ref{Fig:Hist_Frontbody}). This fact allows us to believe that the Buddha statue can be reconstructed by good point clouds. In another word, we can achieve an accurate 3D Buddha statue reconstruction. Histograms of differences between before \& after image points from other experiments are similar as two histograms given by Figure~\ref{Fig:Hist_Frontbody} from the front body experiment that warrants our conclusion on the proposed algorithm of creating point clouds.
\newpage

Finally, because we do not have any theoretical tool as~\eqref{eq:Func_Error_WP} and~\eqref{eq:Func_Error_PC} to evaluate the 3D reconstruction work without ground-truths, the best way to evaluate our point clouds for the 3D Buddha statue reconstruction is bring out some figures like Figure~\ref{Fig:Total_PClouds_1}. Figure~\ref{Fig:Total_PClouds_2} gives six different views of the refinement point clouds corresponding to distance $\delta = 0.05$ centimeter along with similar view images. The part reconstructed by the \emph{front body experiment} is highlighted by red. From these six views, we strongly confirm that the 3D shape of the Buddha statue can be modeled with a high accuracy. The dense clouds rendered by these point clouds after applying some bundle adjustment techniques would be a realistic and natural model of the Buddha statue. Figure~\ref{Fig:Face_PClouds} shows more detail as it gives four different views of just the head. We hope the readers can see the beautiful and gentle face of the Buddha statue through these point clouds. The point clouds used in Figure~\ref{Fig:Face_PClouds} are a combine of the refinement corresponding to distance $\delta = 0.05$ centimeter and the other corresponding to $\epsilon = 5$ pixels. To finish this section, we emphasize that these point clouds are not final, and we will work to get better point clouds.

\section{Conclusion}\label{Sec:Conclusion}

Beginning with a simple work that marks 112 special points on the Buddha statue surface and takes 402 images of this statue such that each image captures at least these six special points, the study from this article builds the point clouds with millions points to 3D reconstruct this Buddha statue. The point clouds cover the entire surface of the statue with most their relative errors are less than 30 pixels. To our experience on image-based 3D reconstruction, this achievement will give the statue a realistic and natural look when these point clouds are rendered. Our point clouds building process is based on three steps. First, we detect the special points on images that we call key points, and accurately estimate the correspondences from these key points. Second, we estimate both the 3D coordinates for the initial special points and the projection matrices for all images. Third, based on the estimated special points and the estimated projection matrices from the second step, we find as many geo-feature correspondences as possible and create the point clouds. The first step is studies by our previous work in~\citet{le2022multi}, the second and third steps are presented by Sections~\ref{Sec:World_points} and~\ref{Sec:Point_Clouds}, respectively.  Concretely, algorithm \texttt{WPfC} solves the second step and algorithm \texttt{CrPC} solves the third step.

As we discuss the protocol to create the point clouds simulated by Figure~\ref{Fig:Creating_PClouds} in Section~\ref{Sec:PC_from_Correspondences}, it seems that our method is not limited to the point clouds with millions points and their relative errors less than 30 pixels. It can help us to obtain the better point clouds with tens or hundreds million points and their relative errors less than 20 pixels. Indeed, if this discussion is correct, we hope our proposed method succeeds not only in 3D reconstruction of Buddha statues but also other objects such as human face or human body. Imagining that how much fun it is when we success to create the point clouds for human faces and we are able to recognize people from these point clouds. We believe that this is a huge challenge. However, the results we obtained in Figure~\ref{Fig:Face_PClouds} will give us a strong incentive to study and address this challenge in the future.

Finally, our results on the projection matrix estimation and the point clouds creation differ from other existing results. For instance, we estimate the projection matrices not through the fundamental matrix estimation, but through the assignment of the first five world points and create a new point for the point clouds based on the new feature named the geodesic feature or the geo-feature. Therefore, we believe that there are better 3D reconstructions than currently available if we combine our methods with other existing methods.

\bibliographystyle{plainnat}
\bibliography{refs}

\begin{thebibliography}{55}
\providecommand{\natexlab}[1]{#1}
\providecommand{\url}[1]{\texttt{#1}}
\expandafter\ifx\csname urlstyle\endcsname\relax
  \providecommand{\doi}[1]{doi: #1}\else
  \providecommand{\doi}{doi: \begingroup \urlstyle{rm}\Url}\fi

\bibitem[Agarwal et~al.(2010)Agarwal, Snavely, Seitz, and
  Szeliski]{Agarwal2010}
Sameer Agarwal, Noah Snavely, Steven~M. Seitz, and Richard Szeliski.
\newblock Bundle adjustment in the large.
\newblock In \emph{Proceedings of the 11th European Conference on Computer
  Vision (ECCV), Part II}, pages 29--42, Crete, Greece, 2010.

\bibitem[Alhichri and Kamel(2003)]{Alhichri2003}
Haikel~Salem Alhichri and Mohamed Kamel.
\newblock Virtual circles: a new set of features for fast image registration.
\newblock \emph{Pattern Recognit. Lett.}, 24\penalty0 (9-10):\penalty0
  1181--1190, 2003.

\bibitem[Barath(2018)]{Barath2018}
Daniel Barath.
\newblock Five-point fundamental matrix estimation for uncalibrated cameras.
\newblock In \emph{Proceedings of the {IEEE} Conference on Computer Vision and
  Pattern Recognition (CVPR)}, pages 235--243, Salt Lake City, UT, 2018.

\bibitem[Bay et~al.(2006)Bay, Tuytelaars, and Gool]{Bay2006}
Herbert Bay, Tinne Tuytelaars, and Luc~Van Gool.
\newblock {SURF:} speeded up robust features.
\newblock In \emph{Proceedings of the 9th European Conference on Computer
  Vision (ECCV), Part I}, pages 404--417, Graz, Austria, 2006.

\bibitem[Bresson et~al.(2017)Bresson, Alsayed, Yu, and Glaser]{Bresson2017}
Guillaume Bresson, Zayed Alsayed, Li~Yu, and Sebastien Glaser.
\newblock Simultaneous localization and mapping: {A} survey of current trends
  in autonomous driving.
\newblock \emph{{IEEE} Trans. Intell. Veh.}, 2\penalty0 (3):\penalty0 194--220,
  2017.

\bibitem[Bruno et~al.(2010)Bruno, Bruno, De~Sensi, Luchi, Mancuso, and
  Muzzupappa]{Bruno2010}
Fabio Bruno, Stefano Bruno, Giovanna De~Sensi, Maria-Laura Luchi, Stefania
  Mancuso, and Maurizio Muzzupappa.
\newblock From 3d reconstruction to virtual reality: A complete methodology for
  digital archaeological exhibition.
\newblock \emph{Journal of Cultural Heritage}, 11:\penalty0 42--49, 03 2010.

\bibitem[Busemann(2012)]{Busemann2012}
Herbert Busemann.
\newblock \emph{The geometry of geodesics}.
\newblock Courier Corporation, 2012.

\bibitem[Cadena et~al.(2016)Cadena, Carlone, Carrillo, Latif, Scaramuzza,
  Neira, Reid, and Leonard]{Cadena2016}
Cesar Cadena, Luca Carlone, Henry Carrillo, Yasir Latif, Davide Scaramuzza,
  Jos{\'{e}} Neira, Ian Reid, and John~J. Leonard.
\newblock Past, present, and future of simultaneous localization and mapping:
  Toward the robust-perception age.
\newblock \emph{{IEEE} Trans. Robotics}, 32\penalty0 (6):\penalty0 1309--1332,
  2016.

\bibitem[Carpo(2001)]{Carpo2001}
Mario Carpo.
\newblock \emph{Architecture in the age of printing: orality, writing,
  typography, and printed images in the history of architectural theory}.
\newblock MIT Press, 2001.

\bibitem[Choi et~al.(2022)Choi, Liu, and Lui]{Choi2022}
Gary Choi, Yechen Liu, and Lok~Ming Lui.
\newblock Free-boundary conformal parameterization of point clouds.
\newblock \emph{Journal of Scientific Computing}, 90\penalty0 (1):\penalty0
  1--26, 2022.

\bibitem[Chuang et~al.(2021)Chuang, Ho, Umam, Chen, Hwang, and
  Chen]{Chuang2021}
Jen{-}Hui Chuang, Chih{-}Hui Ho, Ardian Umam, HsinYi Chen, Jenq{-}Neng Hwang,
  and Tai{-}An Chen.
\newblock Geometry-based camera calibration using closed-form solution of
  principal line.
\newblock \emph{{IEEE} Trans. Image Process.}, 30:\penalty0 2599--2610, 2021.

\bibitem[Dionisio et~al.(2013)Dionisio, III, and Gilbert]{Dionisio2013}
John David~N. Dionisio, William G.~Burns III, and Richard Gilbert.
\newblock 3{D} {V}irtual {W}orlds and the {M}etaverse: {C}urrent status and
  future possibilities.
\newblock \emph{ACM Comput. Surv.}, 45\penalty0 (3):\penalty0 1--38, Jul 2013.

\bibitem[Doll{\'{a}}r and Zitnick(2013)]{dollar2013}
Piotr Doll{\'{a}}r and C.~Lawrence Zitnick.
\newblock Structured forests for fast edge detection.
\newblock In \emph{Proceedings of the {IEEE} International Conference on
  Computer Vision (ICCV)}, pages 1841--1848, Sydney, Australia, 2013.

\bibitem[Harltey and Zisserman(2006)]{Hartley2004}
Richard~I. Harltey and Andrew Zisserman.
\newblock \emph{Multiple view geometry in computer vision {(2.} ed.)}.
\newblock Cambridge University Press, 2006.

\bibitem[Hartigan and Wong(1979)]{Hartigan1979}
John~A Hartigan and Manchek~A Wong.
\newblock Algorithm as 136: A k-means clustering algorithm.
\newblock \emph{Journal of the royal statistical society. series c (applied
  statistics)}, 28\penalty0 (1):\penalty0 100--108, 1979.

\bibitem[Hartley and Kahl(2007)]{Hartley2007}
Richard~I. Hartley and Fredrik Kahl.
\newblock Optimal algorithms in multiview geometry.
\newblock In \emph{Proc. Eighth Asian Conference on Computer Vision, Part {I},
  ({ACCV})}, pages 13--34, 2007.

\bibitem[Heimann and Meinzer(2009)]{Heimann2009}
Tobias Heimann and Hans{-}Peter Meinzer.
\newblock Statistical shape models for 3{D} medical image segmentation: {A}
  review.
\newblock \emph{Medical Image Anal.}, 13\penalty0 (4):\penalty0 543--563, 2009.

\bibitem[Hermon(2008)]{Hermon2008}
Sorin Hermon.
\newblock Reasoning in 3{D}: {A} critical appraisal of the role of 3{D}
  modelling and virtual reconstructions in archaeology.
\newblock \emph{Beyond illustration: 2D and 3D digital technologies as tools
  for discovery in archaeology}, 1805:\penalty0 36--45, 2008.

\bibitem[Horn(1987)]{Horn1987}
Berthold K.~P. Horn.
\newblock Closed-form solution of absolute orientation using unit quaternions.
\newblock \emph{J. Opt. Soc. Am. A}, 4\penalty0 (4):\penalty0 629--642, Apr
  1987.

\bibitem[Horst and Tuy(1996)]{Horst1996}
Reiner Horst and Hoang Tuy.
\newblock \emph{Global optimization - deterministic approaches}.
\newblock Berlin: Springer, 1996.

\bibitem[Hsieh et~al.(1996)Hsieh, Liao, Fan, and Ko]{Hsieh1996}
Jun{-}Wei Hsieh, Hong{-}Yuan~Mark Liao, Kuo{-}Chin Fan, and Ming{-}Tak Ko.
\newblock A fast algorithm for image registration without predetermining
  correspondences.
\newblock In \emph{Proceedings of the 13th International Conference on Pattern
  Recognition (ICPR)}, pages 765--769, Vienna, Austria, 1996.

\bibitem[Jenks(1967)]{Jenks1967}
George~F Jenks.
\newblock The data model concept in statistical mapping.
\newblock \emph{International yearbook of cartography}, 7:\penalty0 186--190,
  1967.

\bibitem[Kemp and Livingstone(2006)]{Kemp2006}
Jeremy Kemp and Daniel Livingstone.
\newblock Putting a {S}econd {L}ife “metaverse” skin on learning management
  systems.
\newblock In \emph{Proceedings of the Second Life education workshop at the
  Second Life community convention}, volume~20. The University of Paisley CA,
  San Francisco, 2006.

\bibitem[Lawler(1963)]{Lawler1963}
Eugene~L Lawler.
\newblock The quadratic assignment problem.
\newblock \emph{Management science}, 9\penalty0 (4):\penalty0 586--599, 1963.

\bibitem[Le and Ho(2020{\natexlab{a}})]{Kien2020}
Trung{-}Kien Le and King~Choi Ho.
\newblock Algebraic complete solution for joint source and sensor localization
  using time of flight measurements.
\newblock \emph{{IEEE} Trans. Signal Process.}, 68:\penalty0 1853--1869,
  2020{\natexlab{a}}.

\bibitem[Le and Ho(2020{\natexlab{b}})]{Kien2020_2}
Trung-Kien Le and King~Choi Ho.
\newblock Joint source and sensor localization by angles of arrival.
\newblock \emph{{IEEE} Trans. Signal Process.}, 68:\penalty0 6521--6534,
  2020{\natexlab{b}}.

\bibitem[Le and Li(2022)]{le2022multi}
Trung-Kien Le and Ping Li.
\newblock Multi-view geometry: Correspondences refinement based on algebraic
  properties.
\newblock \emph{arXiv preprint arXiv:2205.01634}, 2022.

\bibitem[Le and Ono(2016)]{Kien2016}
Trung{-}Kien Le and Nobutaka Ono.
\newblock Closed-form and near closed-form solutions for toa-based joint source
  and sensor localization.
\newblock \emph{{IEEE} Trans. Signal Process.}, 64\penalty0 (18):\penalty0
  4751--4766, 2016.

\bibitem[Lowe(1999)]{Lowe1999}
David~G. Lowe.
\newblock Object recognition from local scale-invariant features.
\newblock In \emph{Proceedings of the International Conference on Computer
  Vision (ICCV)}, pages 1150--1157, Kerkyra, Corfu, Greece, 1999.

\bibitem[Luong and Faugeras(1996)]{Luong1996}
Quang{-}Tuan Luong and Olivier~D. Faugeras.
\newblock The fundamental matrix: Theory, algorithms, and stability analysis.
\newblock \emph{Int. J. Comput. Vis.}, 17\penalty0 (1):\penalty0 43--75, 1996.

\bibitem[McInerney and Terzopoulos(1996)]{McInerney1996}
Tim McInerney and Demetri Terzopoulos.
\newblock Deformable models in medical image analysis: {A} survey.
\newblock \emph{Medical Image Anal.}, 1\penalty0 (2):\penalty0 91--108, 1996.

\bibitem[Mikolajczyk and Schmid(2005)]{Mikolajczyk2004}
Krystian Mikolajczyk and Cordelia Schmid.
\newblock A performance evaluation of local descriptors.
\newblock \emph{{IEEE} Trans. Pattern Anal. Mach. Intell.}, 27\penalty0
  (10):\penalty0 1615--1630, 2005.

\bibitem[Munkres(1957)]{Munkres1957}
James Munkres.
\newblock Algorithms for the assignment and transportation problems.
\newblock \emph{Journal of the society for industrial and applied mathematics},
  5\penalty0 (1):\penalty0 32--38, 1957.

\bibitem[Myronenko and Song(2010)]{Myronenko2010}
Andriy Myronenko and Xubo Song.
\newblock Point set registration: Coherent point drift.
\newblock \emph{IEEE Trans. Pattern Anal. Mach. Intell.}, 32\penalty0
  (12):\penalty0 2262--2275, 2010.

\bibitem[Nist{\'{e}}r(2004)]{Nister2004}
David Nist{\'{e}}r.
\newblock An efficient solution to the five-point relative pose problem.
\newblock \emph{{IEEE} Trans. Pattern Anal. Mach. Intell.}, 26\penalty0
  (6):\penalty0 756--777, 2004.

\bibitem[Oliveira and Tavares(2014)]{Francisco2014}
Francisco~P.M. Oliveira and João Manuel~R.S. Tavares.
\newblock Medical image registration: A review.
\newblock \emph{Comput Methods Biomech Biomed Engin .}, 17\penalty0
  (2):\penalty0 73--93, 2014.

\bibitem[Remondino and Fraser(2006)]{Remondino2006}
Fabio Remondino and Clive Fraser.
\newblock Digital camera calibration methods: {C}onsiderations and comparisons.
\newblock \emph{Int. Arch. Photogramm. Remote Sens.}, 36\penalty0 (5):\penalty0
  266--272, 2006.

\bibitem[Rosten and Drummond(2006)]{Rosten2006}
Edward Rosten and Tom Drummond.
\newblock Machine learning for high-speed corner detection.
\newblock In \emph{Proceedings of the 9th European Conference on Computer
  Vision (ECCV), Part I}, pages 430--443, Graz, Austria, 2006.

\bibitem[Sch\"{o}nemann(1970)]{Schonemann1970}
Peter~H. Sch\"{o}nemann.
\newblock On metric multidimensional unfolding.
\newblock \emph{Psychometrika}, 35\penalty0 (3):\penalty0 349--366, 1970.

\bibitem[Sharp and Ovsjanikov(2020)]{Sharp2020}
Nicholas Sharp and Maks Ovsjanikov.
\newblock Point{T}ri{N}et: {L}earned {T}riangulation of 3{D} point sets.
\newblock In \emph{Proceedings of the 16th European Conference on Computer
  Vision (ECCV), Part XXIII}, pages 762--778, Glasgow, UK, 2020.

\bibitem[Smith et~al.(2016)Smith, Carrivick, and Quincey]{Smith2016}
Mark~William Smith, Jonathan~L. Carrivick, and Duncan~J. Quincey.
\newblock Structure from {M}otion {P}hotogrammetry in {P}hysical {G}eography.
\newblock \emph{Progress in Physical Geography}, 40\penalty0 (2):\penalty0
  247--275, 2016.

\bibitem[Stew{\'{e}}nius et~al.(2005)Stew{\'{e}}nius, Schaffalitzky, and
  Nist{\'{e}}r]{Stewenius2005}
Henrik Stew{\'{e}}nius, Frederik Schaffalitzky, and David Nist{\'{e}}r.
\newblock How hard is 3-view triangulation really?
\newblock In \emph{Proceedings of the 10th {IEEE} International Conference on
  Computer Vision (ICCV)}, pages 686--693, Beijing, China, 2005.

\bibitem[Swoboda et~al.(2019)Swoboda, Kainm{\"{u}}ller, Mokarian, Theobalt, and
  Bernard]{Swoboda2019}
Paul Swoboda, Dagmar Kainm{\"{u}}ller, Ashkan Mokarian, Christian Theobalt, and
  Florian Bernard.
\newblock A convex relaxation for multi-graph matching.
\newblock In \emph{Proceedings of the {IEEE} Conference on Computer Vision and
  Pattern Recognition (CVPR)}, pages 11156--11165, Long Beach, CA, 2019.

\bibitem[Torresani et~al.(2012)Torresani, Kolmogorov, and
  Rother]{Torresani2012}
Lorenzo Torresani, Vladimir Kolmogorov, and Carsten Rother.
\newblock A dual decomposition approach to feature correspondence.
\newblock \emph{IEEE Trans. Pattern Anal. Mach. Intell.}, 35\penalty0
  (2):\penalty0 259--271, 2012.

\bibitem[Triggs et~al.(1999)Triggs, McLauchlan, Hartley, and
  Fitzgibbon]{Triggs1999}
Bill Triggs, Philip~F. McLauchlan, Richard~I. Hartley, and Andrew~W.
  Fitzgibbon.
\newblock Bundle adjustment - {A} modern synthesis.
\newblock In \emph{Proceedings of the International Workshop on Vision
  Algorithms (ICCV)}, pages 298--372, Corfu, Greece, 1999.

\bibitem[Wang et~al.(2007)Wang, Thorpe, Thrun, Hebert, and
  Durrant{-}Whyte]{Wang2007}
Chieh{-}Chih Wang, Charles~E. Thorpe, Sebastian Thrun, Martial Hebert, and
  Hugh~F. Durrant{-}Whyte.
\newblock Simultaneous {L}ocalization, {M}apping and {M}oving {O}bject
  {T}racking.
\newblock \emph{Int. J. Robotics Res.}, 26\penalty0 (9):\penalty0 889--916,
  2007.

\bibitem[Wei and Ma(1994)]{Wei1994}
Guo-Qing Wei and Song~De Ma.
\newblock Implicit and explicit camera calibration: theory and experiments.
\newblock \emph{IEEE Trans. Pattern Anal. Mach. Intell.}, 16\penalty0
  (5):\penalty0 469--480, 1994.

\bibitem[Weng et~al.(1992)Weng, Cohen, and Herniou]{Weng1992}
Juyang Weng, Paul~R. Cohen, and Marc Herniou.
\newblock Camera calibration with distortion models and accuracy evaluation.
\newblock \emph{{IEEE} Trans. Pattern Anal. Mach. Intell.}, 14\penalty0
  (10):\penalty0 965--980, 1992.

\bibitem[Wu and Yau(2020)]{Wu2020}
Tianqi Wu and Shing-Tung Yau.
\newblock Computing {H}armonic {M}aps and {C}onformal {M}aps on {P}oint
  {C}louds.
\newblock \emph{arXiv preprint arXiv:2009.09383}, 2020.

\bibitem[Yan et~al.(2016)Yan, Cho, Zha, Yang, and Chu]{Yan2016}
Junchi Yan, Minsu Cho, Hongyuan Zha, Xiaokang Yang, and Stephen~M. Chu.
\newblock Multi-graph matching via affinity optimization with graduated
  consistency regularization.
\newblock \emph{IEEE Trans. Pattern Anal. Mach. Intell.}, 38\penalty0
  (6):\penalty0 1228--1242, 2016.

\bibitem[Yang et~al.(2019)Yang, Fang, Zhao, and Deng]{yang2019}
Kui Yang, Wei Fang, Yan Zhao, and Nianmao Deng.
\newblock Iteratively reweighted midpoint method for fast multiple view
  triangulation.
\newblock \emph{IEEE Robot. Autom. Lett.}, 4\penalty0 (2):\penalty0 708--715,
  2019.

\bibitem[Zach(2014)]{Zach2014}
Christopher Zach.
\newblock Robust bundle adjustment revisited.
\newblock In \emph{Proceedings of the 13th European Conference on Computer
  Vision (ECCV), Part V}, pages 772--787, Zurich, Switzerland, 2014.

\bibitem[Zhang(2000)]{Zhang2000}
Zhengyou Zhang.
\newblock A flexible new technique for camera calibration.
\newblock \emph{{IEEE} Trans. Pattern Anal. Mach. Intell.}, 22\penalty0
  (11):\penalty0 1330--1334, 2000.

\bibitem[Zheng et~al.(1999)Zheng, Wang, and Teoh]{Zheng1999}
Zhiqiang Zheng, Han Wang, and Eam~Khwang Teoh.
\newblock Analysis of gray level corner detection.
\newblock \emph{Pattern Recognit. Lett.}, 20\penalty0 (2):\penalty0 149--162,
  1999.

\bibitem[Ziou and Tabbone(1998)]{Ziou1998}
Djemel Ziou and Salvatore Tabbone.
\newblock Edge detection techniques - an overview.
\newblock \emph{Pattern Recognition and Image Analysis}, 8:\penalty0 537--559,
  1998.

\end{thebibliography}

\end{document}